\documentclass[11pt,a4paper]{article}
\usepackage[top=3cm, bottom=3cm, left=3cm, right=3cm]{geometry}

\usepackage{booktabs}
\usepackage{color}
\usepackage{latexsym} 
\usepackage{amsmath}               
\usepackage{amssymb}               
\usepackage{amsfonts}              
\usepackage{amsthm}                
\usepackage{tcolorbox}
\usepackage{times}
\usepackage[utf8]{inputenc}
\usepackage[T1]{fontenc}
\usepackage{enumitem}
\usepackage{xcolor}
\usepackage{thmtools}
\usepackage{thm-restate}
\usepackage[square,numbers]{natbib}
\usepackage[colorlinks=true,allcolors=blue]{hyperref}
\usepackage{cleveref}
\usepackage{url}            
\usepackage{booktabs} 
\usepackage{nicefrac}       
\usepackage{microtype}      
\usepackage[american]{babel}
\usepackage{thmtools}
\usepackage{booktabs}
\usepackage{algorithm}
\usepackage[tikz]{bclogo}
\usepackage{multicol,tabularx,capt-of}
\usepackage{hhline}
\usepackage{multirow}
\usepackage{bibunits}

\usepackage{graphicx}

\usepackage{microtype}
\usepackage{algpseudocode}

\usepackage{stmaryrd}
\usepackage{hyperref}

\usepackage[parfill]{parskip}

\usepackage{url}

\usepackage{subfigure} 

\usepackage{bibunits}

\usepackage{booktabs}

 \usepackage{cite}
 
\usepackage{xcolor}
 \definecolor{dgreen}{rgb}{0.00,0.49,0.00}
\definecolor{dblue}{rgb}{0,0.08,0.75}
\hypersetup{
    colorlinks = true,
    citecolor=dgreen,
    urlcolor=dblue,
    linkcolor=dblue
}

\usepackage{newtxtext}      

\newcommand{\pdm}{\mathbb{S}_+}

\newcommand{\Nystrom}[1]{{Nystr\"om}}

\providecommand{\scal}[2]{\left\langle{#1},{#2}\right\rangle}

\providecommand{\tr}{\operatorname{Tr}}
\providecommand{\diag}{\operatorname{diag}}

\newcommand{\psdm}{\mathbb{S}_+}

\DeclareMathOperator{\erf}{erf}

\newcommand{\ib}{\boldsymbol{1}}

\DeclareMathOperator{\argmax}{argmax}

\newcommand{\R}{\mathbb R}

\newcommand{\N}{\mathbb N}

\newcommand{\dyad}{{\cal D}}
\newcommand{\wtx}{\widetilde{X}}

\newcommand{\ghat}{\widehat{g}}
\newcommand{\ahat}{\widehat{a}}
\newcommand{\Zhat}{\widehat{Z}}
\newcommand{\Zp}{Z_p}
\newcommand{\hh}{\mathcal H}

\newcommand{\la}{\lambda}
\newcommand{\eps}{\varepsilon}
\newcommand{\lspan}[1]{\operatorname{span}\{#1\}}

\newcommand{\argmin}[1]{\mathop{\operatorname{argmin}}_{#1}}

\newcommand{\Nsos}[2]{\left\|#1\right\|_{\texttt{sos},#2}}

\newcommand{\expect}[1]{{\mathbb E}[#1]}

\newcommand{\xx}{{\cal X}}

\newcommand{\uu}{{\cal U}}

\newcommand{\eqals}[1]{\begin{align*}#1\end{align*}}
\newcommand{\eqal}[1]{\begin{align}#1\end{align}}
\newcommand{\bpr}{\begin{proof}}
\newcommand{\epr}{\end{proof}}
\newcommand{\be}{\begin{equation}}
\newcommand{\ee}{\end{equation}}

\newtheorem{definition}{Definition}
\newcommand{\bd}{\begin{definition}}
\newcommand{\ed}{\end{definition}}
\newlist{enumdef}{enumerate}{1} 
\setlist[enumdef]{label=\upshape(\alph*),ref=\upshape\thedefinition(\alph*)}
\crefalias{enumdefi}{theorem} 

\newcommand{\bi}{\begin{itemize}}
\newcommand{\ei}{\end{itemize}}

\newtheorem{ass}{Assumption}
\newcommand{\ba}{\begin{ass}}
\newcommand{\ea}{\end{ass}}
\newlist{enumasm}{enumerate}{1} 
\setlist[enumasm]{label=\upshape(\alph*),ref=\upshape\theass(\alph*)}
\crefalias{enumasmi}{ass} 

\newtheorem{remark}{Remark}
\newcommand{\br}{\begin{remark}}
\newcommand{\er}{\end{remark}}
\newlist{enumrem}{enumerate}{1} 
\setlist[enumrem]{label=\upshape(\alph*),ref=\upshape\theremark(\alph*)}
\crefalias{enumremi}{theorem} 

\newtheorem{proposition}{Proposition}
\newcommand{\bp}{\begin{proposition}}
\newcommand{\ep}{\end{proposition}}
\newlist{enumprop}{enumerate}{1} 
\setlist[enumprop]{label=\upshape(\alph*),ref=\upshape\theproposition(\alph*)}
\crefalias{enumpropi}{proposition} 

\newtheorem{lemma}{Lemma}
\newcommand{\blm}{\begin{lemma}}
\newcommand{\elm}{\end{lemma}}
\newlist{enumlm}{enumerate}{1} 
\setlist[enumlm]{label=\upshape(\alph*),ref=\upshape\thelemma(\alph*)}
\crefalias{enumlmi}{lemma} 

\newtheorem{theorem}{Theorem}
\newcommand{\bt}{\begin{theorem}}
\newcommand{\et}{\end{theorem}}
\newlist{enumthm}{enumerate}{1} 
\setlist[enumthm]{label=\upshape(\alph*),ref=\upshape\thetheorem(\alph*)}
\crefalias{enumthmi}{theorem} 

\newtheorem{corollary}{Corollary}
\newcommand{\bcor}{\begin{corollary}}
\newcommand{\ecor}{\end{corollary}}
\newlist{enumcor}{enumerate}{1} 
\setlist[enumcor]{label=\upshape(\alph*),ref=\upshape\thecorollary(\alph*)}
\crefalias{enumcori}{corollary} 

\newtheorem{example}{Example}
\newcommand{\bex}{\begin{example}}
\newcommand{\eex}{\end{example}}
\newlist{enumex}{enumerate}{1} 
\setlist[enumex]{label=\upshape(\alph*),ref=\upshape\theexample(\alph*)}
\crefalias{enumexi}{theorem} 

\crefname{ass}{Assumption}{Assumptions}
\crefname{equation}{Eq.}{Eqs.}
\crefname{figure}{Fig.}{Figs.}
\crefname{table}{Table}{Tables}
\crefname{section}{Sec.}{Secs.}
\crefname{theorem}{Theorem}{Theorems}
\crefname{lemma}{Lemma}{Lemmas}
\crefname{corollary}{Cor.}{Cors.}
\crefname{example}{Example}{Examples}
\crefname{appendix}{Appendix}{Appendixes}
\crefname{remark}{Remark}{Remark}
\crefname{page}{page}{page}

\newcommand{\fgauss}{\widehat{f}_{\tau,m,\lambda}}
\newcommand{\Agauss}{\widehat{A}_{\tau,m,\lambda}}
\newcommand{\Mgauss}{\widehat{M}_{\tau,m,\lambda}}

\newcommand{\fproj}{\widetilde{f}_{\tau,m,\epsilon}}
\newcommand{\Mproj}{\widetilde{M}_{\tau,m,\epsilon}}
\newcommand{\Aproj}{\widetilde{A}_{\tau,m,\epsilon}}

\newcommand{\feps}{f_{\tau,\epsilon}}
\newcommand{\Meps}{M_{\tau,\epsilon}}

\newcommand{\xt}{\tilde{x}}

\newcommand{\Cp}{C^{\prime}}
\newcommand{\Cpp}{C^{\prime\prime}}

\newcommand{\fhat}{\widehat{f}}

\newcommand{\Ahat}{\widehat{A}}
\newcommand{\phat}{\widehat{p}}

\newcommand{\ggauss}{\widehat{g}_{\tau,m,\lambda}}
\newcommand{\agauss}{\widehat{a}_{\tau,m,\lambda}}
\newcommand{\Zgauss}{\widehat{Z}_{\tau,m,\lambda}}

\newcommand{\pgauss}{\widehat{p}_{\tau,m,\lambda}}

\newcommand{\fp}{f_p}
\newcommand{\gp}{g_p}
\newcommand{\gph}{\widehat{g}_p}
\newcommand{\pp}[2]{f(#1;~#2)}
\newcommand{\mm}{M}
\newcommand{\ppl}[2]{g(#1;~#2)}

\newcommand{\Wass}{\mathbb{W}}

\newcommand{\Kmatrix}[3]{K_{#1,#2,#3}}
\newcommand{\Kmat}[2]{K_{#1,#2}}
\newcommand{\Imatrix}[4]{G_{#1,#2,#3,#4}}
\newcommand{\Imat}[3]{G_{#1,#2,#3}}

\DeclareMathOperator{\vect}{vec}
\DeclareMathOperator{\Lip}{Lip}
\DeclareMathOperator{\Diam}{diam}
\newcommand{\Zm}{\widetilde{Z}_{\eta,m}}

\newcommand{\hhe}{\hh_{\eta}}
\newcommand{\phie}{\phi_{\eta}}
\newcommand{\hht}{\widetilde{\hh}_{\eta,m}}

\newcommand{\geps}{g_{\tau,\eps}}

\newcommand{\Cop}{C_{\eta}}
\newcommand{\Cn}{\widehat{C}_{\eta}}
\newcommand{\Cnl}{\widehat{C}_{\eta,\lambda}}
\newcommand{\Cl}{C_{\eta,\lambda}}
\newcommand{\Projm}{\widetilde{P}_{\eta,m}}

\newcommand{\Sn}{\widehat{S}_{\eta}}

\newcommand{\FT}{{\cal F}}
\newcommand{\nut}{\widetilde{\nu}}

\newcommand{\Rgauss}{\widehat{R}_{\tau,m,\la}}

\newcommand{\psample}{p_{\texttt{sample}}}

\newcommand{\precr}{\prec_{\rho}}

\newcommand{\Yb}{\boldsymbol{Y}}

\newcommand{\Lipt}{\widetilde{\Lip}}

\newcommand{\Wt}{\widetilde{W}}

\renewcommand{\leq}{\leqslant}
\renewcommand{\geq}{\geqslant}

\usepackage{authblk}

\author{\bfseries
Ulysse Marteau-Ferey}
\author{\bfseries
Francis Bach}
\author{\bfseries
Alessandro Rudi
}
\affil{
INRIA - D{\'e}partement d'Informatique de l'{\'E}cole Normale Sup{\'e}rieure \\
PSL Research University\\
Paris, France
}

\title{Sampling from Arbitrary Functions via PSD Models}

\begin{document}

\maketitle
\begin{abstract}
  In many areas of applied statistics and machine learning, generating an arbitrary number of independent and identically distributed (i.i.d.) samples from a given distribution is a key task. When the distribution is known only through evaluations of the density, current methods either scale badly with the dimension or require very involved implementations. Instead, we take a two-step approach by first modeling the probability distribution and then sampling from that model. We use the recently introduced class of positive semi-definite (PSD) models, which have been shown to be efficient for approximating probability densities. We show that these models can approximate a large class of densities concisely using few evaluations, and present a simple algorithm to effectively sample from these models. We also present preliminary empirical results to illustrate our assertions.  
\end{abstract}

\section{Introduction}

In many fields such as biochemistry, statistical mechanics and machine learning, effectively sampling arbitrary numbers of independent and identically distributed (i.i.d.) samples from probability distributions is a key task \citep{gelmanbda04,liu2008,lelievre2010}. 

Basic sampling methods include rejection sampling and gridding, and rely on simple properties of the density. However, they are suitable only in small dimensions, except for very structured cases. Moreover, they are hard to adapt to probabilities which are known up to their renormalization constant, which is often the case when dealing with exponential models that are common in applications \citep{robert2004}.

More involved methods have been developed to address these dimensionality and renormalization issues, in the class of so-called Markov chain Monte Carlo (MCMC) methods. However, they are complex to set up: in particular, independence between samples is not directly guaranteed, convergence can be slow and hard to measure non-asymptotically \citep{lelievre2010,robert2004}.

In this work, we address the problem in a different way, by incorporating a modeling step. Instead of sampling directly from the target density, we first model this density using a positive semi-definite (PSD) model \citep{marteau20,rudi2021psd}, and then sample from this PSD model.

PSD models have been introduced by \citet{marteau20} and their relevance for modeling probability distributions has been further established by \citet{rudi2021psd}, showing that i) they are stable under key operations for probabilistic inference, such as marginalization, integration (also called ``sum-rule''), and product, which can be done efficiently in practice, and ii) they concisely approximate a large class of probability distributions. We present these models in \cref{sec:psd_models}. Building on this work, we show that these models are also relevant in the context of sampling, making the following main contributions. 
\begin{enumerate}[wide,labelindent=4pt]
\item[(1)] In \cref{sec:sampling}, we derive an algorithm that is easy to implement and which can generate an arbitrary number of i.i.d. samples from a given PSD model, with any given precision. This answers one of the open questions outlined by \citet{rudi2021psd} and shows that one can indeed efficiently sample from a PSD model.
\item[(2)] In \cref{sec:approximation_psd} we show that we can sample an arbitrary number of i.i.d.~samples from a target probability distribution that is regular enough, with any given precision. The algorithm consists in (a) approximating the un-normalized density $p$ via a PSD model, using evaluations of $p$, and (b) extracting i.i.d.~samples from the PSD model. We show that for sufficiently regular densities the resulting PSD model is concise and avoids the curse of dimensionality: to achieve error $\varepsilon$, the PSD model requires a number of parameters and a number of evaluations of $p$ that are in the order $\varepsilon^{-2-d/\beta}$, where $d$ is the dimension of the space and $\beta$ is the order of differentiability of the density. For regular probabilities, i.e., when $\beta \geq d$, the rate does not depend exponentially on $d$ and is bounded by $O(\varepsilon^{-3})$ (the constant term instead may depend exponentially on $d$).

\end{enumerate}
In \cref{sec:experiments}, we also present numerical simulations which demonstrate the quality of both our sampling technique and approximation results.


\section{Backround on Positive Semi-Definite (PSD) Models}\label{sec:psd_models}

Denote by $\R^d_{++}$ the vectors of $\R^d$ with positive components and $\psdm^m$ the set of positive semi-definite $m$ by $m$ matrices.
Following \citet{marteau20,rudi2021psd}, a Gaussian PSD model is parametrized by a triplet $(A,X,\eta) \in \psdm^m \times \R^{m \times d} \times \R^d_{++}$, and is defined for any $x \in \R^d$ as
\begin{equation}\label{eq:df_gaussian_psd}
    \pp{x}{A,X,\eta} = \sum_{i,j = 1 }^{m}{A_{ij}k_{\eta}(x,x_i)k_{\eta}(x,x_j)},
\end{equation}
where, with $\diag(\eta)$ being the diagonal matrix with diagonal $\eta$, $k_\eta(x,x') = e^{-(x-x^{\prime})^\top \diag(\eta)(x-x^{\prime})}$ is the Gaussian kernel of parameter $\eta$ 
, $X \in \R^{n\times d}$ is the matrix whose rows corresponds to the centers $x_1, \dots, x_n$ of the Gaussian PSD model, and $A$ is a matrix of coefficients which is positive semi-definite, to guarantee the non-negativity of $f$. 

Note that when $A = aa^{\top},~a \in \R^m$, is a rank-$1$ operator, a Gaussian PSD model is simply the square of a 
linear model $\pp{x}{A,X,\eta} = \ppl{x}{a,X,\eta}^2$ of the form,
\begin{equation}
    \label{eq:df_gaussian_linear}
    \ppl{x}{a,X,\eta} = \sum_{i=1}^m{a_i k_{\eta}(x,x_i)},
\end{equation}
for any $x \in \R^d$.
This particular case of PSD model will appear when approximating an arbitrary probability density~$p$ in \cref{secsec:from_distribution}.

\subsection{Main properties of PSD models}
As explained in the introduction, PSD models show properties that make them particularly well suited to model non-negative functions and probability distributions. Such properties are analyzed by \citet{marteau20} and \citet{rudi2021psd}, here we recall the one that are important for our purpose.

\paragraph{Non-negativity.}
Since $A$ is positive semidefinite, then the PSD model $\pp{x}{A,X,\eta}$ satisfies $\pp{x}{A,X,\eta} \geq 0$ for all $x \in \R^d$.

\paragraph{Preservation of convex functionals.}
Using the PSD model to represent non-negative functions in a problem of the form $\min_{f \geq 0} L(f)$, where $L$ is a convex functional, leads to a convex problem $\min_{A \in \pdm(\R^m)} L(\pp{\cdot}{A, X, \eta})$. Indeed, the constraint $A \in \pdm(\R^m)$ is convex, the PSD model $\pp{\cdot}{A, X, \eta}$ is linear in the parameter matrix $A$ and a composition of a convex function $L$ with a linear function is convex. This allows, e.g., to perform empirical risk minimization for the square and logarithmic losses. 

\paragraph{Conciseness of the representation.}
under mild conditions, recalled in \cref{asm:1b}, a PSD can approximate a probability density that is $\beta$-times differentiable with error $\eps$, using a number of centers $m  = O(\eps^{-d/\beta})$ (which is minimax optimal).   \citet{rudi2021psd}  provide also an algorithm to learn the PSD model given i.i.d.~samples from the probability. However, we cannot use this result in our context since we do not assume to have samples from our density.

\paragraph{Integration over hyper-rectangles in closed form.} As integration of PSD models will play a key role in the algorithm developed for sampling in \cref{sec:sampling}, both for theoretical an computational reasons, we recall this integration aspect more in details here.

A hyper-rectangle $Q \subset \R^d$ can be parametrized with its \textit{corners} $a,b \in \R^d,~a \leq b$  by writing $Q= \prod_{k =1}^d{[a_k,b_k[}$; $a$ corresponds to the ``bottom left'' corner and $b$ to the ``top right'' one. 

For $X\in \R^{m\times d}$ and $\eta \in \R^d_{++}$, we denote with $\Kmat{X}{\eta} \in \R^{m\times m}$ the \textit{kernel matrix} such that $[\Kmat{X}{\eta}]_{ij} = k_{\eta}(x_i,x_j)$.  The integral of a PSD model in \cref{eq:df_gaussian_psd} over a hyper-rectangle can be expressed with simple matrices, leveraging the fact that for any pair $(x_i,x_j)$, it holds $k_{\eta}(x,x_i)k_{\eta}(x,x_j) = k_{\eta/2}(x_i,x_j)k_{2\eta}(x,(x_i + x_j)/2)$. Then we have 
\begin{align}
I(Q;A,X,\eta) &:= \int_{Q}{f(x;A,X,\eta)~dx} \nonumber\\
& = \sum_{i,j=1}^m{A_{ij} k_{\frac{\eta}{2}}(x_i,x_j) \int_{Q}{k_{2\eta}(x,\tfrac{x_i+x_j}{2})~dx}}\nonumber\\
& = \sum_{i,j=1}^m{A_{ij}[\Kmat{X}{\eta/2}]_{ij} [\Imat{X}{2\eta}{Q}]_{ij}},\label{eq:int_psd}
\end{align}
where $[\Imat{X}{\eta}{Q}]_{ij} = \int_{Q_{ij}}{k_{\eta}(x,0)~dx}$, and  $Q_{ij} = Q- (x_i+x_j)/2$. These integrals can be computed by $2d$ calls to the $\erf$ function, as, for any $i,j \in \{1,...,m\}$:
\begin{equation}\label{eq:Phi_from_erf}
   [\Imat{X}{\eta}{Q}]_{ij} = c_{\eta} \prod_{k=1}^d{\left[\erf(\sqrt{\eta_k}{\cal B}_{ijk}) - \erf(\sqrt{\eta_k} {\cal A}_{ijk})\right]},
\end{equation}
where $c_{\eta} = (\pi/4)^{d/2}\det \diag (\eta)^{-1/2}$, ${\cal A}, {\cal B} \in \R^{d \times m \times m}$, ${\cal A}$ is the tensor of bottom left corners and ${\cal B}$ is the tensor of top right corners, defined formally from the means tensor $\overline{X}_{ijk} = \tfrac{1}{2}(X_{ik} + X_{jk})$ as
\begin{equation}
    \label{df:tensors}
     {\cal A}_{ijk} = a_k - \overline{X}_{ijk},\quad {\cal B}_{ijk} = b_k - \overline{X}_{ijk}.
\end{equation}

This shows that, for any hyper-rectangle $Q$, we can compute $\Imat{X}{\eta}{Q}$ with exactly  $2dm^2$ calls to the $\erf$ function and $dm^2$ arithmetic operations (so there is no dependence on the dimension of the hyper-rectangle).

\section{A sampling algorithm for PSD models}\label{sec:sampling}

In this section, we fix a Gaussian PSD model on $\R^d$ parametrized by $(A,X,\eta) \in \psdm^m \times \R^{m \times d} \times \R^d_{++}$ for a given $m \in \N$. To simplify notations, we will omit the parameters of the PSD model using $f(x)$ as a shorthand for $f(x;A,X,\eta)$ and $I(Q)$ as a shorthand of $I(Q) = I(Q; A, X, \eta)$.

Given a bounded hyper-rectangle $Q$ (see \cref{secsec:sampling}), denote by $p_Q$ the function
\eqal{
p_Q(x) =f(x)\ib_Q(x)/I(Q),
}
where $\ib_Q(x) = 1$ when $x \in Q$ and $0$ otherwise. In \cref{secsec:unbounded}, we explain that even in the case of an infinite hyper-rectangle (e.g., $Q = \R^d$), we can easily find a finite hyper-rectangle $\tilde{Q}$ on which the whole mass of $f$ is essentially concentrated, and thus approximately sample in this case as well. We end this section with a discussion on the main elements needed to sample, and which could allow to generalize this approach to PSD models with different kernels.

\subsection{A sampling algorithm on a finite hyper-rectangle}\label{secsec:sampling}

Given the function $f$, the algorithm will take three inputs $(Q,N,\rho)$: the hyper-rectangle $Q$ (with sides parallel to the axes) from which we would like to sample, the number of i.i.d.~samples $N$ which we would like to obtain, and a parameter $\rho$ which defines the quality of the approximation of $p_Q$ from which the algorithm generates samples. The effect of $\rho$ on the precision of the algorithm is formally established in \cref{thm:variation_bounds}.

We start with the case $N=1$. Starting from $Q$, we cut $Q$ in half in its longest direction forming two sub-rectangles $Q_1,Q_2$. If $X_Q$ were a random variable following the law of $p_Q$, then $X_Q \in Q_i$ with probability $p_i = I(Q_i)/I(Q)$, and $X_Q|\{X_Q \in Q_i\}$ follows the law of $p_{Q_i}$.  Therefore, when looking for a sample from $p_Q$, we  randomly choose with probability $p_i$ one of the two smaller sub-rectangles $Q_i$ in which to look for the sample and then call the algorithm recursively to get a sample from $p_{Q_i}$. Of course, we need a stopping criterion: when the maximal side of $Q$ has length smaller than $\rho$ then we stop and we return a point sampled uniformly at random in $Q$. The complete algorithm is presented in \cref{alg:sampling} and is explained below.

\paragraph{Details for \cref{alg:sampling}.} In \cref{line:sample_rec}, we define the recursive function {\normalfont \textproc{SampleRec}} which will generate samples recursively. The main algorithm {\normalfont \textproc{Sample}} in \cref{line:sample} simply calls the function {\normalfont \textproc{SampleRec}} and randomly reshuffles the samples in order to guarantee independence (see {\normalfont \textproc{RandomPerm}} \cref{line:reshuffle}). In \cref{line:maxlen}, the function {\normalfont \textproc{MaxLen}} applied to $Q$ returns the maximum of the lengths of the sides of $Q$; the condition can therefore be translated as ``if all sides of $Q$ are smaller than $\rho$''. If it is the case, in \cref{line:sample_unif}, we return $N$ i.i.d. samples from the uniform distribution on $Q$ using {\normalfont \textproc{SampleUniform}}. If it is not, in line \cref{line:cut} we cut the hyper-rectangle $Q$ in half along its largest side with minimal index (i.e., along side $k = \min\argmax{(b_i-a_i)}$), yielding two sub hyper-rectangles $Q_1,Q_2$. This is the purpose of the function {\normalfont \textproc{SplitLargestSide}}. In \cref{line:proba}, we compute the probability $q$ that a given sample from $p_Q$ belongs to $Q_1$ using the fact that we can integrate the PSD model exactly. Since we have to generate $N$ samples, we will select $k$ of them from $Q_1$ and $N-k$ from $Q_2$ where $k$ is a sample from a binomial law of paramter $q$: this is the purpose of the function {\normalfont \textproc{SampleBinomial}} and \cref{line:binom}. We then call the algorithm recursively to generate the $k$ samples from $Q_1$ using $p_{Q_1}$ and the $N-k$ samples from $Q_2$ from $p_{Q_2}$ (\cref{line:Q1samples,line:Q2samples}).

\begin{algorithm}
\caption{Approximately sampling from $p_Q$}\label{alg:sampling}
\begin{algorithmic}[1]
{\small
\Function{SampleRec}{$Q,N, \rho$} \label{line:sample_rec}
\If{$N = 0$} 
\State \textbf{return} \Call{EmptyList}{}
\ElsIf{ \Call{MaxLen}{$Q$} $\leq \rho$} \label{line:maxlen}
\State \textbf{return} \Call{SampleUniform}{$Q,N$} \label{line:sample_unif} 
\Else
\State $Q_1, Q_2 = $ \Call{SplitLargestSide}{$Q$} \label{line:cut}
\State $q = I(Q_1) / I(Q)$ \label{line:proba}
\State $k = $ \Call{SampleBinomial}{$N,q$} \label{line:binom}
\State $L_1 = $ \Call{SampleRec}{$Q_1,k, \rho$} \label{line:Q1samples}
\State $L_2 = $ \Call{SampleRec}{$Q_2,N - k, \rho$} \label{line:Q2samples}
\State \textbf{return} \Call{Concatenate}{$L_1, L_2$}
\EndIf
\EndFunction
\vspace{0.2cm}
\Function{Sample}{$Q,N, \rho$}\label{line:sample}
\State $L = $ \Call{SampleRec}{$Q,n,\rho$}
\State \textbf{return} \Call{RandomPerm}{$L$}\label{line:reshuffle}
\EndFunction
}
\end{algorithmic}
\end{algorithm}

\paragraph{Guarantees of the algorithm.} Given $(Q,N,\rho)$, \cref{alg:sampling} does not sample $N$ i.i.d.~samples from the exact distribution $p_Q$ but rather from an approximation $p_{Q,\rho}$ of $p_Q$, controlled by the parameter $\rho$. More formally, let $\dyad_{Q,\rho}$ be the set of dyadic sub-rectangles of $Q$ with largest possible size smaller than $\rho$ (see \cref{app:algorithm_proof} for a formal definition). Our algorithm will effectively sample from a piece-wise constant approximation of $p$ on the elements of $\dyad_{Q,\rho}$ : 
\begin{equation}\label{eq:dyadic_approx_density}
p_{Q,\rho}= \tfrac{1}{I(Q)}\sum_{Q_{\rho} \in \dyad_{Q,\rho}}{\tfrac{I(Q_\rho)}{|Q_\rho|}\ib_{Q_{\rho}}},
\end{equation}
where $\ib_{Q_{\rho}}$ is the indicator function of $Q_{\rho}$.
The guarantees of the algorithm are established in the following theorem, proved formally in \cref{app:proof_theorem_1}.

\begin{theorem}\label{thm:approximation_distribution} 
Given $(Q,N,\rho)$ where $Q$ is a bounded hyper-rectangle of $\R^d$, $\rho > 0$ and $N \in \N$, the function
{\normalfont \textproc{Sample}} in \cref{alg:sampling} returns $N$ i.i.d.~samples from the distribution $p_{Q,\rho}$ defined in \cref{eq:dyadic_approx_density}. 
Moreover, the number of integral computations of the form $I(\widetilde{Q})$ performed during the algorithm is bounded by $ N\log_2(|Q|) + Nd\log_2\tfrac{2}{\rho} + 1$, and the number of $\erf$ computations is $O(N~m^2 ~d~(\log_2(2|Q|) + d\log_2(2/\rho)))$, where $m$ is the dimension of the PSD model.
\end{theorem}



\paragraph{Approximation error of the algorithm.} Since by \cref{thm:approximation_distribution}, the algorithm does not generate samples exactly from $p_Q$ but rather from the piecewise constant approximation $p_{Q,\rho}$ defined in \cref{eq:dyadic_approx_density}, it is necessary to quantify the distance between $p_{Q}$ and its approximation $p_{Q,\rho}$. We do so in \cref{thm:variation_bounds} for three different distances.

The weakest distance will be the Wasserstein-1 distance (also called earth mover's distance) \citep{Santambrogio2015}. It quantifies the discrepancies in the allocation of mass between two distributions, and is defined as 
\begin{equation}
    \label{df:wasserstein1}
    \Wass_1(p_1,p_2) = \sup_{\Lip(f) \leq 1}{\left|\int_{\xx}{f(x)(p_1(x) - p_2(x))dx}\right|},
\end{equation}
where $\Lip(f)$ is the Lipschitz constant of $f$ for the Euclidean norm.
It is structurally the most adapted to the approximation $p_{Q,\rho}$ since on each hyper-rectangle of $\dyad_{Q,\rho}$, $p_{Q,\rho}$ has the same mass as $p_{Q}$ but distributes it uniformly. Hence, the discrepancy in mass allocation will be confined to small hyper-rectangles whose sides are of size at most $\rho$.

We will also use two stronger distances : the total variation (TV) distance $d_{TV}(p_1,p_2) = \|p_1-p_2\|_{L^1(\xx)}$, and the Hellinger distance $H(p_1,p_2) = \|\sqrt{p_1}-\sqrt{p_2}\|_{L^{2}(\xx)}$, which is particularly relevant for exponential models \citep{LeCam1990}, and, in our paper, when using rank-1 PSD models (see \cref{secsec:from_distribution}). These distances will naturally appear in \cref{sec:approximation_psd} to quantify the discrepancy between a given probability density and its approximation as a Gaussian PSD model.  For more details on these distances, see \cref{app:measure_proba_densities}. \cref{thm:variation_bounds} provides bounds on these distances between the target density $p_Q = f\ib_Q / I(Q)$ and $p_{Q,\rho}$ as a function of $\rho$, and some Lipschitz constant (where $
\Lip_{\infty}(g)$ denotes the Lipschitz constant of $g$ for the norm $\|x\|_{\infty} = \sup{|x_i|}$). A more general theorem is proved in \cref{app:thm_variation_bounds} as \cref{thm:variation_bounds_evolved}.

\bt[Variation bounds]\label{thm:variation_bounds}
Let $Q$ be a hyper-rectangle, $\rho > 0$, $p_Q = f\ib_Q/I(Q)$ and $p_{Q,\rho}$ defined in \cref{eq:dyadic_approx_density}. It holds: 
\begin{align}
&H(p_Q,p_{Q,\rho}) \leq \sqrt{\tfrac{|Q|}{I(Q)}} \Lip_{\infty}(\sqrt{f})~\rho
    \label{eq:bound_hellinger_approx_main}\\
    &d_{TV}(p_Q,p_{Q,\rho}) \leq \tfrac{|Q|}{I(Q)}\Lip_{\infty}(f) \rho\label{eq:bound_tv_approx_main}\\
    &\Wass_1(p_Q,p_{Q,\rho}) \leq \sqrt{d}\rho. \label{eq:bound_wass_approx_main}
\end{align}
\et 

Combining the result of \cref{thm:approximation_distribution,thm:variation_bounds}, we have that, given a PSD model on $m$ centers, an hyper-rectangle of interest $Q$ and an error $\rho$, \cref{alg:sampling} provides $N$ i.i.d.samples whose distribution is distant $\sqrt{d}\rho$ in terms of $\mathbb{W}_1$ from the density represented by the PSD model over the hyper-rectangle. In particular, \cref{alg:sampling} computes the $N$ i.i.d. samples with a cost of $O(N~m^2 ~d~(\log_2(2|Q|) + d\log_2(2/\rho)))$.

\begin{figure}[h]
    \centering
    \includegraphics[width =\textwidth]{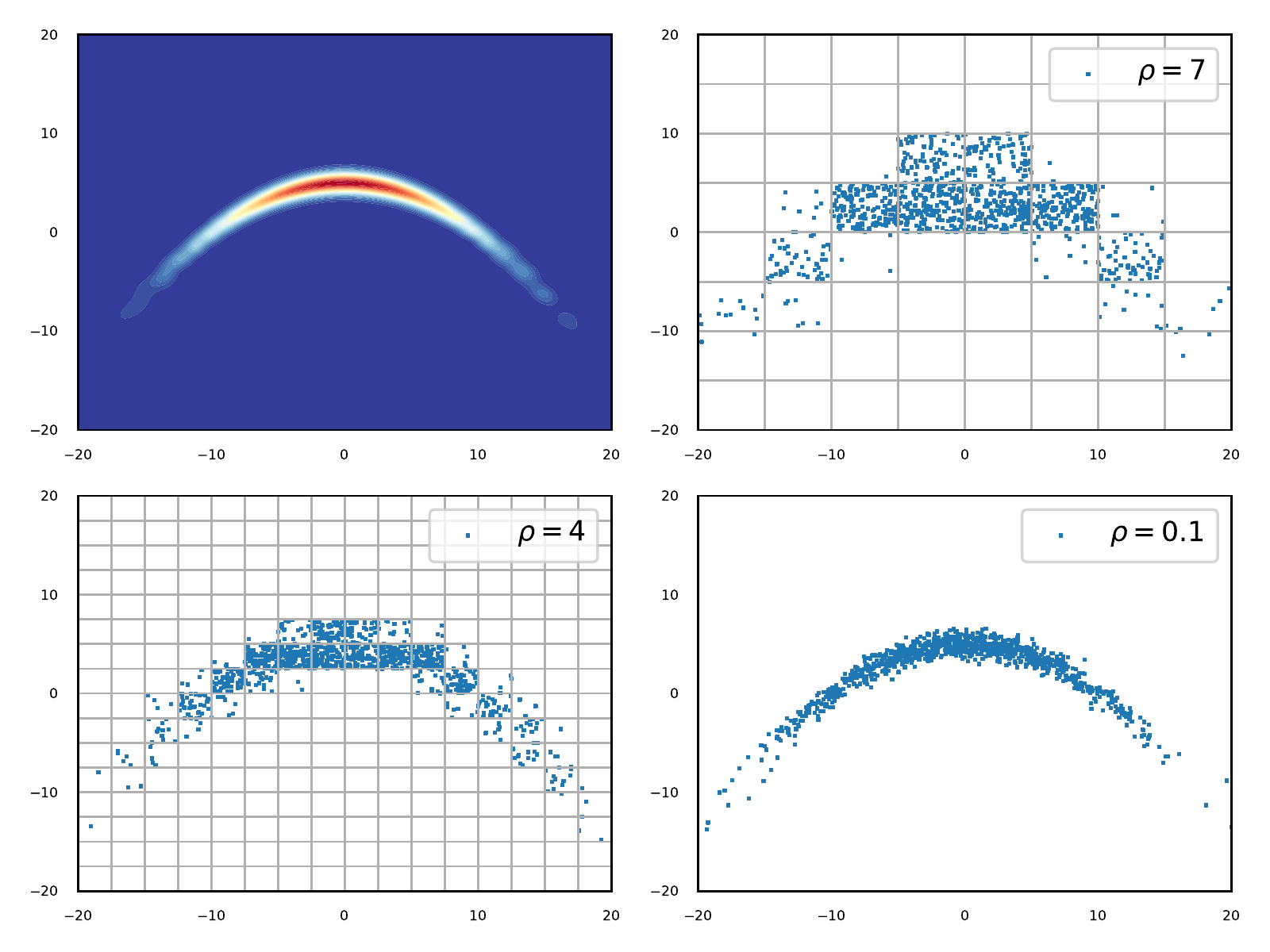}
    \vspace{-0.75cm}
    \caption{Samples obtained from \cref{alg:sampling} using different values for $\rho$}
    \label{fig:rho_selection}
\end{figure}

\paragraph{Selection of $\rho$.} In \cref{fig:rho_selection}, we observe the effect of $\rho$ on the quality of sampling, when sampling from a PSD model whose distribution is illustrated by the heat map defined on the top left figure. We highlight the fact that decreasing $\rho$ corresponds to refining the dyadic decomposition of the hyper-rectangle and hence sampling more precisely. In practice, one can therefore choose $\rho$ manually (for instance $\rho =10^{-4},10^{-6}$) and have an upper bound on the distance between $p_{Q,\rho}$ and $p_Q$ from \cref{thm:variation_bounds}. If one wishes to select $\rho$ in a more principled way to bound the total variation or Hellinger distance, this can also be done using only accessible quantities.
If $f$ is a PSD model with parameters $(A,X,\eta)$ for $\eta = \tau \ib_d$, and $K$ is a shorthand for $\Kmat{X}{\eta}$, the following Lipschitz constants can be bounded in the following way:
\begin{align}
    &\Lip_{\infty}(f) \leq \sqrt{8\tau}d \|K^{1/2}AK^{1/2}\| =: \Lipt(A) \label{eq:bound_lip_A}\\
    &\Lip_{\infty}(\sqrt{f}) \leq \sqrt{2\tau} d\|K^{1/2}a\| =:\Lipt(a)  \label{eq:bound_lip_a},
\end{align}
where for \cref{eq:bound_lip_a}, $A = aa^{\top}$ is assumed to be a rank-1 operator.\footnote{See \cref{lm:der_bound_psd} in \cref{app:bound_support_der} for a proof of \cref{eq:bound_lip_A} and \cref{lm:differential_gaussian_embedding} in \cref{app:properties_gaussian_kernel} for a proof of \cref{eq:bound_lip_a}.}
These quantities only depend on $a,A,K$ and can be computed explicitly. Combining \cref{eq:bound_lip_a,eq:bound_lip_A} with \cref{eq:bound_tv_approx_main,eq:bound_hellinger_approx_main}, we get the following adaptive ways to select $\rho$ in \cref{alg:sampling}.

\br[Adaptive selection of $\rho$]\label{rk:adaptive_rho}
Let $\eps > 0$. Let $f$ be a PSD model with matrix of
coefficients $A$. Define 
\begin{equation}
\label{eq:df_rho_adaptive}
    \rho_{\eps}^{TV} = \tfrac{I(Q) \eps}{|Q|\Lipt(A)},\qquad \rho_{\eps}^H = \tfrac{\sqrt{I(Q)} \eps}{\sqrt{|Q|}\Lipt(a)},
\end{equation}
where $\rho_{\eps}^H$ is defined if $A = aa^{\top}$ is a rank one matrix.
If $\rho = \rho_{\eps}^{TV}$ (resp. $\rho = \rho_{\eps}^H$), then \cref{alg:sampling} applied to $(Q,N,\rho)$ returns $N$ i.i.d.~samples from a distribution $p_{Q,\eps}$ which satisfies $d_{TV}(p_Q,p_{Q,\eps}) \leq \eps$ (resp. $H(p_{Q},p_{Q,\eps}) \leq \eps$).
\er

\subsection{Discussion}\label{secsec:unbounded}

\paragraph{Sampling from the distribution on $\R^d$.} 

It is possible to approximately sample from an infinite hyper-rectangle. To do so, one has to find a large enough hyper-rectangle $Q$ such that almost all the mass is contained on $Q$ and then apply the previous algorithm to this hyper-rectangle. One can, for instance, use \cref{alg:support}.

\begin{algorithm}
\caption{Finding an approximate support $Q$ }\label{alg:support}
\begin{algorithmic}
\Function{FindApproximateSupport}{$f(\cdot;A,X,\eta), \, \delta$}
\State $Q = \prod_{1 \leq k \leq d}{[\min_{1 \leq i \leq n}{X_{ik}},\max_{1 \leq i \leq n}{X_{ik}}]}$
\State $I = I(\R^d)$
\While {$I(Q)/I \leq 1-\eps $}
\State $Q = \Call{DoubleSize}{Q}$
\EndWhile
\EndFunction
\end{algorithmic}
\end{algorithm}

Note that one can also concentrate $f$ a priori using only its parameters $(X,A,\eta)$, using \cref{eq:tail_bound} of \cref{tail_bound_gaussian} in \cref{app:bound_support_der}.
One can use this bound to bound the number of steps in \cref{alg:support}.

\paragraph{Generality of the algorithm.} \cref{alg:sampling} only relies on the fact that one can compute integrals on hyper-cubes of the model $f$. If we were to replace the Gaussian kernel $k_\eta$ by a kernel $k$, and therefore have a PSD model of the form $\sum_{ij}{A_{ij}k(x,x_i)k(x,x_j)}$ with another positive definite kernel and $A \in \psdm^m$, then one would be able to run the algorithm as soon as computations of the form $\int_{Q}{k(x,x_i)k(x,x_j)dx}$ were tractable. This would extend this framework to more general PSD models, described by \citet{marteau20}.

\section{Sampling from any distribution using PSD models}\label{sec:approximation_psd}

The previous section provides an algorithm to approximately sample from a distribution in the form of a PSD model. In this section, we show how to leverage that fact to be able to generate $N$ approximate i.i.d.~samples from a very general class of probability distributions on a hyper-rectangle $\xx \subset \R^d$. The strategy is simple : a) approximate the target distribution $p$ with a PSD model, and b) approximately sample from that PSD model using the algorithm presented in \cref{sec:sampling}. The main challenge is to quantify the distance between the target distribution and the approximation in the PSD model.

Approaching a distribution by a PSD model by accessing the distribution through samples has been done in Sec. 3. of  \citet{rudi2021psd}. Instead, in this work, we access the distribution through function evaluations, as our goal is to be able to generate samples. However, a similar algorithm can be implemented and analysed to learn a PSD model from function evaluations under the same conditions (see \cref{asm:1b} and \cref{secsec:dtv}). In \cref{secsec:from_distribution}, we use a stronger assumption, adapted in particular to densities of the form $p(x) \propto e^{-V(x)}$ (see \cref{asm:1a}) which leads to a faster algorithm for learning, as the problem becomes a large scale least-squares problem (and not a semi-definite program) which can be solved using tools from \citet{rudi2015less,rudi2017falkon,meanti20}.

\paragraph{Assumptions.} In order to approximate the target distribution $p$ and to obtain guarantees on this approximation, we make assumptions on the distribution $p$ and in particular on its order of differentiability, parametrized by $\beta \in \N$. More formally, we will ask $p$ to be a sum of squares of functions belonging to the space $\Wt^{\beta}(\xx) = W^{\beta}_2(\xx) \cap L^{\infty}(\xx)$ which is the space of bounded functions whose derivatives of order less or equal to $\beta$ are square integrable, and which can be equipped with the norm $\|\cdot\|_{\Wt^{\beta}(\xx)} = \|\cdot\|_{W^{\beta}_2(\xx)} + \|\cdot\|_{L^\infty(\xx)}$ (see \cref{app:sobolev_spaces} for more precise definitions).  

We will assume that $\xx  = (-1, 1)^d$ in this section for simplicity, as is done by \citet{rudi2021psd}. In principle, we could approximate $p$ on any bounded domain $\xx$ from which we can sample uniformly, and still obtain analog results. In that case, we would apply \cref{alg:sampling} on a hyper-rectangle containing the domain, and reject a sample outside of it. In \cref{secsec:from_distribution}, we will use the following assumption.
\begin{ass}[Square distribution]\label{asm:1a}
There exists a function $q$ belonging to $\Wt^{\beta}(\xx)$ such that $p = q^2$. Moreover, we have access to $p$ only through function evaluations of the form $\gp(x)$, where $\gp \propto q$ and where the proportionality constant is unknown. 
\end{ass}
Note that this assumption is satisfied if $p \propto e^{-V(x)}$ for a potential $V$ which is $\beta$ times continuously differentiable which we can evaluate.
On the other hand, in \cref{secsec:dtv}, we will use the same assumption (up to the evaluation part) as \citet{rudi2021psd}, and which is more general. However, in this case, finding the approximate PSD model requires   solving   a semi-definite program.
\begin{ass}[Sum of squares distribution]\label{asm:1b}
There exists $J \in \N$ and functions $q_1,...,q_J$ belonging to $\Wt^{\beta}(\xx)$ such that $p = \sum_{j=1}^J{q_j^2}$. Moreover, we have access to $p$ only through function evaluations of the form $\fp(x)$ where $\fp \geq 0$ is given, is proportional to $p$, and where the proportionality constant is unknown. We define $\Nsos{p}{\xx,\beta} = \inf \sum_{j=1}^J{\|q_j\|^2_{\Wt^{\beta}(\xx)}}$ where the infimum is taken over all such decompositions of $p$.
\end{ass}
\paragraph{Main parameters.} For the rest of the section, we will take two sequences of i.i.d.~samples taken uniformly from $\xx$ :  $x_1,...,x_n$ represented by  $X \in \R^{n \times d}$ and $\xt_1,...,\xt_m$ represented by $\wtx_m \in \R^{m\times d}$. The parameter $\eta$ used in the Gaussian linear and PSD models will always be isotropic, i.e., of the form $\eta = \tau \ib_d$ for a strictly positive $\tau$. To simplify notation, take $K_{mm} := \Kmat{\wtx_m}{\eta}$ and $K_{nm} := \Kmatrix{X}{\wtx_m}{\eta}$. The parameter $\la$ will always be a strictly positive real number. The parameter $n$ will control the number of points at which we evaluate our probability density to estimate it; the parameter $m$ will control the number of points, also called \textit{Nystr\"{o}m centers}, which we use to represent our PSD model (as $n$ and $m$ increase, the quality of the approximation increases); the parameter $\tau$ will control the width of the Gaussian kernel and must be adapted to the number of points $n$ selected to cover the space $\xx$, and the parameter $\la$ will be used to control the regularity of our approximation of $p$ by a PSD model. In the following propositions and theorems, we will give values for the parameters to show there exists values (or minimal values in the case of $n$ and $m$) for which our algorithms will reach a given precision $\eps$. However, in practice and in the experiments, we will usually fix $n$ and $m$ (see this as a computational budget) and select $\la$ and $\tau$ using validation techniques. 

\paragraph{Interpretation of the results.} In the following \cref{secsec:dtv,secsec:from_distribution}, we present approximation results of the form $d(p,\psample) \leq \|p\|~\eps$, where $\|p\|$ is a certain norm on $p$ defined in each section, showing that when the parameters are selected in a certain way, one can reach $\eps$ precision. The main points we want to highlight are the following. 
\begin{enumerate}[wide,labelindent=0pt]
    \item Even though we only have access to the distribution up to a re-normalizing constant, this does not influence the theoretical results, i.e., the bounds we get only depend on the density $p$ through its norm $\|p\|$.
    \item The dependence of the parameters in the target error $\eps$ improves with smoothness. More precisely,  the number of samples $n,m$ needed (and hence the complexity of the sampling and of the approximation algorithm) is polynomial in the quantities $O(\eps^{-1}), O(\eps^{-d/\beta})$, showing that as soon as $\beta \geq d$, the dimension plays no role in the exponents of these error terms and thus {\em breaking the curse of dimensionality} in the rates. However, the constants in the $O(\cdot)$ term can be exponential in $d$, and without more hypotheses, \textbf{they are unimprovable} \citep{novak2006deterministic}. We therefore keep a form of ``curse of dimensionality'' in the constants, but not in the rate. Concretely this means that we need a number of points in the order of the constants before having a reasonable error (i.e., $\eps = 1$). However, as soon as this number is reached, one can rapidly gain in precision, if the function is regular. Moreover, in practice, we do not always pay this exponential constant, owing to some additional regularity of the function. Interestingly, this phenomenon is shared with approximation, learning and optimization problems over a wide family of functions (see  \citep{novak2006deterministic} for more details). 
\end{enumerate}

\paragraph{Algorithm.} \cref{alg:sampling_general} implements the procedure described in \cref{secsec:from_distribution}, where the function {\normalfont \textproc{SolveHellinger}} in \cref{line:hellinger} solves \cref{eq:problem_hellinger} in order to approximate $p$ with a rank one PSD model. 

\begin{algorithm}
\caption{Approximately sampling any distribution}\label{alg:sampling_general}
{\small
\hspace*{\algorithmicindent} \textbf{Input} $p,Q,N$\\
\hspace*{\algorithmicindent}\textbf{Parameters} (approximation) $n,m,\tau,\lambda$ \\
\hspace*{\algorithmicindent}\textbf{Paramters} (sampling) $\rho$ \\
\hspace*{\algorithmicindent} \textbf{Output} $N$ approximate samples from $p|_{Q}$}
\begin{algorithmic}[1]
{\small
\Function{ApproximateSamples}{$p,Q,n,m,\tau,\lambda,\rho,N$}
\State $X_n$ = \Call{UniformSamples}{$n,Q$}
\State $X_m$ =\Call{UniformSamples}{$m,Q$}
\State $a = $ \Call{SolveHellinger}{$p,X_n,X_m,\tau,\lambda$}\label{line:hellinger}
\State $\widehat{p}(\cdot) = f(\cdot\,|\,aa^{\top},X_m,\tau)$
\State $X_N =$ \Call{Sample}{$Q,N,\rho$} from $\widehat{p}$
\State \textbf{return} $X_N$
\EndFunction
}
\end{algorithmic}
\end{algorithm}

\subsection{A general method}\label{secsec:dtv}

In this section, we present a method to approximately sample from the density by approximating it by a PSD model solving a semi-definite program. It can be solved in polynomial time in the problem dimension $m$.
We use an method similar to the one presented in section 3 of  \citet{rudi2021psd}; \cref{asm:1b} under which guarantees on the conciseness and quality of the approximation hold is quite general and encompasses many cases (see Assumption 1 and Proposition 5 of the same work).
Under \cref{asm:1b}, let $\fp \in \Wt^{\beta}(\xx)$ be the non-negative function proportional to the density $p$. We construct a Gaussian PSD model $\fhat = \pp{\bullet}{\Ahat,\wtx_m,\eta}$, where $\Ahat \in \psdm^m$ is the solution to the empirical semi-definite problem
\begin{align}\label{eq:empirical_problem}
    \widehat{A} = \argmin{A \in \psdm^m}\int_{\xx}{\pp{x}{A}^2 dx}\nonumber\\
    -2 \sum_{i=1}^n{\fp(x_i)\pp{x_i}{A}} + \lambda \|K^{1/2}_{mm}AK^{1/2}_{mm}\|^2_F,
\end{align}
where $\pp{x}{A} := \pp{x}{A,\wtx_m,\eta}$. This problem is a quadratic problem in $A$ and can be solved in polynomial time in $m$ using semi-definite programming. Let $\Zhat = \int_{\xx}{\fhat(x)~dx}$ which can be computed in closed form as the integral over a hyper-cube of a PSD model. Let $\phat = \fhat/\Zhat$ be the re-normalized version of $\fhat$. The approximation properties of $\phat$ w.r.t. $p$ are bounded in total variation distance in the following proposition, proved as \cref{thm:performance_pgauss_app} in \cref{app:proof_general_method}. Although the optimization criterion \cref{eq:empirical_problem} is an empirical version of the $L^2$ distance between $\fp$ and $f$, we present results in total variation distance as it is more suited for probability measures.

\bp[Performance of $\pgauss$]\label{thm:performance_pgauss} There exist constants $\eps_0 > 0$ depending only on $d,\beta$, and $\Nsos{p}{\xx,\beta}$ and $C_1, \Cp_1,\Cp_2,\Cp_3$ depending only on $d,\beta$ such that the following holds.
Let $\delta \in (0,1]$ and $\eps \leq \eps_0$, and assume $n$ and $m$ satisfy 
\begin{align}
&m \geq \Cp_1 \eps^{-d/\beta} \log^d \left(\tfrac{\Cp_2}{\eps}\right) \log \left(\tfrac{\Cp_3}{\eps\delta} \right) , \label{eq:bound_m_main_psd}\\
    &n \geq \eps^{-2 -d/\beta}\log^{d}\left(\tfrac{1}{\eps}\right)\log\left(\tfrac{2}{\delta}\right)\label{eq:bound_n_mmain_psd}.
\end{align}
Let $\lambda = \eps^{2+2d/\beta}$ and  $\tau = \eps^{-2/\beta}$. With probability at least $1-2\delta$, it holds
\begin{equation}\label{eq:error_p_gauss_main}
    d_{TV}(\phat,p) \leq C_1~\Nsos{p}{\xx,\beta}~\eps.
\end{equation}
\ep

 Let $\psample$ be the dyadic approximation of $\phat$ on $ \xx = (-1,1)^d$ and of width $\rho$ (see \cref{eq:dyadic_approx_density}). Applying \cref{alg:sampling} with $\phat$ to $(Q,N,\rho)$ where $Q = \xx$ returns $N$ i.i.d. samples from $\psample$ by \cref{thm:approximation_distribution}. We provide a choice of $\rho$ in order to guarantee a bound for the total variation distance in the following theorem. It is proved as \cref{thm:performance_p_sample_psd_app} in \cref{app:proof_general_method}.

\bt[Performance of $\psample$]\label{thm:performance_p_sample_psd} Under the assumptions and notations of \cref{thm:performance_pgauss}, there exists a constant $C_2$ depending only on $d,\beta$, such that the following holds. If $\rho$ is set either as $\eps^{1 + (d+1)/\beta}$ or adaptively as $\rho^{TV}_{\eps}$, then with probability at least $1-2\delta$,
 \begin{equation}
    \label{eq:bound_final_psd_main}
    d_{TV}(p,\psample) \leq C_2~\Nsos{p}{\xx,\beta}~ \eps.
\end{equation}
Moreover, the adaptive $\rho^{TV}_{\eps}$ is lower bounded by $ \eps^{1+(d+1)/\beta}/(C_3~\Nsos{p}{\xx,\beta})$. In both cases, this guarantees that the complexity in terms of $\erf$ computations is of order $O(Nm^2\log(1/\rho))$, which in terms of $\eps$ yields
    $O\left(N~\eps^{-2d/\beta} \log^{2d+1}\left(\tfrac{1}{\eps}\right) \log^2\left(\tfrac{1}{\delta \eps}\right)\right)$,
where the $O$ notations is taken with constants depending on $d,\beta$, $\Nsos{p}{\xx,\beta}$.
\et

\subsection{Efficient method with a rank one model}\label{secsec:from_distribution}

In this section, we present a method to approximately sample from the density $p$ by approximating it by a PSD model solving a linear system. This simpler and faster method comes at the expense of the stronger \cref{asm:1a} needed to provide guarantees. In this setting, we can rely on the very extensive work which has already been done on kernel least squares problems both in terms of theoretical bounds and practical algorithms \citep{caponnetto2007,rudi2015less,rudi2017generalization,rudi2017falkon,marteau19glob,meanti20}. This algorithm is written in pseudo-code in \cref{alg:sampling_general}.

Let $\gp \in \Wt^{\beta}(\xx)$ such that $\gp^2 \propto p$. To approximate $p$ with a PSD model, we start by approximating $\gp$ by a linear model $\ghat=\ppl{\bullet}{\ahat,\wtx_m,\eta}$ (see \cref{eq:df_gaussian_linear}), where $\ahat \in \R^m$ is the solution to the empirical problem
\begin{equation}\label{eq:problem_hellinger}
   \min_{a \in \R^m}{\tfrac{1}{n}\sum_{i=1}^n{\left|\ppl{x_i}{a} - \gp(x_i)\right|^2} + \la a^\top K_{mm} a },
\end{equation}
where $\ppl{x}{a} := \ppl{x}{a,\wtx_m,\eta}$ and $g_n = (\gp(x_i))_{1 \leq i \leq n}$. $\ahat$ is the solution to the system :
\begin{equation}
    \label{a:system}
    \left(K_{nm}^{\top} K_{nm} + (\lambda n)K_{mm}\right)a = K_{nm}^\top g_n,
\end{equation}

which can be solved either directly in time $O(nm^2 + m^3)$ \citep{rudi2015less} or using a pre-conditioned conjugate gradient method in time $O(m^3 + nm)$ \citep{rudi2017falkon,meanti20,marteau19glob}. We then define $\fhat = \ghat^2$ which is a rank-1 PSD model with coefficients $\Ahat = \ahat \ahat^{\top}$, $\Zhat = \int_{\xx}{\fhat(x)dx} =  \|\ghat\|_{L^2(\xx)}^2$ which is computable in closed form as the integral of a PSD model (see \cref{eq:int_psd}), and our rank 1 PSD approximation $\phat = \fhat/\Zhat$ of $p$. Note that given the form of the empirical problem \cref{eq:problem_hellinger}, it is natural to measure the distance between $p$ and $\phat$ using the Hellinger distance $H$, defined in \cref{secsec:sampling}. The following proposition shows the approximation performance of $\phat$, and is proved as \cref{thm:bound_learning_hellinger_app} in \cref{app:target_distribution}.

\bp[Performance of $\phat$]\label{thm:bound_learning_hellinger}
Let $\nut > \min(1,d/(2\beta))$. There exists a constant $\eps_0$ depending only on $\|q\|_{\Wt^{\beta}(\xx)},\beta,d$, constants $C_1,C_2,C_3,C_4$ depending only on $\beta,d$ and a constant $\Cp_1$ depending only on $\beta,d,\nut$ such that the following holds.

Let $\delta \in (0,1]$ and $\eps \leq \eps_0$, and assume $m$ and $n$ satisfy
\begin{align}
    &m \geq C_1 \eps^{-d/\beta}\log^d\left(\tfrac{C_2}{ \eps}\right) \log\tfrac{C_3}{\delta \eps}\label{eq:bound_m_final_hell_main}\\
    &n \geq \Cp_1 \eps^{-2\nut} \log \tfrac{8}{\delta} \label{eq:bound_n_final_hell_main}
\end{align}
Let $\tau = \eps^{-2/\beta}$ and $\la = \eps^{2 + d/\beta}$. 
With probability at least $1-3\delta$, it holds 
\begin{equation}
\label{eq:error_p_gauss_hellinger_main}
    H(\phat,p) \leq C_4\|q\|_{\Wt^{\beta}(\xx)}~\eps.
\end{equation}
\ep

 Let $\psample$ be the dyadic approximation of $\phat$ on $ \xx = (-1,1)^d$ and of width $\rho$ (see \cref{eq:dyadic_approx_density}).  \cref{thm:approximation_distribution} shows that \cref{alg:sampling} applied to $(Q,N,\rho)$ where $Q = \xx$ returns $N$ i.i.d. samples from $\psample$. The following theorem provides a choice of $\rho$ in order to guarantee a bound for the Hellinger distance. It is proved as \cref{thm:performance_p_sample_hellinger_app} in \cref{app:target_distribution}.

\bt[Performance of $\psample$]\label{thm:performance_p_sample_hellinger} Under the assumptions and notations of \cref{thm:bound_learning_hellinger}, there exists a constant $C_5$ depending only on $d,\beta$, such that the following holds.
If on the one hand $\rho$ is set either as $\eps^{1+(d+2)/(2\beta)}$ or adaptively as $\rho^{H}_{\eps}$ (see \cref{rk:adaptive_rho}), then with probability at least $1-3\delta$,
\begin{equation}
    \label{eq:bound_final_hellinger_1_main}
\qquad H(p,\psample) \leq C_5\|q\|_{\Wt^{\beta}(\xx)}~\eps.
\end{equation}
Moreover, the adaptive $\rho^H_{\eps}$ is lower bounded by $\eps^{1+(d+2)/\beta}/(C_5~\|q\|_{\Wt^{\beta}(\xx)})$. In both cases, this guarantees that the complexity in terms of $\erf$ computations is bounded by $O(N m^2\log \tfrac{1}{\rho})$, which, in terms of $\eps$, yields
$ O\left(N~\eps^{-2d/\beta}  \log^{2d+1}\left(\tfrac{1}{\eps}\right) \log^2\left(\tfrac{1}{\delta \eps}\right) \right)$
where the $O$ notation incorporates constants depending on $d,\beta$, $\|q\|_{\Wt^{\beta}(\xx)}$.
\et 

\section{Experiments}\label{sec:experiments}
The experiments in this work were executed on a mac-book pro equipped with a 2,8 GHz Quad-Core Intel Core i7 processor and 16Gb of RAM. 

\paragraph{Qualitative performance of our algorithm.}
In \cref{fig:example_learning_sampling}, we show an example of the way our algorithm approximates a certain target density $p_1$ known up to a renormalization constant: $p_1(x) \propto 0.08 k_{0.7}(x,-1)-0.4k_{0.6}(x,1) + 0.4 k_{0.7}(x,1)$. In the top left figure, a heat map of $p_1$ is plotted. We then use \cref{alg:sampling_general} to approximate $p_1$ by a rank one PSD model $\widehat{p}_1$ (whose heat-map is plotted on the top right figure) and then sample $N = 1000$ samples from this approximation (plotted in the bottom left figure). Note that in order to approximate $p_1$ by $\widehat{p}_1$, $n = 10^5,m=300$ were fixed and $\gamma=2,\la=10^{-9}$ were selected on a test set.

\begin{figure}[h]
    \centering
    \includegraphics[width =\textwidth]{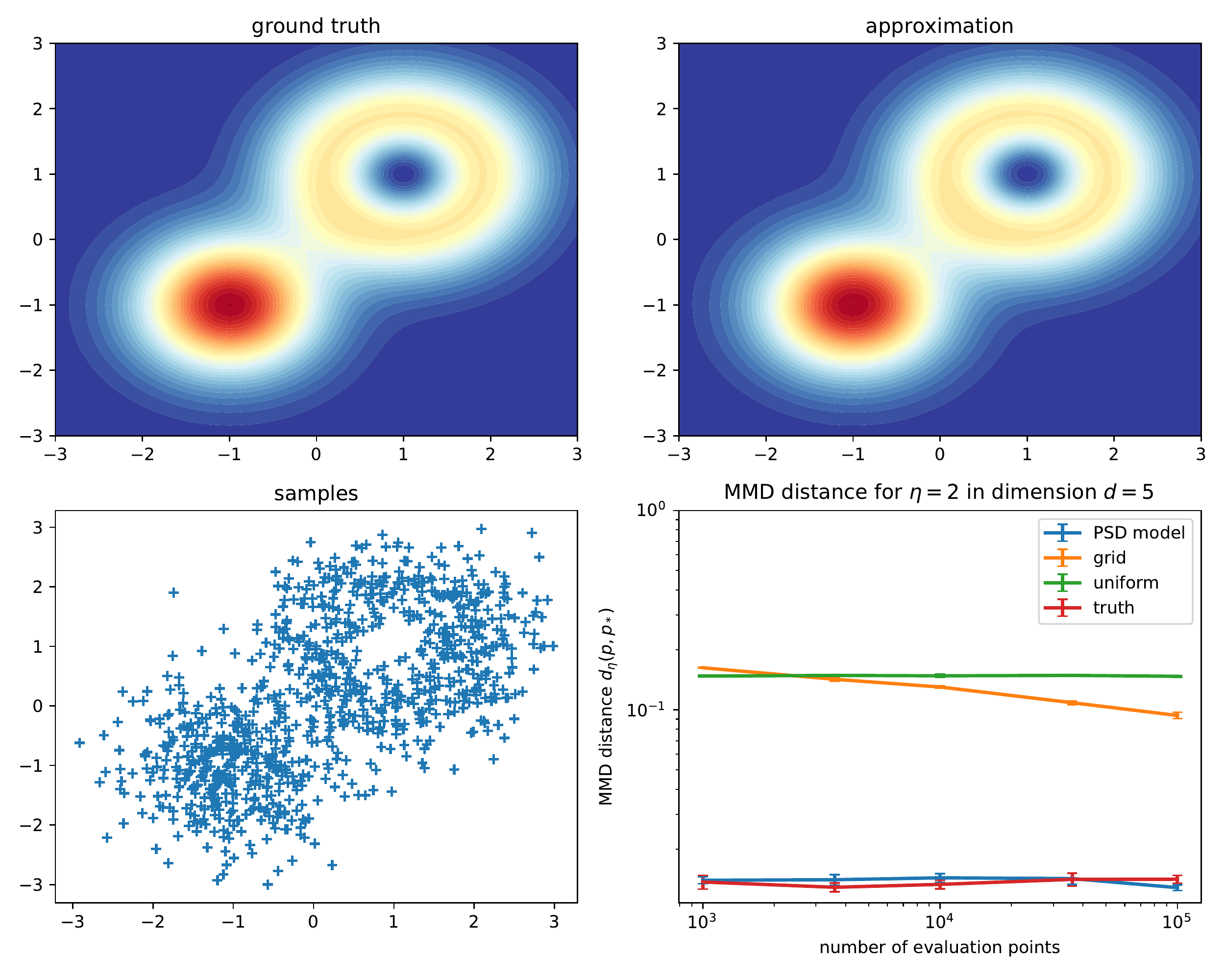}
    \vspace{-0.75cm}
    \caption{\textit{(top left)} Plot of the distribution $p_1$, \textit{(top right)} heat map of an approximation $\widehat{p}_1$ of $p_1$. \textit{(bottom left)} samples generated from $\widehat{p}_1$, \textit{(bottom right)} performance of our method in MMD distance.}
    \label{fig:example_learning_sampling}
    \vspace{-0.2cm}
\end{figure}

\paragraph{Quantitative performance of our algorithm.} To further demonstrate the promising nature of our sampling algorithm, we tried learning the density $p_2(x) \propto (k_{1/5}(x,1) - k_{1/5}(x,-1))^2$ on $Q = [-1,1]^5$. As this is a PSD model, we can sample from it with very high precision (here, we chose $\rho = 10^{-6}$).

We compared the performance of our model to the naive gridding algorithm which, if allowed $n$ function evaluations, computes a grid  $G$ of side $n^{1/d}$, which we identify to the set of centers of the tiles of the grid, and evaluates $p$ at each point in the grid. To sample a point, one chooses a point $g \in G$ with probability $p(g)/\sum_{h \in G}{p(h)}$, and then draws a sample uniformly in that tile. It is the algorithm called 'grid' in the bottom right figure of \cref{fig:example_learning_sampling}.

We compare our algorithm with the gridding algorithm by fixing the number $n$ of function evaluations of $p$ each method is allowed, and computing the distance between each method and the ground truth. The distance we use between distributions is the empirical version of the  Maxmium Mean Discrepancy distance (MMD) \citep{sriperumbudur10a,sriperumbudur11}, which is defined, for the Gaussian kernel $k_{\eta}$ of parameter $\eta$, as 
$d_{\eta}(p,\tilde{p}) = \left\|\mathbb{E}_{X\sim p}[\phi_{\eta}(X)] - \mathbb{E}_{X\sim \tilde{p}}[\phi_{\eta}(X)]\right\|_{\hh_{\eta}}$
where $\phi_{\eta}$ is the embedding associated to the gaussian kernel $k_{\eta}$ (for more details, see \cref{app:def_notations}).
This distance can be approximated using $N$ samples $(x_i)_{1 \leq i \leq N}$ from $p$ and  $N$ samples $(\tilde{x}_j)_{1 \leq j \leq N}$ from $\tilde{p}$ as 
$\widehat{d}_{\eta}(p,\tilde{p}) = \left\|\tfrac{1}{N}\sum_{i=1}^N{\phi_{\eta}(x_i)} - \tfrac{1}{N}\sum_{j=1}^N{\phi_{\eta}(\tilde{x}_j)}\right\|_{\hh_{\eta}}$. This quantity can be computed explicitly using kernel matrices \citep{sriperumbudur10a}. However,  
\citet{tolstikhin16} show that the minimax rate cannot exceed $1/\sqrt{N}$, i.e., that  $\widehat{d}_{\eta}$ approximates $d_{\eta}$ only with precision of order $1/\sqrt{N}$.

In our experiments, we take $N =10^4$. We compute the empirical distances $\widehat{d}_{\eta}$ five times using newly generated samples from each distribution, and compute an empirical mean and standard deviation, reported as error bars on the plot. When approximating $p_2$ by a PSD model using \cref{alg:sampling_general}, we take $m=50$, as there is no need to increase $m$ to reach better precision than the target distribution for $\widehat{d}_{\eta}$. We take $\rho = 10^{-3}$ and select $\tau,\lambda$ by using half of the evaluation points as a test set.   

The results reported on the bottom-right plot of \cref{fig:example_learning_sampling} show that in dimension $5$, the 'grid' method is not competitive anymore, and is close to the uniform distribution in performance for $\eta =2$. Note that the choice of $\eta$ in a wide range from $0.1$ to $10$ does not change these results. They also show that when taking only $N=10^4$ to approximate the MMD distance, our method is indistinguishable from the target measure itself (also reported in the plot to make the effect linked to the approximation of $d_\eta$ by $\widehat{d}_{\eta}$ apparent), even when $n=1000$ (this is because both distances are smaller than the minimax bound of order $1/\sqrt{N}$). Hence, due to this precision of the method, we are as of yet unable to evaluate precisely enough its performance to illustrate precisely its evolution in terms of $n$ and $m$.

\section{Extensions and future work}\label{sec:extensions}
In this paper, we have introduced a method for sampling any distribution from function values by first approximating it with a so-called PSD model and then sampling from this PSD model using the algorithm introduced in \cref{sec:sampling}.

Natural extensions of this work include the fact that while we cast a least squares problem in \cref{secsec:dtv}, we can actually minimize more general convex losses adapted to distributions, such as maximum log-likelihood estimation. Moreover, as mentioned in \cref{sec:sampling}, the proposed algorithm only relies on integral computations, and could therefore be extended to other kernels, provided they can easily be integrated on hyper-rectangles.

Future work will start with trying to scale the sampling method up in terms of generation of samples, by both theoretical means (to make computation saving approximations) and computational means (use of GPUs, parallelization). 

\paragraph{Acknowledgements.}
This work was funded in part by the French government under management of Agence Nationale de la Recherche as part of the “Investissements d’avenir” program, reference ANR-19-P3IA-0001(PRAIRIE 3IA Institute). We also acknowledge support from the European Research Council (grants SEQUOIA 724063 and REAL 947908), and support by grants from Région Ile-de-France.

\bibliography{biblio_arxiv.bib}

\appendix

\newpage

{\Huge{Organization of the Supplementary Material}}

\begin{itemize}
    \item [\large{\textbf{\ref{app:def_notations}.}}]\hyperref[app:def_notations]{{\textbf{\large{Definitions and Notations}}}} \vspace{0.2cm} \\
    We set the main notations and tools of the appendix (Fourier transform, vector and matrix notations, notations concerning hyper-rectangles, RKHS and specifically the Gaussian kernel).
    \begin{itemize}
        \item[{\textbf{\ref{app:sobolev_spaces}.}}] \hyperref[app:sobolev_spaces]{\textbf{{Sobolev spaces}}} \\ 
        In this section, we focus more on notations and basic results concerning Sobolev spaces, as they will be our main tool to measure the regularity of a function.
        \item[{\textbf{\ref{app:measure_proba_densities}.}}] \hyperref[app:measure_proba_densities]{\textbf{{Measuring distances between probability densities}}}\\
        In this section, we define and compare the basic distances we will be using to compare probability distributions in the paper, since we are always "approximating" a certain distribution with another. In particular, we define the total variation, Hellinger and Wasserstein distances.
        \item[{\textbf{\ref{app:general_psd_models}.}}] \hyperref[app:general_psd_models]{\textbf{{General PSD models}}}\\
        We define PSD models in general \citep{marteau20,rudi2021psd}. They will be our main tool for approximation and sampling, and relates to the more restrictive definition in \cref{sec:psd_models}.
    \end{itemize}
    \vspace{0.1cm} 
    \item [\large{\textbf{\ref{app:gaussian_kernel}.}}]\hyperref[app:gaussian_kernel]{{\textbf{\large{Properties of the Gaussian RKHS}}}} \vspace{0.2cm} \\
    Throughout the paper the Gaussian kernel $k_{\eta}$ and the associated Gaussian RKHS will be central objects. We introduce different properties and results.
    \begin{itemize}
        \item[{\textbf{\ref{app:properties_gaussian_kernel}.}}] \hyperref[app:properties_gaussian_kernel]{\textbf{{Properties of the Gaussian kernel $k_{\eta}$}}} \\
        We introduce certain properties of the Gaussian kernel involving products, as well as a bound on the derivative of the associated embedding in \cref{lm:differential_gaussian_embedding}.
        \item[{\textbf{\ref{app:linop}.}}] \hyperref[app:linop]{\textbf{{Useful Matrices and Linear Operators on the Gaussian RKHS}}}\\
        We introduce the most important theoretical objects of the paper. We introduce kernel matrices, matrices which will appear in the integration of Gaussian PSD models, operators which relate $L^2$ to the RKHS $\hhe$, operators which allow to discretize using samples and "compression" operators which allow concise representations.
        \item[{\textbf{\ref{app:approximation_gaussian_kernel}.}}] \hyperref[app:approximation_gaussian_kernel]{\textbf{{Approximation properties of the Gaussian kernel}}}\\
        We prove two important results concerning the approximation properties of the Gaussian RKHS in \cref{prp:approximation_eps} and the concise representation of models in \cref{lm:bound_projection}.
    \end{itemize}
    \vspace{0.1cm}
    \item[\large{\textbf{\ref{app:gaussian_kernel_psd}.}}]\hyperref[app:gaussian_kernel_psd]{{\textbf{\large{Properties of Gaussian PSD models}}}}\vspace{0.2cm}\\
    We present the results specific to Gaussian PSD models. These results are often reformulations of theorems presented by \citet{rudi2021psd}. 
    \begin{itemize}
        \item [{\textbf{\ref{app:bound_support_der}.}}] \hyperref[app:bound_support_der]{\textbf{{Bounds on the support and the derivatives}}} \\
        We present result to understand how the mass of a Gaussian PSD model is concentrated (\cref{tail_bound_gaussian}) and how the derivative of a Gaussian PSD model can be bounded using only its parameters (\cref{lm:der_bound_psd}).
        \item[{\textbf{\ref{app:compression_gaussian_psd}.}}] \hyperref[app:compression_gaussian_psd]{\textbf{{Compression as a Gaussian PSD model}}} \\
        We restate Theorem C.4 of \citet{rudi2021psd} as \cref{thm:compression_psd_model_bound} on the effect of a compression operator on a PSD model.
        \item[{\textbf{\ref{app:approximation_psd}.}}] \hyperref[app:approximation_psd]{\textbf{{Approximation properties of Gaussian PSD model}}} \\
        We refine Theorem D.4 of \citet{rudi2021psd} in \cref{thm:D4} in order to approximate a sum of squares using a PSD model on the Gaussian RKHS $\hhe$.
    \end{itemize}
    \vspace{0.1cm}
    \item[\large{\textbf{\ref{app:algorithm_proof}.}}]\hyperref[app:algorithm_proof]{{\textbf{\large{The sampling algorithm}}}}\vspace{0.2cm}\\
     We prove that the sampling algorithm indeed returns $N$ i.i.d. samples from the right distribution, and characterize the distance between the sampling distribution and the original PSD distribution.
    \begin{itemize}
        \item[{\textbf{\ref{app:dyad_dec_convergence}.}}] \hyperref[app:dyad_dec_convergence]{\textbf{{Dyadic decompositions and convergence of \cref{alg:sampling}}}} \\
        We formally prove that \cref{alg:sampling} finishes and returns $N$ samples from a distribution characterized by a structural induction formula (see \cref{lm:termination_algorithm}).
        \item[{\textbf{\ref{app:proof_theorem_1}.}}] \hyperref[app:proof_theorem_1]{\textbf{{Proof of \cref{thm:approximation_distribution}}}} \\
        We prove \cref{thm:approximation_distribution} by structural induction, showing that when the samples are randomly shuffled, we end up with $N$ i.i.d. samples from the distribution defined in \cref{eq:dyadic_approx_density}. This is done by matching the distribution with the one from the previous section using a structural induction.
        \item[{\textbf{\ref{app:thm_variation_bounds}.}}] \hyperref[app:thm_variation_bounds]{\textbf{{Evaluating the error of the sampling algorithm : proof of \cref{thm:variation_bounds}}}} \\
        We prove \cref{thm:variation_bounds} in \cref{thm:variation_bounds_evolved}, bounding the distance between the distribution of the PSD model and the actual distribution from which \cref{alg:sampling} samples (see \cref{eq:dyadic_approx_density}). This is done in different distances, all related to the problem in different way (Wasserstein is the most adapted in spirit, but we also need stronger distances such  as total variation and Hellinger, which can be bounded using Lipschitz constants of the PSD models).
    \end{itemize}
    \vspace{0.1cm}
    \item [\large{\textbf{\ref{app:proof_general_method}.}}]\hyperref[app:proof_general_method]{{\textbf{\large{A general method of approximation and sampling}}}} \vspace{0.2cm} \\
    We prove that we can approximate any probability distribution satisfying \cref{asm:1b} using non necessarily normalized function values, by solving \cref{eq:empirical_problem} with the right parameters in \cref{thm:performance_pgauss_app} which is labeled in the main text as 
\cref{thm:performance_pgauss}. We then show that applying \cref{alg:sampling} with the right value of $\rho$ yields a good sampling algorithm from a good approximation of the distribution.
This proves \cref{thm:performance_p_sample_psd} and is proved here as \cref{thm:performance_p_sample_psd_app}. 
    \vspace{0.1cm}
    \item [\large{\textbf{\ref{app:target_distribution}.}}]\hyperref[app:target_distribution]{{\textbf{\large{Approximation and sampling using a rank one PSD model}}}} \vspace{0.2cm} \\
    We prove that we can approximate any probability distribution satisfying \cref{asm:1a} using non necessarily normalized function values, by solving \cref{eq:problem_hellinger} with the right parameters in \cref{thm:bound_learning_hellinger_app} which is labeled in the main text as 
\cref{thm:bound_learning_hellinger}. This has an advantage compared to the previous method which is that the approximation phase is much faster (it solves a linear system instead of an SDP). We then show that applying \cref{alg:sampling} with the right value of $\rho$ yields a good sampling algorithm from a good approximation of the distribution.
This proves \cref{thm:performance_p_sample_hellinger} and is proved here as \cref{thm:performance_p_sample_hellinger_app}. 
    \vspace{0.1cm}
\end{itemize}
\vspace{1cm}

\section{Definitions and Notations}\label[appendix]{app:def_notations}

In this section we recall results from \citet{rudi2021psd} which will be useful in the different statements and proofs.

\paragraph{Basic vector and matrix notations.}
Let $n,d \in \N$. We denote by $\R^d_{++}$ the space vectors in $\R^d$ with positive entries,  $\R^{n \times d}$ the space of $n \times d$ matrices, $\psdm^n=\psdm(\R^n)$ the space of positive semidefinite $n \times n$ matrices. Given a vector $\eta\in\R^d$, we denote $\diag(\eta)\in\R^{d \times d}$ the diagonal matrix associated to $\eta$. We denote by $A \circ B$ the entry-wise product between two matrices $A$ and $B$. We denote by $\|A\|_{F}, \|A\|, \det(A), \vect(A)$ and $A^\top$ respectively the Frobenius norm, the operator norm (i.e.  maximum singular value), the determinant, the (column-wise) vectorization of a matrix and the (conjugate) transpose of $A$. With some abuse of notation, where clear from context we write element-wise products and division of vectors $u,v\in\R^{d}$ as $uv, u/v$. The term $\ib_n\in\R^n$ denotes the vector with all entries equal to $1$.  

\paragraph{Hyper-rectangles}

Define a hyper-rectangle $Q$ as a product of the form $\prod_{k=1}^d{[a_k,b_k[}$, where $a \leq b$. Given a hyper-rectangle $Q$ we denote its extremities with $a(Q) \leq b(Q) \in \R^d$ (i.e. $Q = \prod_{k=1}^d{[a_k(Q),b_k(Q)[}$), and its side-lengths $\rho(Q)  = b(Q)-a(Q)$. We sometimes omit $Q$ when it is implied by the context.

We will also use the so-called \textit{error function}, which is defined as follows : 
\[\erf(x) = \tfrac{2}{\sqrt{\pi}}\int_{0}^{x}{e^{-t^2~dt}}.\]

This function is implemented as an elementary function in most libraries.

\paragraph{Multi-index notation} Let $\alpha \in \N^d$, $x \in \R^d$ and $f$ be an infinitely differentiable function on $\R^d$, we introduce the following notation
$$|\alpha| = \sum_{j=1}^d \alpha_i, \quad \alpha! = \prod_{j=1}^d \alpha_j!, \quad x^\alpha = \prod_{j=1}^d x_j^{\alpha_j}, \quad \partial^\alpha f = \frac{\partial^{|\alpha|} f}{\partial x_1^{\alpha_1}\cdots\partial x_d^{\alpha_d}}.$$
We introduce also the notation $D^\alpha$ that corresponds to the multivariate distributional derivative of order $\alpha$ and such that 
$$D^\alpha f = \partial^\alpha f$$
for functions that are differentiable at least $|\alpha|$ times \citep{adams2003sobolev}.

\paragraph{Fourier Transform} 
Given two functions $f,g:\R^d \to \R$ on some set $\R^d$, we denote by $f \cdot g$ the function corresponding to {\em pointwise product} of $f, g$, i.e.,
$$(f \cdot g)(x) = f(x)g(x), \quad \forall x \in \R^d.$$
Let $f, g \in L^1(\R^d)$ we denote the {\em convolution} by $f \star g$ 
$$(f \star g)(x) = \int_{\R^d} f(y) g(x-y) dy.$$
We now recall some basic properties, that will be used in the rest of the appendix.
\bp[Basic properties of the Fourier transform \citep{wendland2004scattered}, Chapter 5.2.]\label{prop:fourier}
$ $

\begin{enumprop}
\item\label{prop:fourier:L2} There exists a linear isometry $\FT: L^2(\R^d) \to L^2(\R^d)$ satisfying 
$$\FT[f] = \int_{\R^d} e^{-2 \pi i \,\omega^\top x} \,f(x)\, dx \quad  \forall f \in L^1(\R^d) \cap L^2(\R^d),$$
where $i = \sqrt{-1}$. The isometry is uniquely determined by the property in the equation above. 
\item\label{prop:fourier:plancherel} Let $f \in L^2(\R^d)$, then $\|\FT[f]\|_{L^2(\R^d)} = \|f\|_{L^2(\R^d)}$.
\item\label{prop:fourier:scale} Let $f \in L^2(\R^d), r > 0$ and define $f_r(x) = f(\frac{x}{r}), \forall x \in \R^d$, then $\FT[f_r](\omega) = r^d \FT[f](r\omega)$. 
\item\label{prop:fourier:product} Let $f, g \in L^1(\R^d)$, then $\FT[f \cdot g] =  \FT[f] \star \FT[g]$.
\item\label{prop:fourier:derivative} Let $\alpha \in \N^d$,  $f, D^\alpha f \in L^2(\R^d)$, then
$\FT[D^\alpha f](\omega) = (2\pi i)^{|\alpha|} \omega^\alpha \FT[f](\omega)$, $\forall \omega \in \R^d$.
\item\label{prop:fourier:Linfty-omega} Let $f \in L^1(\R^d) \cap L^2(\R^d)$, then $\|\FT[f]\|_{L^\infty(\R^d)} \leq \|f\|_{L^1(\R^d)}$.
\item\label{prop:fourier:Linfty-x} Let $f \in L^\infty(\R^d) \cap L^2(\R^d)$, then $\|f\|_{L^\infty(\R^d)} \leq \|\FT[f]\|_{L^1(\R^d)}$.
\end{enumprop}
\ep

\paragraph{Reproducing kernel Hilbert spaces for translation invariant kernels.} We now list some important facts about reproducing kernel Hilbert spaces in the case of translation invariant kernels on $\R^d$. For this paragraph, we refer to \citet{steinwart2008support,wendland2004scattered}. For the general treatment of positive kernels and Reproducing kernel Hilbert spaces, see \citet{aronszajn1950theory,steinwart2008support}.
Let $v:\R^d \to \R$ such that its Fourier transform $\FT[v] \in L^1(\R^d)$ and satisfies $\FT[v](\omega) \geq 0$ for all $\omega \in \R^d$. Then, the following hold.
\begin{enumerate}[label=(\alph*)]
\item The function $k:\R^d \times \R^d \to \R$ defined as $k(x,x') = v(x-x')$ for any $x,x' \in \R^d$ is a positive kernel and is called {\em translation invariant kernel}.
\item The {\em reproducing kernel Hilbert space} (RKHS) $\hh$ and its norm $\|\cdot\|_{\hh}$ are characterized by 
\eqal{\label{eq:tr-inv-rkhs-def}
\hh = \{f \in L^2(\R^d) ~|~ \|f\|_{\hh} < \infty \}, \quad \|f\|^2_{\hh} = \int_{\R^d} \frac{|\FT[f](\omega)|^2}{\FT[v](\omega)}d\omega,
}
\item $\hh$ is a separable Hilbert space, whose inner product $\scal{\cdot}{\cdot}_{\hh}$ is characterized by
$$\scal{f}{g}_{\hh} = \int_{\R^d} \frac{\FT[f](\omega)\overline{\FT[g](\omega)}}{\FT[v](\omega)} d\omega.$$
In the rest of the paper, when clear from the context we will simplify the notation of the inner product, by using $f^\top g$ for $f,g \in \hh$, instead of the more cumbersome $\scal{f}{g}_\hh$.
\item The feature map $\phi:\R^d \to \hh$ is defined as $\phi(x) = k(x-\cdot) \in \hh$ for any $x \in \R^d$.
\item The functions in $\hh$ have the {\em reproducing property}, i.e.,
\eqals{\label{eq:reproducing-property}
f(x) = \scal{f}{\phi(x)}_\hh, \quad \forall f \in \hh, x \in \R^d,
}
in particular $k(x',x) = \scal{\phi(x')}{\phi(x)}_\hh$ for any $x',x \in \R^d$.
\end{enumerate}

We now introduce the main tool of our analysis, the Gaussian RKHS, which will be further explored in \cref{app:gaussian_kernel_psd}.

\begin{example}[Gaussian Reproducing Kernel Hilbert Space]\label{ex:gaussian-rkhs}
Let $\eta \in \R^d_{++}$ and $k_\eta(x,x') = e^{-(x-x')^\top\diag(\eta)(x-x')}$, for $x,x' \in \R^d$ be the Gaussian kernel with precision $\eta$. The function $k_\eta$ is a translation invariant kernel, since $k_\eta(x,x') = v(x-x')$ with $v(z) = e^{-\|D^{1/2}z\|^2}, D = \diag(\eta)$ and $\FT[v](\omega) = c_\eta e^{-\pi^2 \|D^{-1/2}\omega\|^2}$, $c_\eta = \pi^{d/2} \det(D)^{-1/2}$, for $\omega \in \R^d$ is in $L^1(\R^d)$ and satisfies $\FT[v](\omega) \geq 0$ for all $\omega \in \R^d$. The associated reproducing kernel Hilbert space $\hhe$ is defined according to \cref{eq:tr-inv-rkhs-def}, with norm
\eqal{
\|f\|^2_{\hhe} ~~=~~ \frac{1}{c_\eta} \int_{\R^d} ~|\FT[f](\omega)|^2 ~e^{\pi^2 \|D^{-1/2}\omega\|^2} ~d\omega, \qquad \forall f \in L^2(\R^d).
}
The inner product and the feature map $\phie$ are defined as in the discussion above.
\end{example}

\subsection{Sobolev spaces}\label[appendix]{app:sobolev_spaces}
Let $\beta \in \N, p \in [1,\infty]$ and let $\Omega \subseteq \R^d$ be an open set. The set $L^p(\Omega)$ denotes the set of $p$-integrable functions on $\Omega$ for $p \in [1,\infty)$ and that of the essentially bounded on $\Omega$ when $p = \infty$.
The set $W^\beta_p(\Omega)$ denotes the Sobolev space, i.e., the set of measurable functions with their distributional derivatives up to $\beta$-th order belonging to $L^p(\Omega)$,
\eqal{\label{eq:norm-sobolev-derivative}
W^\beta_p(\Omega) = \{f \in L^p(\Omega) ~|~\|f\|_{W^\beta_p(\Omega)} < \infty\}, \quad \|f\|^p_{W^\beta_p(\Omega)} = \sum_{|\alpha| \leq \beta} \|D^\alpha f\|^p_{L^p(\Omega)},
}
where $D^\alpha$ denotes the distributional derivative. In the case of $p = \infty$, 
\eqals{
\|f\|_{W^\beta_\infty(\Omega)} = \max_{|\alpha| \leq \beta} \|D^\alpha f\|_{L^\infty(\Omega)}
}

We now recall some basic results about Sobolev spaces that are useful for the proofs in this paper.
First we start by recalling the restriction properties of Sobolev spaces. Let $\Omega \subseteq \Omega' \subseteq \R^d$ be two open sets. Let $\beta \in \N$ and $p \in [1,\infty]$. By definition of the Sobolev norm above we have
$$\|g|_\Omega\|_{W^s_p(\Omega)} \leq \|g\|_{W^s_p(\Omega')},$$
and so $g|_\Omega \in W^s_p(\Omega)$ for any $g \in W^s_p(\Omega')$. Now we recall the extension properties of Sobolev spaces, which will allow us to consider the case 

The formal definition of a set with Lipschitz boundary is provided by \citet{adams2003sobolev}. Note that if $\xx = (-1,1)^d$, as will be the case later on for simplicity, then $\xx$ is bounded and has Lipschitz boundary.

The following result shows that being in an intersection space allows to extend the function to the whole of $\R^d$. This will be useful in order to use the properties of translation invariant kernels in order to approximate functions which are a priori defined only on $\xx$ but which we extend using this result. 

\bp[Corollary A.3 of \citet{rudi2021psd}]\label{cor:extension-intersection}
Let $\xx \subset \R^d$ be a non-empty open set with Lipschitz boundary. Let $\beta \in \N, p \in [1,\infty]$. Then for any function $f \in W^\beta_p(\xx) \cap L^\infty(\xx)$ there exists an extension $\tilde{f}$ on $\R^d$, i.e. a function $\tilde{f} \in W^\beta_p(\R^d) \cap L^\infty(\R^d)$ such that 
\eqals{
f = \tilde{f}|_\xx ~\textrm{a.e. on } \xx, \quad \|\tilde{f}\|_{L^\infty(\R^d)} \leq C \|f\|_{L^\infty(\xx)}, \quad \|\tilde{f}\|_{W^\beta_p(\R^d)} \leq C' \|f\|_{W^\beta_p(\xx)}.
}
The constant $C$ depends only on $\xx, d$, and the constant $C'$ only on $\xx,\beta,d,p$ 
\ep

The following proposition gives an idea of what these intersection spaces contain. 

\bp[Proposition A.4 of \citet{rudi2021psd}.]\label{prop:inclusion-Cb-in-Wb-Linfy}
Let $\xx$ be an open bounded set with Lipschitz boundary. Let $f$ be a function that is $m$ times differentiable on the closure of $\xx$. Then there exists a function $\tilde{f} \in W^m_p(\xx) \cap L^\infty(\xx)$ for any $p \in [1,\infty]$, such that $\tilde{f} = f$ on $\xx$.
\ep

The following proposition provides a useful characterization of the space $W^\beta_2(\R^d)$ in terms of Fourier transform; this will be particularly useful when approximating functions in $W^{\beta}_2(\R^d)$ by functions in a Gaussian RKHS $\hhe$ using the characterization of the norm in terms of Fourier transform for those kernels in \cref{eq:tr-inv-rkhs-def}.
\begin{proposition}[Characterization of the Sobolev space $W^k_2(\R^d)$, \citet{wendland2004scattered}, Proposition A.5 of \citet{rudi2021psd}]\label{prop:sobolev}
Let $k \in \N$. The norm of the Sobolev space $\|\cdot\|_{W^k_2(\R^d)}$ is equivalent to the following norm
\eqals{
\|f\|'^{\,2}_{W^k_2(\R^d)} = \int_{\R^d} ~|\FT[f](\omega)|^2 ~(1+\|\omega\|^2)^{k} ~d\omega, \quad \forall f \in L^2(\R^d)
}
and satisfies
\eqal{\label{eq:sobolev-norm-char}
\tfrac{1}{(2\pi)^{2k}} \|f\|^2_{W^k_2(\R^d)} \leq \|f\|'_{W^k_2(\R^d)}~\leq~ 2^{2k} \|f\|^2_{W^k_2(\R^d)}, \quad \forall f \in L^2(\R^d)
}
Moreover, when $k > d/2$, then $W^k_2(\R^d)$ is a reproducing kernel Hilbert space.
\end{proposition}

\subsection{Measuring distances between probability densities}\label[appendix]{app:measure_proba_densities}

In this work, since our aim is to approximate a probability distribution, we will often compare probability distributions, with different distances. 

To simplify definitions, we will only consider distances between probability densities $p_1,p_2$ defined on a Borel subset $\xx$ of $\R^d$ with respect to the Lebesgue measure. Note that while the total variation distance, the Hellinger distance and the Wasserstein distance do not actually depend on the choice of such a base measure and can be defined intrinsically, the $L^2$ distance cannot; that is why it is less appropriate from a statistical point of view. We consider it here because it is the natural distance in which we are able to solve \cref{eq:empirical_problem}.

\paragraph{The total variation (TV) or $L^1$ distance}:
    \begin{equation}
        \label{df:dtv}
        d_{TV}(p_1,p_2) := \|p_1-p_2\|_{L^1(\xx)}= \int_{\xx}{|p_1(x)-p_2(x)|dx}.
    \end{equation}
    This distance can also be expressed using a dual formulation (see Chapter 3.2 of \citet{LeCam1990}).
    \begin{equation}
        \label{df:dual_tv}
        d_{TV}(p_1,p_2) = \sup_{|f|\leq 1}{\left|\int_{\xx}{f(x)(p_1(x)-p_2(x))~dx}\right|}
    \end{equation}
\paragraph{The Hellinger distance}: (this distance is particularly suitable in the case of exponential models; see \citet{LeCam1990} and in particular Chapter 3).
    \begin{equation}
        \label{df:hellinger}
        H(p_1,p_2) := \|\sqrt{p_1}-\sqrt{p_2}\|_{L^2(\xx)} =  \left(\int_{\xx}{|\sqrt{p_1}(x)-\sqrt{p_2}(x)|^2~dx}\right)^{1/2}
    \end{equation}
\paragraph{The Wasserstein distance} In the case where $\xx$ is bounded (for simplicity), the $p$ Wasserstein distance for $p \geq 1$ (see chapter 5 of \citet{Santambrogio2015}):
    \begin{equation}
        \label{df:wasserstein}
        \Wass_p^p(p_1,p_2) = \inf_{\gamma\in \Pi(p_1,p_2)}{\int_{\xx \times \xx}{|x-y|^p~d\gamma(x,y)}},
    \end{equation}
    where $\Pi(p_1,p_2)$ is the set of all probability measures on $\xx \times \xx$ with marginals $p_1$ and $p_2$. Note that one has the following easier dual formulation when $p=1$ (see the chapter on Kantorovich duality by \citet{Santambrogio2015}):
    \begin{equation}
        \label{df:wasserstein_1}
        \Wass_1(p_1,p_2) = \sup_{f \in \Lip_1(\xx)}{\int_{\xx}{f(x)(p_1(x)-p_2(x))dx}},
    \end{equation}
    where $\Lip_1(\xx)$ is the set of $1$-Lipschitz functions on $\xx$. Wasserstein distances capture the moving of mass; they are quite weak but are well-adapted to capture the behavior of our sampling algorithm which approximates probability densities on each hyper-rectangle.
\paragraph{The $L^2$ distance}:
    \begin{equation}
        \label{df:l2dist}
        \|p_1 - p_2\|_{L^2(\xx)} = \left(\int_{\xx}{(p_1(x)-p_2(x))^2~dx}\right)^{1/2}
    \end{equation}

\paragraph{Relating these difference distances}. The following well known bounds exist between distances. 

\begin{equation}
    \label{eq:hellingerl1}
    H^2(p_1,p_2) \leq d_{TV}(p_1,p_2) \leq \sqrt{2}H(p_1,p_2).
\end{equation}
Moreover, if $\xx$ is bounded, we have for any $p \geq 1$, using the Holder inequality:
\begin{align}
    &\Wass_p(p_1,p_2) \leq \Diam(\xx)^{(p-1)/p}\Wass_1(p_1,p_2)^{1/p}\label{eq:bound_wass_wass},\\
    &\Wass_1(p_1,p_2) \leq \Diam(\xx)d_{TV}(p_1,p_2)\label{eq:bound_wass_tv},\\
    & d_{TV}(p_1,p_2) \leq |\xx|^{1/2}\|p_1-p_2\|_{L^2(\xx)},\label{eq:boundtv2}
\end{align}

where $\Diam(\xx)$ denotes the diameter of the set $\xx$.

\subsection{General PSD models}\label[appendix]{app:general_psd_models}
In this section, we recall the definition of a PSD model more generally as introduced by \citet{rudi2021psd}.

Following \citet{marteau20,rudi2021psd}, we consider the family of positive semi-definite (PSD) models, namely non-negative functions parametrized by a feature map $\phi:\xx\to\hh$ from an input space $\xx$ to a suitable feature space $\hh$ (a separable Hilbert space e.g. $\R^q$) and a linear operator $M\in\psdm(\hh)$, of the form 
\eqal{\label{eq:psd-models-general}
   \pp{x}{\mm,\phi}  = \phi(x)^\top \mm~\phi(x).
}
PSD models offer a general way to parametrize non-negative functions (since $\mm$ is positive semidefinite, $\pp{x}{\mm,\phi}\geq 0$ for any $x\in\xx$) and enjoy several additional appealing properties discussed in the following. 
%
%
In this work. we focus on a special family of models i.e. Gaussian PSD models defined in \cref{sec:psd_models} and \cref{eq:df_gaussian_psd}. These models parametrize probability densities over $\xx \subset \R^d$. It is a special case of \cref{eq:psd-models-general} where $i)$ $\phi = \phie:\R^d\to\hhe$ is a feature map associated to the Gaussian kernel defined in \cref{ex:gaussian-rkhs}, or by \citet{scholkopf2002learning} and, $ii)$ the operator $\mm$ lives in the span of $\phi(x_1),\dots, \phi(x_n)$ for a given set of points $(x_i)_{i=1}^n$, namely there exists $A\in\psdm(\R^n)$ such that $\mm = \sum_{ij} A_{ij} \phie(x_i) \phie(x_j)^\top$. 

Thus, given the triplet $(A,X,\eta)$ characterizing the Gaussian PSD model in \cref{eq:df_gaussian_psd}, we have 
\begin{align*}
\sum_{1 \leq i,j\leq n}{A_{ij}k_{\eta}(x,x_i)k_{\eta}(x,x_j)} = f(x; A,X,\eta) &= f(x;M,\phie) \\
M&= \sum_{1 \leq i,j \leq n}{A_{ij}\phie(x_i)\otimes \phie(x_j)},
\end{align*}

 where $(u\otimes v)w= u v^{\top} w = \scal{v}{w}u$.

\section{Properties of the Gaussian RKHS}\label[appendix]{app:gaussian_kernel}

In this section, we introduce notations and results associated to the Gaussian RKHS (see \cref{ex:gaussian-rkhs}) $\hhe$ for a given $\eta \in \R_{++}^d$ ($\eta$ will sometimes be taken in the form $\tau \ib_d$). Recall that the Gaussian embedding is written $\phie : \R^d \rightarrow \hhe$ and that the Gaussian kernel is denoted with $k_{\eta}$.

\subsection{Properties of the Gaussian kernel $k_{\eta}$}\label[appendix]{app:properties_gaussian_kernel}

The following lemma has an immediate proof.

\blm[product of gaussian kernels] Let $K \in \N$,
let $\eta_1,...,\eta_K \in \R^d_{++}$ and let $y_1,...,y_K \in \R^d$. The following equality holds:

\[\forall x \in \R^d,~ \prod_{k=1}^K{k_{\eta_k}(x,y_k)}=k_{\overline{\eta}}(x,\overline{y})\prod_{k=1}^K{k_{\eta_k}(y_k,\overline{y})}\]
where $\overline{\eta} = \sum_{k=1}^K{\eta_k}$ and $\overline{y} = \sum_{k} \eta_k y_k/\overline{\eta}$

\elm 

Let us now state an useful corollary.

\bcor 
Let $\eta \in \R^d_{++}$, $y_1,y_2 \in \R^d$. Then 
\begin{equation}
\label{eq:formula_prod}
    \forall x \in \R^d,~ k_\eta(x,y_1)k_\eta(x,y_2) = k_{2\eta}(x,(y_1+y_2)/2)k_{\eta/2}(y_1,y_2).
\end{equation}

\ecor

\blm[Gaussian embedding derivative]\label{lm:differential_gaussian_embedding}
Let $\eta \in \R_{++}^d$, $x \in \R^d$ and $\alpha \in \N^d$. The derivative $\partial_{\alpha}\phie(x)$ is well defined in $\hhe$, and $\|\partial_{\alpha} \phie(x)\|_{\hhe} = 2^{|\alpha|/2}\eta^{\alpha/2}$. Moreover, if $g \in \hhe$, then $\sup_{x\in\R^d}{|(\partial_{\alpha}g)(x)|} \leq 2^{|\alpha|/2}\eta^{\alpha/2}\|g\|_{\hhe}$.
\elm

\begin{proof}
Let $\alpha \in \N^d$ and let $v_{\eta}(z) = k_{\eta}(z,0) = \exp(-z^{\top}\diag(\eta)z)$. If the function $\tfrac{\partial^\alpha}{\partial x^\alpha} k_\eta(x,y)$ belongs to $\hhe$, then $\partial_{\alpha} \phie(x)$ is in $\hhe$ and is equal to that function by the reproducing property.

First, note that 
\[\forall x,y \in \R^d,~ \tfrac{\partial^\alpha}{\partial x^\alpha} k_\eta(x,y) = (-1)^{|\alpha|}\partial_{\alpha}\tau_{x}[v_{\eta}](y),\]
where $\tau_{x} : f \mapsto f(\cdot - x)$, commutes with the differential operator $\partial_{\alpha}$, and satisfies the following relation wrt to the Fourier transform : $\FT[\tau_x g](\xi) = e^{-2i \pi x \xi} \FT[f](\xi)$. Hence, using (e) of \cref{prop:fourier}, we get the following fourier transform wrt $y$:

\[\FT_y[\tfrac{\partial^\alpha}{\partial x^\alpha} k_\eta(x,y)](\xi) = (-2\pi i)^{|\alpha|} \xi^{\alpha} e^{-2 i \pi \xi x}\FT[v_\eta](\xi).\]

Hence, we have using \cref{eq:tr-inv-rkhs-def}:

\begin{align*}
    \|\tfrac{\partial^\alpha}{\partial x^\alpha} k_\eta(x,\cdot)\|_{\hhe}^2  &= \int_{\R^d}{(2  \pi)^{2|\alpha|} \xi^{2\alpha} \FT[v_{\eta}](\xi)~d\xi} \\
    &= (-1)^{|\alpha|}\int_{\R^d}{(2 i \pi)^{2|\alpha|} \xi^{2\alpha} \FT[v_{\eta}](\xi)~d\xi}\\ &=(-1)^{|\alpha|}\int_{\R^d}{ \FT[\partial_{2\alpha} v_{\eta}](\xi)~d\xi} =(-1)^{|\alpha|} \partial_{2\alpha}v_{\eta}(0),
\end{align*}

where the last equality comes from the inverse Fourier transform. A simple recursion then shows that $(-1)^{|\alpha|}\partial_{2\alpha}v_{\eta}(0) = 2^{|\alpha|}\eta^{\alpha}$, hence the result. The last point of the lemma is simply a consequence of the fact that $\partial_{\alpha}g(x) = \scal{g}{\partial_{\alpha}\phie(x)}_{\hhe}$. 

\end{proof}

    \subsection{Useful Matrices and Linear Operators on the Gaussian RKHS}\label[appendix]{app:linop}

Recall that we denote with $\phie$ the embedding associated to the RKHS $\hhe$ of the Gaussian kernel $k_{\eta}$ defined in \cref{ex:gaussian-rkhs}. In this section, we define operators which will be useful throughout the rest of this section and which we will use in \cref{app:proof_general_method,app:target_distribution}. In order to make the dependence in $\eta$ appear (indeed, $\eta$ will be a parameter to choose in the the next sections), we will keep it as an index for all of these operators. Recall that for any two vectors $u,v$ in a Hilbert space $\hh$, we can define their tensor product $u \otimes v$ which is a linear rank one operator on $\hh$ defined by  $ (u\otimes v)w = \scal{v}{w}_{\hh} u$. For the sake of simplicity, we will often write $u \otimes v$ as $uv^{\top}$, so that the formula $(u\otimes v)w = uv^{\top}w$ is formally true.

\paragraph{Kernel matrices.} We start off by setting the notations for kernel matrices as done by \citet{rudi2021psd}. Let $X \in \R^{n\times d}$ and $X^{\prime} \in \R^{n^{\prime} \times d}$ be two matrices corresponding to points $x_1,...,x_n \in \R^d$ and $x^{\prime}_1,...,x^{\prime}_{n^{\prime}} \in \R^d$. We denote with $\Kmatrix{X}{X^{\prime}}{\eta}$ the matrix in $\R^{n \times n^{\prime}}$ such that 
\begin{equation}
    \label{eq:df_kernel_matrix}
   \forall 1 \leq i \leq n,~ \forall 1 \leq j \leq n^{\prime},~ [\Kmatrix{X}{X^{\prime}}{\eta}]_{ij} = k_{\eta}(x_i,x^{\prime}_j).
\end{equation}
If $X = X^{\prime}$, then we just write $\Kmat{X}{\eta}$ and it is positive semi-definite, i.e. $\Kmat{X}{\eta} \in \psdm{(\R^n)}$.

\paragraph{Integration matrices.} In this work, we also define, for a given hyper-rectangle $Q = \prod_{k=1}^d[a_k,b_k]$, the following integration matrix $\Imatrix{X}{X^{\prime}}{\eta}{Q} \in \R^{n \times n^{\prime}}$:

\begin{align}
\label{eq:df_integration matrix}
&\forall 1 \leq i \leq n,~ \forall 1 \leq j \leq n^{\prime},~ [\Imatrix{X}{X^{\prime}}{\eta}{Q}]_{ij} = \int_{Q}k_{\eta}{(x-(x_i + x_j)/2)dx} \nonumber \\
&= \prod_{k=1}^d{\sqrt{\tfrac{\pi}{4 \eta_k}}\left(\erf(\sqrt{\eta_k} (b_k + (x_{ik} + x^{\prime}_{jk})/2) )- \erf(\sqrt{\eta_k} (a_k+(x_{ik} + x^{\prime}_{jk})/2))\right)},
\end{align}

where the $\erf$ function is defined in the notations section. Similarly, if $X = X^{\prime}$, we simply write $\Imat{X}{\eta}{Q}$.

This matrix is defined in order to satisfy the following property, which is a direct application of \cref{eq:formula_prod}: for any $X \in \R^{n\times d}$, any $A \in \psdm^n$ and $\eta \in \R^d_{++}$, the following holds.

\begin{equation}
\label{eq:integration_with_matrices}
    \int_{Q}{\pp{x}{A,X,\eta}~dx} = \sum_{1 \leq i,j \leq n}{[A \circ \Kmat{X}{\eta/2} \circ \Imat{X}{2\eta}{Q}]_{ij}} = \vect(A \circ \Kmat{X}{\eta/2} \circ \Imat{X}{2\eta}{Q})^{\top}\ib_{n^2}
\end{equation}

\paragraph{Co-variance operator.} Let $\xx \subset \R^d$ be a measurable set of $\R^d$ with finite Lebesgue measure $|\xx|$. Define the associated co-variance operator:
\begin{equation}\label{eq:df_cov}
\Cop \in \psdm(\hhe),\qquad \Cop = \tfrac{1}{|\xx|}{\int_{\xx}{\phie(x)\otimes \phie(x)~dx}},\qquad \Cl = \Cop + \la I.
\end{equation}

Note that $\Cop$ is a trace class operator with and that $\tr{(\Cop)} = 1$ by linearity of the trace and since $\tr{(\phie(x)\otimes \phie(x))} = \|\phie(x)\|^2 = k_{\eta}(x,x) = 1$. Moreover, since $\Cl \succeq \la I$, $\Cl$ is inverible for any $\la > 0$.

Note that we do not make the set $\xx$ appear in the notation of the co-variance operator (which can actually be defined with respect to any probability distribution on $\R^d$ and not just $\tfrac{\ib_{\xx}~dx}{|\xx|}$). This is because the set $\xx$ will usually explicit in the next sections, and in particular equal to the unit hyper-cube $\xx = (-1,1)^d$.

\paragraph{Sampling operators.} Let $n \in \N$ $(x_1,..,x_n) \in (\R^{d})^n$ be points of $\R^d$ which should be seen as samples from a certain distribution. We define the following sampling operators.

\begin{align}
    &\Cn \in \psdm(\hhe),\qquad \Cn = \tfrac{1}{n}\sum_{i=1}^n{\phie(x_i) \otimes \phie(x_i)},\qquad\Cnl = \Cn + \lambda I \label{eq:df_cn}\\
    &\Sn : \hhe \rightarrow \R^n,\qquad \Sn(g) = \tfrac{1}{\sqrt{n}}(g(x_i))_{1 \leq i \leq n}  \label{eq:df_sn}\\
    &\Sn^* : \R^n \rightarrow \hhe,\qquad \Sn^*(a)= \tfrac{1}{\sqrt{n}} \sum_{i=1}^n{a_i \phie(x_i)} \label{eq:df_snt}
\end{align}

where $\Sn^*$ and $\Sn$ are adjoint operators. We will usually use the $\widehat{\bullet}$ notation to denote sampling operators, and imply the underlying $(x_1,...,x_n)$. These operators will be used in later sections in order to quantify the difference between objects resulting from the sampling of distributions and the "ideal" objects (typically the difference between an empirical risk minimizer and the true expected risk minimizer). For instance, it is clear the $\Cn$ is an empirical version of $\Cop$, if the $x_i$ are i.i.d. samples from the uniform distribution on $\xx$.   

\paragraph{Compression operators.} Following the notations of \citet{rudi2017generalization,rudi2015less,rudi2021psd}, a compression operator of size $m$ is an operator $\Zm : \hhe \rightarrow \R^m$. We call it a \textit{compression operator} since we use it to project every element of $\hhe$ onto the range of the adjoint operator $\Zm^* : \R^m \rightarrow \hhe$. This range, which we denote with $\hht \subset \hhe$, is a subset of dimension at most $m$. We also denote with $\Projm : \hhe \rightarrow \hhe$ the orthogonal projection onto $\hht$, which can also be written $\Projm = \Zm^*(\Zm\Zm^*)^{\dagger}\Zm$, where $\dagger$ denotes the Moore-Penrose pseudo-inverse.

In this work, we will always use the notation $\widetilde{\bullet}_m$ to denote a compression operator, and the index $m$ to make the size of the compression explicit. 

In this work, we take a specific form of compression operator as in appendix C of \citet{rudi2021psd}. Indeed, let $\wtx_m \in \R^{m\times d}$ be a data point matrix representing vectors $\xt_1,...,\xt_m \in \R^d$. The compression operator associated to $\wtx_m$ is the following : 

\begin{equation}
    \label{eq:df_compression_operator}
    \Zm : \hhe \rightarrow \R^m,\qquad \Zm(g) = (g(\xt_j))_{1 \leq j \leq m} = (g^\top\phie(\xt_j))_{1 \leq j \leq m}.
\end{equation}

Note that $\Zm \Zm^* = \Kmat{\wtx_m}{\eta}$ and hence the projection operator can be written $\Projm = \Zm^* \Kmat{\wtx_m}{\eta}^{\dagger}\Zm$  and that it is simply the projection onto $\lspan{\phie(\xt_i)}_{1 \leq i \leq m}$. This compression is also chosen to satisfy the two following properties : 
\begin{itemize}
    \item if $h \in \hhe$, then $\Projm h$ represents a function of the form $\ppl{\bullet}{a,\wtx_m,\eta}$ where $a = \Kmat{\wtx_m}{\eta}^{\dagger}\Zm h$ (see \cref{eq:df_gaussian_linear} for the definition of the Gaussian linear model $\ppl{x}{a,\wtx_m,\eta}$);
    \item if $M \in \psdm{(\hhe)}$, then for any $x \in \R^d$, it holds
    \begin{equation}
        \label{df:compression_effect}
       \pp{x}{\Projm M\Projm,\phie} = \pp{x}{A,\wtx_m,\eta},\qquad A = \Kmat{\wtx_m}{\eta}^{\dagger}\Zm M \Zm^*  \Kmat{\wtx_m}{\eta}^{\dagger},
    \end{equation}
\end{itemize}
meaning that compressed linear (resp. PSD) models can be compressed as a sum of $m$ (resp. $m^2$) Gaussian kernel functions. We quantify the effect of this compression in \cref{lm:bound_projection} and \cref{thm:compression_psd_model_bound}.

\subsection{Approximation properties of the Gaussian kernel}\label[appendix]{app:approximation_gaussian_kernel}

This section aims in quantifying the approximation power of the Gaussian RKHS. We start in \cref{prp:approximation_eps} by quantifying the approximation power of the Gaussian RKHS by finding an $\eps$ approximation of a regular function with controlled norm. We then quantify the "size" of a compression for the Gaussian RKHS in \cref{lm:bound_projection}, which essentially bounds the possible variations of a function in $\hhe$ if it is equal to zero on the compression points $\wtx_m$.

\paragraph{Approximation of a Sobolev function.} This paragraph remolds results in the proof of Theorem D.4 of \citet{rudi2021psd} whose goal is to approximate any function $g \in W^{\beta}_2(\R^d)\cap L^{\infty}(\R^d)$ by a function in $\hhe$.

\bp[Approximation of $W^{\beta}_2(\R^d)\cap L^{\infty}(\R^d)$ in $\hhe$]\label{prp:approximation_eps} Let $g$ be a function in $W^{\beta}_2(\R^d)\cap L^{\infty}(\R^d)$ and $\eta \in \R^{d}_{++}$. Denote with $|\eta|$ the product $|\eta| := \prod_{i=1}^d{\eta_i}$ and $\eta_0 = \min_{1 \leq i\leq d}{\eta_i}$.  For any $\eps \in (0,1]$, there exists $\theta \in \hhe$ such that 
\begin{equation}\label{eq:approx_gaussian}
\left\{
\begin{aligned}
\|\theta - g\|_{L^2(\R^d)}& \leq \eps \|g\|_{W^{\beta}_2(\R^d)}\\ \|\theta-g\|_{L^{\infty}(\R^d)} &\leq C_{1}~\eps^{1-\nu} \|g\|_{\bullet}  
\end{aligned}
\right.
,~~ \|\theta\|_{\hhe} \leq C_{2}~\|g\|_{W^{\beta}_2(\R^d)} ~|\eta|^{1/4}\left(1 + \eps  \exp\left(\tfrac{50}{\eta_0\eps^{2/\beta}}\right)\right),
\end{equation}
where $\|g\|_{\bullet} = \|g\|_{L^{\infty}(\R^d)}$ if $\beta  \leq d/2$ and $\|g\|_{\bullet} = \|g\|_{W^{\beta}_2(\R^d)}$ if $\beta   >  d/2$, $\nu = \min(1,d/(2\beta))$ and $C_{1},C_2$ are constants which depend only on $d,\beta$.
\ep

\bpr 
Recalling the notations from the proof of Theorem D.4. of \citet{rudi2021psd}, let $g_t := t^{-d}g_1(x/t)$ where $g_1$ is defined as $g$ in equation (D.2) of \citet{rudi2021psd}.  The following result hold.
\begin{itemize}
    \item By step 1 of the proof of Theorem D.4, $\|g - g\star g_t\|_{L^2(\R^d)} \leq (2t)^{\beta}~\|g\|_{W^{\beta}_2(\R^d)} $.
    \item By step 2 and the beginning of step 3 of the proof of Theorem D.4,
    $$\|g \star g_t\|_{\hhe} \leq 2^{\beta}  \pi^{-d/4} |\eta|^{1/4}(1 + (t/3)^{\beta}\exp(\tfrac{50}{\eta_0 t^2}))\|g\|_{W^{\beta}_2(\R^d)}.$$
    \item As in step 5 of the proof of Theorem D.4 and in particular the Young convolution inequality combined with the fact that $\|g_1\|_{L^{1}(\R^d)}$ is finite, $\|g\star g_t\|_{L^{\infty}(\R^d)} \leq \|g_1\|_{L^{1}(\R^d)}~\|g\|_{L^{\infty}(\R^d)}$ which in turn implies $\|g-g\star g_t\| \leq (1 + \|g_1\|_{L^{1}(\R^d)})~\|g\|_{L^{\infty}(\R^d)}$.
\end{itemize}
Replacing $t$ by $\eps^{1/\beta}/2$, we get all the bounds except the bound for the $L^{\infty}$ norm in the case where $\beta > d/2$. In that case, we proceed in the following way. Recycling results and notations from the proof of Theorem D.4 of \citet{rudi2021psd}, denoting with $\FT$ the Fourier transform defined in \cref{prop:fourier}, it holds 
\begin{align*}
    \|f-f\star g_t\|_{L^{\infty}(\R^d)} &\leq \|\FT(f -f\star g_t) \|_{L^1(\R^d)}  \text{   \cref{prop:fourier}}\\
    & = \|\FT(f)(1-\FT(g_t))\|_{L^1(\R^d)}\\
    & \leq \|(1+\|\omega\|^2)^{\beta/2}\FT(f)\|_{L^2(\R^d)}~\|(1+\|\omega\|^2)^{-\beta/2}~\FT(1-g_t)\|_{L^2(\R^d)} \\
    &\leq 2^{\beta}\left(\int_{\|\omega\|>1/t}{(1+\|\omega\|^2)^{-\beta}~d\omega}\right)^{1/2}\|f\|_{W^{\beta}_2(\R^d)}  \text{ \cref{eq:sobolev-norm-char}}\\
    &= 2^{\beta}\left(S_d\int_{r>1/t}{r^{d-1}~(1+r^2)^{-\beta}~dr}\right)^{1/2}\|f\|_{W^{\beta}_2(\R^d)}  \text{ (spherical coord.)}\\
    &\leq 5^{\beta/2} S_d^{1/2}\left(\int_{r>1/t}{r^{d-1-2\beta}~dr}\right)^{1/2}\|f\|_{W^{\beta}_2(\R^d)} ~(t < 1/2)\\
    &= 5^{\beta/2}\tfrac{1}{\sqrt{2\beta -d}}S_{d}^{1/2} t^{\beta -d/2}\|f\|_{W^{\beta}_2(\R^d)}\\
    &= 5^{\beta/2} 2^{d/2-\beta} S_d^{1/2}\tfrac{1}{\sqrt{2\beta -d}}\eps^{1-d/(2\beta)}\|f\|_{W^{\beta}_2(\R^d)},
\end{align*}
where $S_d$ is the surface area of the $d-1$ dimensional hyper-sphere.
\epr

\paragraph{A bound on the performance of compression when using uniform samples from $\xx = (-1,1)^d$.} In this paragraph, we study the effect of performing compression with a compression operator of the form $\Zm$ (see \cref{eq:df_compression_operator}) where the associated $\wtx_m$ are iid samples from the uniform measure on the unit hyper-cube $\xx = (-1,1)^d$.

\blm\label{lm:bound_projection}
Let $m \in \N$, $\delta \in (0,1]$, $\tau \geq 1$ and $\rho \in (0,1]$. Let $\eta = \tau \ib_d \in \R^d_{++}$. Let $\wtx_m \in \R^{m\times d}$ be a data matrix corresponding to vectors $\xt_1,...,\xt_m$ which are sampled independently and uniformly from $\xx=(-1,1)^d$ and let $\Projm$ be the associated projection operator in $\hhe$. With probability at least $1-\delta$, if $m \geq C_1 \tau^{d/2} (\log \tfrac{C_2}{\rho} )^d \left(\log \tfrac{C_3}{\delta} + \log \tau + \log \log \tfrac{C_2}{\rho}\right)$, then it holds : 
\begin{equation}\label{eq:bound_projection}
\sup_{x \in \xx}{\|(I - \Projm)\phie(x)\|} \leq \rho,
\end{equation}
where $C_1,C_2,C_3$ are constants which depend only on the dimension $d$ and not on $\tau,m,\delta,\rho$.
\elm

\begin{proof}
Let $h$ denote the fill distance with respect to $\wtx_m$, i.e. 
\begin{equation}\label{eq:fill_distance}
    h = \max_{x \in [-1,1]^d}{\min_{1 \leq j \leq m}{\|x - \xt_j\|}}
\end{equation}
Using Lemma 12 p.19 of \citet{vacher2021}, we there exists two constants $C_1,C_2$ depending only on $d$ such that $h \leq \left(C_1 m^{-1}(\log(C_2 m/\delta)\right)^{1/d}$. 

Applying Theorem C.3 from \citet{rudi2021psd} in the case where $\xx = (-1,1)^d$, $\eta= \tau \ib_d$, there exists constants $C_3,C_4,C_5$ depending only on the dimension $d$ such that when $h \leq \tau^{-1/2}C_3^{-1}$, the following holds : 
\begin{equation}\label{eq:result_C3}
   \sup_{x \in \xx}{\|(I - \Projm)\phie(x)\|} \leq C_4 e^{-\tfrac{C_5}{\tau^{1/2}h} \log \tfrac{C_5}{\tau^{1/2}h} } 
\end{equation}

Now note that taking $C_6 = \max(C_3^{-1},eC_5)$ and $C_7 =\max(e,C_4) $, as soon as $h \leq C_6 \tau^{-1/2} / \log \tfrac{C_7}{\rho}$, it holds a) $h \leq \tau^{-1/2}C_3^{-1}$, b) $\tfrac{C_5}{\tau^{1/2}h} \geq e$ and thus $\log \tfrac{C_5}{\tau^{1/2}h} \geq 1$, 
and hence c)  $\sup_{ x\in \xx}\|(I - \Projm)\phie(x)\| \leq \rho$ using \cref{eq:result_C3}.
Using the bound on $h$, this is satisfied as soon as 
$$m \geq C_8 \tau^{d/2} \left(\log \tfrac{C_7}{\rho}\right)^d \log(C_2 m/\delta),$$ 
where $C_8 = \max(C_1/C_6^d,e)$. Using the fact that $C_2,C_8 \geq e$, and using the reasoning in the proof of Theorem C.5 of \citet{rudi2021psd}, in equation (C.44), a sufficient condition is the following : 
\begin{equation}\label{eq:sufficient_condition}
m \geq 2 C_8 \tau^{d/2} \left(\log \tfrac{C_7}{\rho}\right)^d\left(\log(2C_2C_8/\delta) + \tfrac{d}{2}\log \tau + d \log \log \tfrac{C_7}{\rho}\right).
\end{equation}
The result in the theorem is obtained by taking $C_1 \leftarrow 2 C_8 d,~C_2 \leftarrow C_7,~C_3 \leftarrow 2 C_2C_8$.
\end{proof}

\section{Properties of Gaussian PSD models}\label[appendix]{app:gaussian_kernel_psd}

In this section, we detail some of the properties specific to Gaussian PSD models.

\subsection{Bounds on the support and the derivatives}\label[appendix]{app:bound_support_der}

In this section, we present results which can be used to bound the tail and derivatives of a Gaussian PSD model. These bounds can be used both for theoretical purposes (see \cref{app:proof_general_method,app:target_distribution}) and to perform adaptive bounds in an algorithm (see \cref{sec:sampling})

\blm[tail bound]\label{tail_bound_gaussian} Let $\delta = (\delta_k)\in\R^d$, $\eta \in \R^d_{++}$, $X \in \R^{n\times d}$ and $A \in \psdm(\R^n)$. Let $f(x;A,X,\eta)$ be the associated PSD model. Define $\overline{x}$, $\underline{x}$ :
\[\forall 1 \leq k \leq d,~\overline{x}_k = \max_{ 1 \leq i \leq n}{X_{ik}},~\underline{x}_k = \min_{ 1 \leq i \leq n}{X_{ik}}.\]
Let $Q_\delta = Q(\underline{x} -\delta,\overline{x}+\delta)$. Then the following bound holds:
\begin{equation}\label{eq:tail_bound}
 \int_{\R^d\setminus{Q_\delta}}{|f(x; A,X,\eta)|dx} \leq \left(2 \pi^{d/2}\det(\diag(2\eta))^{-1/2}~\sum_{k=1}^d{e^{-2\eta_k\delta^2_k}}\right)\sum_{i,j}{[A \circ \Kmat{X}{\eta/2}]_{ij}}   
\end{equation}
\elm

\begin{proof} Start by recalling the following simple Chernoff bound:

\begin{equation}
\label{eq:chernoff}
    \forall x > 0,~\int_{x}^{+\infty}{e^{-t^2}~dt} \leq \sqrt{\pi}e^{-x^2}
\end{equation}

Indeed, take $\la > 0$. Since $e^{-2\lambda x}~e^{2\lambda t} \leq \ib_{t>x}$, it holds $$\int_{x}^{+\infty}{e^{-t^2}~dt} \leq e^{-2\la x}e^{\la^2}\int_{-\infty}^{+\infty}{e^{-(t-\la)^2}~dt} \leq \sqrt{\pi} e^{-x^2}e^{(\la - x)^2}.$$
Hence, taking $\la = x$, we get the bound. Then we perform the following bound.
\begin{align*}
    \int_{\R^d\setminus{Q(-\delta,\delta)}}{k_\eta(x,0) dx} &= \frac{1}{\prod_{k=1}^d{\eta_k^{1/2}}} \int_{\R^d\setminus{Q(-\delta\sqrt{\eta},\delta\sqrt{\eta})}}{k_1(x,0) dx} \\
    &\leq \frac{1}{\prod_{k=1}^d{\eta_k^{1/2}}} \sum_{k=1}^d{\left(\pi^{(d-1)/2} 2\int_{\delta_k\sqrt{\eta_k}}^{\infty}{e^{-t^2}~dt}\right)} \\
    & \leq 2\pi^{d/2}\det(\diag(\eta))^{-1/2}~\sum_{k=1}^d{ e^{-\delta^2_k \eta_k}},
\end{align*}
where we go from the first to the second line by noting that 
$$\R^d\setminus{Q(-\delta,\delta)} \subset \cup_{k=1}^d{\R\times....\times \R\setminus{[-\delta_k,\delta_k]}\times...\times \R},$$
 and the last inequality comes from a \cref{eq:chernoff}.

The result immediately follows from \cref{eq:formula_prod} as well as the fact that $Q_\delta$ contains $(x_i+x_j)/2 + Q(-\delta,\delta)$ for all $1 \leq i,j \leq n$.

\end{proof}

\blm[derivative bound for general PSD model]\label{lm:der_bound_psd}
Let $\eta \in \R^d_{++}$, $M \in \psdm(\hhe)$ $X \in \R^{n\times d}$ and $A \in \psdm^n$. The following bounds hold : 
\begin{align}
&\sup_{x \in \R^d} |\partial_{\alpha}\pp{x}{M,\phie}| \leq 2^{3|\alpha|/2}~\eta^{\alpha/2}~\|M\|\label{eq:bound_gen_der_1}\\
&\sup_{x \in \R^d} |\partial_{\alpha}\pp{x}{A,X,\eta}| \leq 2^{3|\alpha|/2}~\eta^{\alpha/2}~\|\Kmat{X}{\eta}^{1/2}A\Kmat{X}{\eta}^{1/2}\| \label{eq:bound_gen_der_2}
\end{align}

\elm 

\bpr 
By derivation of a bi-linear form, we get 
\begin{equation*}
    \partial_{\alpha} \pp{x}{M,\phie} = \sum_{\beta \leq \alpha}{\binom{\alpha}{\beta}\scal{\partial_{\beta}\phie(x)}{M\partial_{\alpha-\beta}\phie(x)}_{\hhe}}
\end{equation*}
Hence, using \cref{lm:differential_gaussian_embedding}, we get, for any $x \in \R^d$, 
\begin{equation}
  |\partial_{\alpha} \pp{x}{M,\phie}| \leq \|M\| \sum_{\beta \leq \alpha}{\binom{\alpha}{\beta}2^{|\beta|/2}\eta^{\beta/2}2^{|\alpha-\beta|/2}\eta^{(\alpha-\beta)/2}} = 2^{3|\alpha|/2} \eta^{\alpha/2} \|M\|.
\end{equation}
In particular, since $\pp{x}{A,X,\eta} = \pp{x}{M_A,\phie}$ with $M_A = Z^*AZ $ for $Z : h \in \hhe \mapsto h(x_i)_{1 \leq i \leq n}$, and since $ZZ^* = \Kmat{X}{\eta}$, it holds 
$$\|M_A\| = \|Z^*AZ\| = \|A^{1/2}ZZ^*A^2\| = \|A^{1/2}\Kmat{X}{\eta}A^{1/2}\| = \|\Kmat{X}{\eta}^{1/2}A\Kmat{X}{\eta}^{1/2}\|,$$
and hence the second equation of the lemma.
\epr

\subsection{Compression as a Gaussian PSD model}\label[appendix]{app:compression_gaussian_psd}

In this section, we restate Theorem C.4 of \citet{rudi2021psd} on the compression of a PSD model of the form $\pp{x}{M,\phie}$ into a Gaussian PSD model.

Let $\eta \in \R^d_{++}$, $M \in \psdm(\hhe)$. Given a matrix $\wtx_m \in \R^{m\times d}$ representing vectors $\xt_1,...,\xt_m \in \R^d$, and the associated projection operator $\Projm$ (for more details, see \cref{app:linop}), one can compress the PSD model $\pp{\bullet}{M,\phie}$ into $\pp{\bullet}{\Projm M \Projm}$ which is also a Gaussian PSD model of the form $\pp{\bullet}{A,\wtx_m,\eta}$ ($A$ is defined in \cref{df:compression_effect}). 
The quality of the compression is given by the following theorem.

\bt[Theorem C.4 of \citet{rudi2021psd}]\label{thm:compression_psd_model_bound} Using the previous notations, the compressed model associated to $ \Projm M \Projm$ of $M$ onto $\wtx_m$ has a distance to the original PSD model associated to $M$ bounded, for any $x \in \xx$, by
\begin{align}\label{eq:compression_effect}
    |\pp{x}{M,\phie} - \pp{x}{\Projm M\Projm,\phie}| &\leq \sqrt{\pp{x}{M,\phie}} ~ \|M\|^{1/2} \sup_{x \in \xx}{\|(I-\Projm)\phie(x)\|}\nonumber\\
    &+ \|M\| \sup_{x \in \xx}{\|(I-\Projm)\phie(x)\|}^2.
\end{align}
\et

We therefore see that the quality of the compression depends mainly on the quantity 
$$\sup_{x \in \xx}{\|(I-\widetilde{P})\phie(x)\|},$$
 which can be bounded using \cref{eq:bound_projection} in \cref{lm:bound_projection}.

\subsection{Approximation properties of Gaussian PSD model}\label[appendix]{app:approximation_psd}

Define, for any measurable $\Omega \subset \R^d$, and any $f:\Omega \rightarrow \R$, the following function (set to $+\infty$ if the set is empty).

\begin{equation}
    \label{df:nsos}
    \Nsos{f}{\Omega,\beta} = \inf \left\{ \sum_{i=1}^Q{\max(\|f_j\|_{L^{\infty}(\Omega)},\|f_j\|_{W^{\beta}_2(\Omega)})^2} ~|~f = \sum_{j=1}^Q{f_j^2},~Q \in [0,+\infty]\right\}
\end{equation}

Here, we recall Theorem D.4 of \citet{rudi2021psd}, refined in a small way to have more control over the dependence in the $f_j$.

\bt[Theorem D.4 of \citet{rudi2021psd}]\label{thm:D4} Let $\tau \geq 1$ and $\eps \in (0,1]$ and $f$ such that $\Nsos{f}{\R^d,\beta} <\infty$. Let $\eta = \tau\ib_d$.
There exists $\Meps \in \psdm(\hhe)$ such that 
 $\feps := \pp{\bullet}{\Meps,\phie}$ is $\eps$ close to $p$ in $L^2$ norm and has controlled trace norm:
\begin{align}
    &\|\feps - f\|_{L^2(\R^d)} \leq C_1~\Nsos{f}{\R^d,\beta}~\eps,\nonumber \\
    & \tr{(\Meps)} \leq C_2~\Nsos{f}{\R^d,\beta}~ \tau^{d/2}(1+\eps^2 \exp(C_3~ \eps^{-2/\beta}/\tau)) \label{eq:carac_m_eps},
\end{align}
where the constants $C_1,C_2,C_3$ depend only on $\beta$, $d$.
\et

\bpr 
Let $\delta > 0$ and take $Q_{\delta} \in [0,+\infty]$ as well as $f_{\delta,j}$ such that $f = \sum_{j=1}^{Q_{\delta}}{f_{\delta,j}^2}$ point-wise and
\[\sum_{i=1}^Q{\|f_{\delta,j}\|_{W^{\beta}_2(\R^d)}\max(\|f_{\delta,j}\|_{L^{\infty}(\R^d)},\|f_{\delta,j}\|_{W^{\beta}_2(\R^d)})}  \leq \Nsos{f}{\beta}.\] 
Now using exactly the same reasoning than in the proof of Theorem D.4 of \citet{rudi2021psd} but setting simply $t = \eps^{1/\beta}$, it holds the existence of $M_{\delta,\tau,\eps}$ and $C_1,C_2,C_3$ depending only on $\beta,d$ such that 
\begin{align*}
&\|f_{\delta,\tau,\eps} - f\|_{L^2(\R^d)} \leq C_1~(\Nsos{f}{\R^d,\beta}+\delta)~\eps,\\
& \tr{(M_{\delta,\tau,\eps})} \leq C_2~(\Nsos{f}{\R^d,\beta} + \delta)~ \tau^{d/2}(1+\eps^2 \exp(C_3~ \eps^{-2/\beta}/\tau)).
\end{align*}
Note that in the proof, $M_{\delta,\tau,\eps}$ is well defined since its trace norm is bounded (normal convergence). Now if $\Nsos{f}{\R^d,\beta} = 0$, then $f = 0$ and there is nothing to prove. If not, then taking $\delta = \Nsos{f}{\R^d,\beta}$, the theorem holds.
\epr

\section{The sampling algorithm }\label[appendix]{app:algorithm_proof}

In this section, we formally prove that \cref{alg:sampling} converges, as in \cref{thm:approximation_distribution}, as well as the different results of \cref{sec:sampling}. We start by introducing some notations around dyadic decomposition of hyper-rectangles. We then introduce a well founded order relation, which we will then use to both construct the random variables we study, justify the convergence of the algorithm and prove its correctness. 

Recall we are given a density (up to a scaling factor) $f(x)$ and that we denote with $I(Q)$ the quantity $\int_{Q}{f(x)dx}$ on any hyper-rectangle $Q$.

\subsection{Dyadic decompositions and convergence of \cref{alg:sampling}}\label[appendix]{app:dyad_dec_convergence}

\paragraph{Dyadic sub-rectangles}

Let $Q = \prod_{k=1}^d{[a_k,b_k[}$ be a hyper-rectangle where $a \leq b$ and let $\delta = b-a$. Let $q \in \N^d$. We define ${\dyad}_{Q,q}$ to be the set of dyadic sub-rectangles of $Q$ whose $k$-th size is cut in half $q_k$ times, i.e.
\[{\dyad}_{Q,q} = \left\{\prod_{k=1}^d{[a_k + \delta_k\tfrac{s}{2^{q_k}},a_k + \delta_k\tfrac{s+1}{2^{q_k}}[}~:~ s \in \prod_{k=1}^d{\llbracket 0,2^{q_k}-1 \rrbracket}\right\}.\]

We denote with $\dyad_Q$ the set of dyadic sub-rectangles of $Q$, i.e. the union $\bigcup_{q \in \N^d}{\dyad_{Q,q}}$. 

Moreover, if $q^\rho_k = \max(0,\lceil \log_2 \tfrac{\delta_k}{\rho}\rceil)$, we also define $\dyad_{Q,\eps} := \dyad_{Q,q^\rho}$ to be the set of dyadic sub-rectangles whose size is just below $\rho$.

\paragraph{Well founded order relation on hyper-rectangles}
For all $\rho > 0$, we define the following strict order relation. We say that 
$Q \precr Q^{\prime}$ if  the following conditions hold :

\begin{enumerate}
\item $Q \in \dyad_{Q^\prime}$;
\item There exists $k \in \llbracket 1,d\rrbracket$ such that $\delta^\prime_k > \rho$ and $\delta_k < \delta^\prime_k$.
\end{enumerate}

This relation is obviously transitive. Moreover, if $s(Q): = \sum_{k=1}^d\delta_k(Q)$, it is easy to show that $Q \precr Q^{\prime}$ implies $s(Q) \leq s(Q^{\prime}) - \rho/2$. Since $s \geq 0$, this in turn shows that any strictly decreasing sequence for $\precr$ is finite, and that $Q \precr Q^{\prime}$ and $Q^{\prime} \precr Q$ are incompatible. 

We are now ready to define the random variable $\Yb_{\rho,Q,n}$ by structural induction on $Q$ for any $n \in \N$. Recall that for $\Omega \subset \R^d$, we denote with ${\cal U}_{\Omega}$ the uniform law on $\Omega$.

\paragraph{Definition of the random variable $\Yb_{\rho,Q,n}$ and relation to the algorithm} We now define a random variable from whose distribution we sample when {\normalfont \textproc{SamplerRec}} in \ref{alg:sampling} is applied.

\begin{itemize}
\item If $\delta(Q) \leq \rho$, then for any $n \in \N$, $\Yb_{\rho,Q,n} \sim {\cal{U}}_Q^{\otimes n}$ 
\item Else, let $n \in \N$ and $k_Q = \min \argmax_{1 \leq k \leq d} \delta_k(Q)$ be the smallest index amongst the largest sides of $Q$. Define $Q_1$ and $Q_2$ to be the two hyper-rectangles obtained by cutting $Q$ in half along the direction $k_Q$. Since $\delta_{k_Q} > \rho$ and $Q_1,Q_2$ are dyadic sub-rectangles of $Q$, we have $Q_1,Q_2 \precr Q$.

 By structural induction, we give ourselves a probability space on which we take we take the following random variables to be independent : $\Yb_{1,m} \sim \Yb_{\rho,Q_1,m},~\Yb_{2,m} \sim ~\Yb_{\rho,Q_2,m}$ for $0 \leq m \leq n$ and $M \sim {\cal{B}}(n,I(Q_1)/I(Q))$ and define 

\begin{equation}\label{eq:df_Y}
\Yb_{\rho,Q,n} = (\Yb_{\rho,Q_1,M},\Yb_{\rho,Q_2,n-M}) :=  \sum_{m=0}^n{\ib_{M=m}(\Yb_{1,m},\Yb_{2,n-m})}.
\end{equation}

\end{itemize}

\begin{lemma}[Termination of the algorithm and first result]\label{lm:termination_algorithm}
For any inputs $\rho >0$, hyper-rectangle $Q$ and  $n \in \N$, {\normalfont \textproc{SamplerRec}} in  \cref{alg:sampling} terminates and  returns a sample $(y_1,...,y_n)$ from $\Yb_{\rho,Q,n}$. 
\end{lemma}

\begin{proof}
This is a simple application of structural induction on the well-founded order $\precr$ for the termination and then again for the fact that a sample $(y_1,...,y_n)$ from $\Yb_{\rho,Q,n}$, using the definition of $\Yb$ above.
\end{proof}

\subsection{Proof of \cref{thm:approximation_distribution} }\label[appendix]{app:proof_theorem_1}

In this section, we prove \cref{thm:approximation_distribution}. To do so, we define a random variable $X_{\rho,Q}$, compute its density with respect to the Lebesgue measure on the hyper-rectangle $Q$ (and show it is our target density), and show that $\Yb = X$ up to some random shuffling.

\paragraph{Definition of the variable $X_{\rho,Q}$} Recall the definition of $\dyad_{Q,\rho}$ from \cref{app:dyad_dec_convergence}. We define a random variable $R_{\rho,Q}$ on $\dyad_{Q,\rho}$ whose law is defined $P(R_{\rho,Q} = r) = I(r)/I(Q)$. Recall that for any $r \subset \R^d$, we denote with ${\cal U}_{r}$ the uniform law on $r$.  We give ourselves a measure space on which there exists a family of random variables $U_r \sim {\cal U}_r$ for $r \in \dyad_{Q,\rho}$ and $R \sim R_{\rho,Q}$ which are all independent and define

\begin{equation}\label{eq:df_x}
X_{\rho,Q} = U_R := \sum_{r \in \dyad_{Q,\rho}}{\ib_{R = r} U_r}
\end{equation}

\begin{lemma}[density of $X_{\rho,Q}$]\label{lm:density_X}
The density of $X_{\rho,Q}$ with respect to the Lebesgue measure is given by \cref{eq:dyadic_approx_density}, i.e.
\begin{equation}
    \label{eq:dyadic_approx_2}
    \forall x \in Q,~ p_{X_{\rho,Q}}(x) = \sum_{r \in \dyad_{Q,\rho}}{\tfrac{I(r)}{I(Q)}  \tfrac{\ib_r(x)}{|r|}}.
\end{equation}

\end{lemma}

\begin{proof}
For any measurable function $f$, it holds 
\begin{align*}
\expect{f(X_{\rho,Q})} &= \sum_{r \in \dyad_{Q,\rho}}{\expect{\ib_{R = r}f(U_r)}} \\
&=\sum_{r \in \dyad_{Q,\rho}}{P(R=r) \expect{f(U_r)}}   \\
&= \sum_{r \in \dyad_{Q,\rho}}{\tfrac{I(r)}{I(Q)} \int_{\R^d}{f(x) \tfrac{\ib_r(x)}{|r|}dx}}\\
 &= \int_{\R^d}{f(x)\left(\sum_{r \in \dyad_{Q,\rho}}{\tfrac{I(r)}{I(Q)}  \tfrac{\ib_r(x)}{|r|}} \right)dx}
\end{align*}
\end{proof}

\paragraph{Action of a permutation and decomposition} Let $n \in \N$. For any permutation $\tau \in \mathfrak{S}_n$ and vector $v \in \R^n$, denote with $\tau \star v$ the permuted vector $(v_{\tau^{-1}(i)})_{1 \leq i \leq n}$.

We now define a decomposition of a permutation of $n$ variables as i) a permutation of the first $m$ variables and a permutation of the last $n-m$ variables ii) followed by a rearrangement of these variables. 

Given $I \subset \llbracket 1,n \rrbracket$ of size $m$, define $\tau_I$ as the unique permutation satisfying $I = \{\tau_I(1),...,\tau_I(m)\}$, $I^c = \{\tau_I(m+1),...,\tau_I(n)\}$ and $\tau_I(1)<...< \tau_I(m)$ and $\tau_I(m+1) < ... < \tau_I(n)$. For any $m \in \llbracket 0,n \rrbracket$, if ${{\cal P}_m(n)}$ denotes the set of subsets of $\{1,...,n\} $ of size $m$, the map from ${{\cal P}_m(n)} \times\mathfrak{S}_m \times \mathfrak{S}_{n-m}$ to  $\mathfrak{S}_n$ defined as
\begin{equation}
\label{eq:map_perm}
(I,\sigma_m,\sigma_{n-m} \mapsto  \left( i \mapsto \left\{\begin{aligned}
&\tau_I(\sigma_m(i)) &&\text{ if } i \leq m\\
&\tau_I(m+ \sigma_{n-m}(i-m)) &&\text{ otherwise}
\end{aligned}
\right. \right)
\end{equation} 
is a bijection.  

\begin{lemma}\label{lm:effect_perm}
Let $\rho >0$. Let $n \in \N$, $Q$ be a hyper-rectangle of $\R^d$. Let $\sigma$ be a random permutation independent of $\Yb_{\rho,Q,n}$. Then $ (\Yb^{\sigma(i)}_{\rho,Q,n})_{1 \leq i \leq n} \sim X_{\rho,Q}^{\otimes n}$. 
\end{lemma}

\begin{proof}
Once again, we prove this by structural induction.
Fix $\rho > 0$. We will prove the following property by structural induction on the set of hyper-rectangles $Q$ equipped with the strict order relation $\precr$ :

For any $n \in \N$, if $\sigma$ is a random permutation (i.e. distributed uniformly amongst all permutations in $\mathfrak{S}_n$), $\Yb_{Q,n} \sim \Yb_{\rho,Q,n}$ and both random variables are independent, then $(\Yb^{\sigma(i)}_{Q,n})_{1 \leq i \leq n} \sim X_{Q,\rho}^{\otimes n}$.

\paragraph{1)} If $\delta(Q) \leq \rho$.

On the one hand, by definition of $\Yb_{\rho,Q,n}$, it holds that for any $n \in \N$,   $\Yb_{\rho,Q,n} \sim \uu^{\otimes n}_{Q}$ and hence $\Yb_{Q,n} \sim \uu^{\otimes n}_Q$. By invariance of the product measure by permutation, it also holds that $(\Yb^{\sigma(i)}_{Q,n})_{1 \leq i \leq n} \sim  \uu^{\otimes n}_{Q}$. 

On the other hand, since $\delta(Q) \leq \rho$, it is easy to see that $q^\rho = 0$ and hence $\dyad_{Q,\rho} = \{Q\}$. Hence, by definition of $X_{\rho,Q}$ in \cref{eq:df_x}, $R$ is deterministic and hence $X_{\rho,Q} = U_Q \sim \uu_Q$. 

Putting things together, this yields $(\Yb^{\sigma(i)}_{Q,n})_{1 \leq i \leq n} \sim X_{\rho,Q}^{\otimes n}$.

\paragraph{2)} Assume $\delta(Q) > \rho$ and take $n \in \N$. By definition of $\Yb_{\rho,Q,n}$ in \cref{eq:df_Y}, and since our property only concerns a convergence in law, we can assume that $\Yb_{Q,n}$ is of the form
\[\Yb_{Q,n} =\sum_{m=0}^n{\ib_{M=m}(\Yb_{Q_1,m},\Yb_{Q_2,n-m})},\]

where $\Yb_{Q_1,m},\Yb_{Q_2,m}$ and $M$ are independent and independent of $\sigma$, $\Yb_{Q_1,m} \sim \Yb_{\rho,Q_1,m}$, $\Yb_{Q_2,m} \sim ~\Yb_{\rho,Q_2,m}$ for $0 \leq m \leq n$ and $M \sim {\cal{B}}(n,I(Q_1)/I(Q))$, and $Q_1,Q_2$ are defined just before \cref{eq:df_Y}. It is easy to see that since $Q_1 \sqcup Q_2 = Q$ and $Q_1,Q_2 \precr Q$, it holds $\dyad_{Q,\rho} = \dyad_{Q_1,\rho} \sqcup \dyad_{Q_2,\rho}$ where $\sqcup$ symbolises a disjoint union.

Fix a measurable function $f$. Using the independence of $M$ from the other variables and the fact that it is discrete, it holds 
\begin{align*}
\expect{f(\sigma\star \Yb_{Q,n})} = \sum_{m=0}^n{P(M=m) \expect{f(\sigma \star(\Yb_{Q_1,m},\Yb_{Q_2,n-m}))}} .
\end{align*}

Now note that using our bijection \cref{eq:map_perm}, it holds 

\begin{align*}
&\mathbb{E}_{\sigma,\mathbf{Y}_{Q_1,m},\mathbf{Y}_{Q_2,m}}\left[f(\sigma \star(\Yb_{Q_1,m},\Yb_{Q_2,n-m}))\right]\\
& = \frac{1}{n!}\sum_{\tau \in \mathfrak{S}_n}{\mathbb{E}_{\mathbf{Y}_{Q_1,m},\mathbf{Y}_{Q_2,m}}\left[f(\tau \star(\Yb_{Q_1,m},\Yb_{Q_2,n-m}))\right] }\\
&= \frac{1}{n!} \sum_{\substack{I \subset \llbracket 1,n \rrbracket\\ |I| = m}}\sum_{\sigma_1 \in \mathfrak{S}_m} \sum_{\sigma_2 \in \mathfrak{S}_{n-m}}{\mathbb{E}_{\mathbf{Y}_{Q_1,m},\mathbf{Y}_{Q_2,m}}\left[f(\tau_I\star( \sigma_1 \star \Yb_{Q_1,m},\sigma_2 \star \Yb_{Q_2,n-m})))\right] }\\
&= \frac{1}{\binom{n}{m}} \sum_{\substack{I \subset \llbracket 1,n \rrbracket\\ |I| = m}}{\mathbb{E}_{\sigma_1,\sigma_2,\mathbf{Y}_{Q_1,m},\mathbf{Y}_{Q_2,m}}\left[f(\tau_I\star( \sigma_1 \star \Yb_{Q_1,m},\sigma_2 \star \Yb_{Q_2,n-m})))\right] }
\end{align*}

Now note that by induction, $\sigma_1 \star \Yb_{Q_1,m} \sim X_{\rho,Q_1}^{\otimes m}$ and $\sigma_2 \star \Yb_{Q_2,n-m} \sim X_{\rho,Q_2}^{\otimes (n-m)}$. 

Let $X^1_1,...,X_1^n \sim X_{\rho,Q_1}$ and $X^1_1,...,X_1^n \sim X_{\rho,Q_2}$ be $2n$ i.i.d. random variables; the previous statement shows that , $\tau_I\star( \sigma_1 \star \Yb_{Q_1,m},\sigma_2 \star \Yb_{Q_2,n-m}) \sim (X_1^i \ib_{i \in I} + (\ib - \ib_{i \in I})X_2^i)_{1 \leq i \leq n}$ (here, $I$ is fixed). Moreover, note that $P(M = m) = \binom{n}{m}{q^m(1-q)^{n-m}}$ where $q = I(Q_1)/I(Q)$. Hence

\[\expect{f(\sigma\star \Yb_{Q,n})} = \sum_{I \subset \llbracket 1,n\rrbracket }{q^{|I|}(1-q)^{n-|I|}\mathbb{E}_{X_1^i,X_2^i}(X_1^i \ib_{i \in I} + (\ib - \ib_{i \in I})X_2^i)_{1 \leq i \leq n} }\]
Now let $B_1,...,B_n$ be $n$ iid bernoulli variables of parameter $q$ independent of the $X_1,X_2$. Note that from the previous equation,
$$\expect{f(\sigma\star \Yb_{Q,n})} = \expect{f((X_1^i B_i + X_2^i (1-B_i))_{1 \leq i \leq n})}$$
It is easy to see that $(X_1^i B_i + X_2^i (1-B_i))_{1 \leq i \leq n}$ are i.i.d. and distributed as $X_{\rho,Q}$, which concludes the proof.

\end{proof}

\begin{proof}[Proof of \cref{thm:approximation_distribution}]
\cref{thm:approximation_distribution} is now  a simple consequence of \cref{lm:density_X,lm:termination_algorithm,lm:effect_perm}.The bound on the number of integral computations can be easily obtained by noting that for any sample, at most $sum_{k=1}^d{q^{\rho}_k}$ hyper-rectangles are visited (we do not count the first since this computation is done once and for all in any case). Since $q^{\rho}_k = \lceil\log_2(\delta_k/\rho)\rceil \leq \log_2(2\delta_k/\rho)$, this yields a bound of $\log_2(2^d|Q|/\rho^d) = \log_2(|Q|) + d\log_2(2/\rho)$ per sample, hence the result. 
\end{proof}

\subsection{Evaluating the error of the sampling algorithm : proof of \cref{thm:variation_bounds}}\label[appendix]{app:thm_variation_bounds}

\cref{thm:variation_bounds} is a specific case of the following theorem. For a given function $g$ defined on a hyper-rectangle $Q$, define its Lipschitz constant with respect to the infinity norm :

\begin{equation}
    \label{df:lip_constant}
    \forall x \in Q,~\|x\|_{\infty} = \sup_{1 \leq k \leq d}{|x_i|},\qquad \Lip_{\infty}(g) = \sup_{\substack{x,y \in Q\\ x\neq y}}\frac{|g(x) - g(y)|}{\|x-y\|_{\infty}}.
\end{equation}

\bt[Variation bounds]\label{thm:variation_bounds_evolved}
Let $Q$ be a hyper-rectangle, $\rho > 0$, $p_Q = f/I(Q)$ and $p_{Q,\rho}$ defined in \cref{eq:dyadic_approx_density}. Recall the definition of $\Lip_{\infty}(f),\Lip_{\infty}(\sqrt{f})$ from \cref{df:lip_constant}. The following bounds hold. 
\begin{align}
    &d_{TV}(p_Q,p_{Q,\rho}) \leq \tfrac{|Q|}{I(Q)}~\Lip_{\infty}(f)~ \rho\label{eq:bound_tv_approx}\\
    & H(p_Q,p_{Q,\rho}) \leq \sqrt{\tfrac{|Q|}{I(Q)}} ~\Lip_{\infty}(\sqrt{f})~\rho \label{eq:bound_hellinger_approx}\\
    &\Wass_p(p_Q,p_{Q,\rho}) \leq \sqrt{d}\rho,\qquad p \geq 1. \label{eq:bound_wass_approx}
\end{align}
\et 
\begin{proof} 

Recall that $p_Q  = f \ib_{Q}/I(Q)$ and hence 
\[\forall x \in Q,~p_Q(x) = \tfrac{1}{I(Q)}\sum_{Q_\rho \in \dyad_{Q,\rho}}{f(x) \ib_{Q_\rho}(x)}\]
Combining the previous equation with \cref{eq:dyadic_approx_density}, it holds :
\begin{equation}\label{eq:separation_hyper-rectangles}
\forall x \in Q,~ p_Q(x) - p_{Q,\rho}(x) = \tfrac{1}{I(Q)}\sum_{Q_\rho \in \dyad_{Q,\rho}}{(f(x)-\tfrac{I(Q_\rho)}{|Q_\rho|})\ib_{Q_\rho}(x)}
\end{equation}

\paragraph{1. Distance between $f$ and its mean on a small cube.}

Let $Q_\rho \in \dyad_{Q,\rho}$ and $x \in Q_\rho$, it holds
\begin{equation}\label{eq:diff_with_mean}
   ~| f(x) - \tfrac{I(Q_\rho)}{|Q_\rho|}| \leq \Lip_{\infty}(f)~ \rho. 
\end{equation}
Indeed, expanding the mean, we get $f(x) - \tfrac{I(Q_\rho)}{|Q_\rho|} = \tfrac{1}{|Q_\rho|}\int_{Q_\rho}{(f(x) - f(y))~dy}$. Moreover, $|f(x) - f(y)| \leq \Lip_{\infty}(f)~ \|x-y\|_{\infty}$. Plugging that back in the previous equation and using the fact that  $\|x-y\|_{\infty} \leq \rho$ on $Q_\rho$, we get \cref{eq:diff_with_mean}

\paragraph{2. Bounds on the total variation and $L^2$ distances.}
Using \cref{eq:separation_hyper-rectangles,eq:diff_with_mean}, we immediately get 

\begin{align*}
    \int_{Q}{|p_Q(x) - p_{Q,\rho}(x)|dx} &= \tfrac{1}{I(Q)}\sum_{Q_\rho \in \dyad_{Q,\rho}}{\int_{Q_\rho}{|f(x) - \tfrac{I(Q_\rho)}{|Q_\rho|}|dx}}\\
    &\leq \tfrac{|Q|\Lip_{\infty}(f)~ \rho}{I(Q)}.
\end{align*}

\paragraph{3. Bound on the Wasserstein norm $\Wass_p$.} Consider the following density on $Q \times Q$:

\begin{equation}
    \gamma(x,y) = \tfrac{1}{I(Q)}~\sum_{Q_{\rho} \in \dyad_{Q,\rho}}{f(x) \ib_{Q_{\rho}}(x)\tfrac{1}{|Q_{\rho}|}\ib_{Q_{\rho}}(y)}.
\end{equation}
A simple computation shows that $\gamma \in \Pi(p_Q,p_{Q,\rho})$ (see \citet{Santambrogio2015} and \cref{df:wasserstein}), i.e. that its marginals are $p_Q$ and $p_{Q,\rho}$. Hence, by definition \cref{df:wasserstein}, we have 
$$ 
\Wass^p_p(p_Q,p_{Q,\rho}) \leq \tfrac{1}{I(Q)}\sum_{Q_{\rho} \in \dyad_{Q,\rho}}{\int_{Q_{\rho}\times Q_{\rho}}{|x-y|^p \tfrac{f(x)}{|Q_{\rho}|}~dxdy}}.
$$

Now using the fact that if $x,y \in Q_{\rho}$, we have $\|x-y\| \leq \sqrt{d}\rho$ as $Q_{\rho}$ is a hyper-rectangle with all sides of length less than or equal to $\rho$, we finally get :
$\Wass_p(p_Q,p_{Q,\rho}) \leq \sqrt{d}\rho$

\paragraph{4. Hellinger distance bound.} Note that we could get a looser bound using \cref{eq:hellingerl1} which only relies on the Lipschitz constant of $f$ and not on that of $\sqrt{f}$. Here, we concentrate on that case.

Let $Q_{\rho} \in \dyad_{Q,\rho}$.
By the intermediate value theorem, there exists $z \in Q_{\rho}$ such that $f(z) = \tfrac{I(Q_{\rho})}{|Q_{\rho}|}$ and hence for any $x \in Q_{\rho}$, it holds 
\[\left|\sqrt{f}(x) - \sqrt{\tfrac{I(Q_{\rho})}{|Q_{\rho}|}} \right|= \left|\sqrt{f}(x) - \sqrt{f}(z)\right|\leq \Lip_{\infty}(\sqrt{f})~ \|x-z\|_{\infty} \leq \Lip_{\infty}(\sqrt{f})~ \rho.\]
Bounding the distance between $p_{Q,\rho}$ and $p_{Q}$ by decomposing on dyadic hyper-rectangles using the previous expression, it holds 
\begin{align*}
    H(p_Q,p_{Q,\rho})^2 &= \sum_{Q_{\rho} \in \dyad_{Q,\rho}}{\int_{Q_{\rho}}{\left|\sqrt{\tfrac{f(x)}{I(Q)}} - \sqrt{\tfrac{I(Q_{\rho})}{|Q_{\rho}|I(Q)}}\right|^2~dx}} \\
    & = \tfrac{1}{I(Q)}\sum_{Q_{\rho} \in \dyad_{Q,\rho}}{\int_{Q_{\rho}}{\left|\sqrt{f(x)} - \sqrt{\tfrac{I(Q_{\rho})}{|Q_{\rho}|}}\right|^2~dx}}\\
    & \leq \tfrac{(\Lip_{\infty}(\sqrt{f})~~\rho)^2}{I(Q)}\sum_{Q_{\rho} \in \dyad_{Q,\rho}}{\int_{Q_{\rho}}{1~dx}} = \left(\sqrt{\tfrac{|Q|}{I(Q)}}\Lip_{\infty}(\sqrt{f})~ \rho\right)^2.
\end{align*}

\end{proof}

\section{A general method of approximation and sampling}
\label[appendix]{app:proof_general_method}

In this section, we prove \cref{thm:performance_pgauss,thm:performance_p_sample_psd} using mainly results from \citet{rudi2021psd}. We introduce those results sequentially, showing the how each one is a building block towards the final result.

For this section, fix a probability distribution $p$ on the set $\xx = (-1,1)^d$ (this is for the sake of simplicity; any hyper-rectangle could do), and assume that \cref{asm:1b} holds for a certain $\beta \in \N,~ \beta > 0$, i.e. there exists $J \in \N$ and $q_1,...,q_J \in W^\beta_2(\xx) \cap L^\infty(\xx)$ such that $p = \sum_{j}{q^2_j}$. 
In this section, this probability distribution $p$ is only known through a function $\fp$ proportional to its density. Denote with $\Zp > 0$ this proportionality constant, i.e. $\fp/\Zp = p$, and with $f_j$ the renormalized $q_j$ : $q_j / \sqrt{\Zp} = f_j$ s.t. $\fp = \sum_j {f_j^2}$.
Our goal is to be able to generate i.i.d. samples from a distribution as close as possible to $p$.

To do so, we first approximate $\fp$ by a Gaussian PSD model $\fgauss = f(\cdot;\Agauss,\wtx_m,\eta)$ where $\eta = \tau \ib_d$ and $\tau > 0$, $\wtx_m \in \R^{m\times d}$ is obtained as $(\xt_1,...,\xt_m)^\top$ from $m$ i.i.d. uniform samples from $\xx$, and 
 $\Agauss$ is obtained by solving the problem \cref{eq:empirical_problem} which we rewrite here for a given $\lambda > 0$ : 
    \begin{equation}\tag{\ref{eq:empirical_problem}}
    \min_{A \in \psdm(\R^m)}{\int_{\xx}{f(x;A,X,\eta)^2 dx} -2 \sum_{i=1}^n{\fp(x_i)f(x_i;A,X,\eta)} + \lambda \|K^{1/2}AK^{1/2}\|_F},
\end{equation}
where $K = \Kmat{\wtx_m}{\eta}$ and the $(x_i)_{1 \leq i \leq n}$ represented by $X \in \R^{n \times d}$ are $n$ i.i.d. samples from the uniform distribution on $\xx$. 
    
The parameters $\tau,m,n,\la$ are selected in order to have an $\eps$ approximation of the probability $p$. 

Using the fact that we can easily compute integrals of Gaussian PSD models, we can easily have access to $\pgauss = \fgauss/ \Zgauss$ where  $\Zgauss = \|\fgauss\|_{L^{1}(\xx)} = \int_{\xx}{\fgauss(x)~dx}$.

We then apply \cref{alg:sampling} to $\pgauss$, the hyper-rectangle $\xx$, the desired number of samples $N$ and a certain $\rho$ controlling the size of the dyadic decomposition of $\xx$ in order to sample from a distribution whose total variation distance to $p$ is less than a constant times $\eps$.

\paragraph{Existence of a compressed $\eps$-close Gaussian PSD model.}

 We start by invoking \cref{thm:D4} in order to obtain an $\eps$-approximation of $\fp$ in the form of a general PSD $\feps$ with associated operator $\Meps \in \psdm{(\hhe)}$. This PSD model can then be compressed using a compression operator as described in \cref{app:compression_gaussian_psd}. This is the object of the following proposition.

\bp[Compression of $\Meps$]\label{cor:compression_meps}
Let $\eps \in (0,1]$, $\tau \geq \eps^{-2/\beta}$ and define $\eta = \tau \ib_d \in \R^d$. Let $\Meps$ be given by \cref{thm:D4} applied to $\fp$ and satisfying \cref{eq:carac_m_eps} and $\feps$ the corresponding PSD model.

Let $m \in \N$, $\wtx_m \in \R^{m\times d}$ be a data matrix corresponding to vectors $\xt_1,...,\xt_m$ which are sampled independently and uniformly from $\xx$, and $\Projm$ be the associated orthogonal projection in $\hhe$. Let $\Mproj := \Projm \Meps \Projm$ be the operator associated to the compressed PSD model $\fproj$ of $\feps$ onto $\wtx_m$ (see \cref{eq:df_compression_operator} and \cref{df:compression_effect} for the definitions).

Let $\delta \in (0,1]$. If one of the two following are true
\begin{align}
& m \geq \Cp_1 \tau^{d/2} \left(\log \tfrac{\Cp_2}{\eps} + \tfrac{d}{2}\log \tau\right)^{d}\left(\log \tfrac{\Cp_3}{\delta} + \tfrac{d}{2}\log \tau + \log \log \tfrac{\Cp_2}{\eps}\right);\label{eq:bound_m_1}\\
& m \geq \Cpp_1 \eps^{-d/\beta} \left(\log \tfrac{\Cp_2}{\eps}\right)^d \left(\log \tfrac{\Cp_2}{\eps} + \log \tfrac{\Cp_3}{\delta}\right),&& \tau = \eps^{-2/\beta} \label{eq:bound_m_2}
\end{align}
then with probability at least $1-\delta$, it holds 
\begin{align}
   \|\feps - \fproj\|_{L^2(\xx)} &\leq 2^d ~\|\feps - \fproj\|_{L^{\infty}(\xx)}  \leq 2 C~\Nsos{\fp}{\R^d,\beta} \eps \nonumber\\
   \tr{(\Mproj)} &\leq \tr{(\Meps)} \leq C~\Nsos{\fp}{\R^d,\beta} ~\tau^{d/2}\label{eq:eps_compression}
\end{align}
The constants $C,\Cp_1,\Cp_2,\Cp_3,\Cpp_1$ depends only on $d$, $\beta$, and not on $\tau,\eps,m,\delta$.
\ep 

\bpr
Using \cref{eq:carac_m_eps} in \cref{thm:D4} applied to $\fp$, we see that if $\eps \leq 1$ and $\tau \geq \eps^{-2/\beta}$, there exists constants $C_4,C_5$ depending only on $d$ ,$\beta$, and not on $\tau,\eps$ such that $\|f(\cdot;\Meps,\phie) - \fp\|_{L^2(\xx)} \leq C_4 ~\Nsos{\fp}{\R^d,\beta}~\eps$ and $\tr{(\Meps)}\leq C_5~\Nsos{\fp}{\R^d,\beta} ~\tau^{d/2}$ (we set $C_5 =C_2(1+e^{C_3})$ where $C_2,C_3$ are introduced in \cref{thm:D4}). Now setting $\rho = \tfrac{\eps}{2^d~ \tau^{d/2}}$ which is less than $1$ since $\eps \leq 1 $ and $\tau \geq \eps^{-2/\beta} \geq 1$, we can apply \cref{lm:bound_projection} and hence, with probability at least $1 - \delta$, if 
\begin{equation}
    \label{eq:first_bound_m}
m \geq C_1 \tau^{d/2} (\log \tfrac{C_2~\tau^{d/2}}{\eps} )^d \left(\log \tfrac{C_3}{\delta} + \log \tau + \log \log \tfrac{C_2~\tau^{d/2}}{\eps}\right),
\end{equation}
with $C_1 \leftarrow C_1$ from \cref{lm:bound_projection}, $C_2 \leftarrow \max(e,C_2 2^d)$ where $C_2$ is given by \cref{lm:bound_projection} and $C_3 \leftarrow C_3$ from \cref{lm:bound_projection}, it holds $\sup_{x \in \xx}{\|(I - \Projm)\phie(x)\|} \leq \rho$ (hence $C_1,C_2,C_3$ depend only on $d$). 

\paragraph{1.} Let us now show that \cref{eq:first_bound_m} is implied by \cref{eq:bound_m_1}.
Let us bound :
\begin{align*}
\log \log \tfrac{C_2~\tau^{d/2}}{\eps} &= \log \left(\log \tfrac{C_2}{\eps}\left(1 + \tfrac{d/2\log \tau}{\log \tfrac{C_2}{\eps}}\right)\right)\\
& = \log \log \tfrac{C_2}{\eps} + \log \left(1 + \tfrac{d/2\log \tau}{\log \tfrac{C_2}{\eps}}\right)\\
&\leq \log \log \tfrac{C_2}{\eps} + \tfrac{d}{2}\log \tau ,
\end{align*}
where the last inequality is obtained since $\log(1+t) \leq t$ and $C_2/\eps \geq C_2 \geq e$ by definition of $C_2$ and since $\eps \leq 1$. Setting $\Cp_1 = 3 C_1$, $\Cp_2 = C_2$ and $\Cp_3 = C_3$, it is therefore clear that \cref{eq:bound_m_1} implies \cref{eq:first_bound_m}.

\paragraph{2.} Moreover, \cref{eq:bound_m_1} is in turn implied by \cref{eq:bound_m_2}. Indeed, in the case where $\tau = \eps^{-2/\beta}$, we have the bound
 \[\log\log \tfrac{\Cp_2}{\eps} + \tfrac{d}{2}\log \tau  \leq \log \tfrac{\Cp_2}{\eps} + \tfrac{d}{2}\log \tau = \log \tfrac{\Cp_2}{\eps} + \tfrac{d}{\beta}\log \tfrac{1}{\eps}\leq (1+d/\beta)\log\tfrac{\Cp_2}{\eps}\]
 since $\Cp_2 \geq e \geq 1$. Thus, taking $\Cpp_1 = \Cp_1(1 + d/\beta)^{d+1}$, \cref{eq:bound_m_2} implies \cref{eq:bound_m_1}.
 
 \paragraph{3.} If \cref{eq:first_bound_m} holds, then \cref{eq:eps_compression} holds with probability at least $1-\delta$. Indeed, for the first part, since \cref{eq:first_bound_m} holds, with probability at least $1-\delta$, $\sup_{x \in \xx}\|(I-\Projm)\phie(x)\| \leq \rho = \tfrac{\eps}{2^d  \tau^{d/2}}$
 
 Moreover, using \cref{eq:compression_effect} combined with the fact that for any $x \in \xx,~|f(x;\Meps,\phie)| = |\scal{\phie(x)}{\Meps\phie(x)}| \leq \|\phie(x)\|_{\hhe}^2 \|\Meps\| = \|\Meps\|$ since $\|\phie(x)\|^2 = k_\eta(x,x) = 1$, it holds 
 \[\|f(\cdot; \Meps,\phie) - f(\cdot;\Mproj,\phie)\|_{L^{\infty}(\xx)} \leq \|\Meps\|(\rho^2 + \rho) \leq 2\|\Meps\|\rho .\]
 We conclude using the fact that for any operator $M$, and any orthogonal projection $P$, $\|M\| \leq \tr{(M)}$ and $\tr{(PMP)} \leq \tr{(M)}$. We then conclude the proof by using the definition of $\rho$ and the fact that $\int_{\xx}{1 ~ dx} = 2^d$, and setting $C\leftarrow C_5$. 

\epr 

Combining \cref{eq:carac_m_eps} and \cref{eq:eps_compression}, we see that if $m$ is large enough, one can find a Gaussian PSD model of the form $\fproj = f(~\cdot~;\Aproj,\wtx_m,\tau \ib_d)$ (where $\Aproj$ is defined through \cref{df:compression_effect} from $\Mproj$) which is $C\Nsos{\fp}{\R^d,\beta}\eps$ close to $\fp$ and whose trace is controlled. It now remains to compare the performance of $\fproj$ with the Gaussian PSD model learned from evaluations of $\fp$,  $\fgauss$, which is the solution of \cref{eq:empirical_problem} which we can compute.  

\paragraph{Controlling the $L^2$ distance between $\fgauss$ and $\fp$.}
This theorem is a rewriting of Theorem 7 of \citet{rudi2021psd}, but with the point of view of $\eps$ instead of $n$.

\bp[Performance of $\fgauss$]\label{thm:performance_fgauss_app} Let $n \in \N$ and let $(x_1,...,x_n)$ be $n$ i.i.d. samples from $p$. Let $\delta \in (0,1]$ and $\eps \leq \tfrac{1}{e}$. Assume $n$ satisfies
\begin{equation}
    \label{eq:bound_n}
    n \geq \eps^{-(d + 2\beta)/\beta}\log^{d}\left(\tfrac{1}{\eps}\right)\log\left(\tfrac{2}{\delta}\right),
\end{equation}
Let $m \in \N$ and assume $m$ satisfies \cref{eq:bound_m_2} and let $\wtx_m \in \R^{m\times d}$ be a data matrix corresponding to vectors $\xt_1,...,\xt_m$ which are sampled independently and uniformly from $\xx$. Let $\lambda = \eps^{2(\beta+d)/\beta}$, $\tau = \eps^{-2/\beta}$ and $\fgauss$ be the Gaussian PSD model associated to the solution $\Agauss$ of \cref{eq:empirical_problem} with $\wtx_m,\la,\tau$. With probability at least $1-2\delta$, the following holds
\begin{equation}\label{eq:error_f_gauss}
    \left(\|\fgauss - \fp\|^2_{L^2(\xx)} + \la \| \Mgauss \|_F^2 \right)^{1/2}\leq C~\Nsos{\fp}{\xx,\beta}~\eps ,
\end{equation}
where $C$ is a constant depending only on $d$, $\beta$, and not on $\eps,\delta, \lambda,m,\tau,\fp$.
\ep

\begin{proof} We start by applying the same reasoning as in the proof of Theorem 7 by \citet{rudi2021psd}.

Note that since $\tau = \eps^{-2/\beta}$ and \cref{eq:bound_m_2} is satisfied, with probability at least $1-\delta$, it holds $\|\fproj - \fgauss\|_{L^2(\xx)} \leq 2C_1 ~\Nsos{\fp}{\R^d,\beta}~ \eps$ (where $C_1 \leftarrow C$ from \cref{eq:eps_compression}) and hence $\|\fp - \fgauss\|_{L^2(\xx)} \leq (C_0+2C_1)~\Nsos{\fp}{\R^d,\beta}~\eps$, (where $C_0 \leftarrow C_1$ from \cref{thm:D4}). $C_0,C_1$ are both constants depending only on $d,\beta$. Moreover, since the Frobenius norm is bounded by the trace norm, by definition of $\tau$, we also have $\|\Mgauss\|_{F} \leq \tr{(\Mgauss)} \leq C_1~\Nsos{\fp}{\R^d,\beta}~ \tau^{d/2} \leq C_1 ~\Nsos{\fp}{\R^d,\beta}~\eps^{-d/\beta}$.

We can modify Theorem E.2 from \citet{rudi2021psd} by taking $\widehat{v} = \tfrac{1}{n} \sum_{i=1}^n{\fp(x_i)\psi_{\eta}(x_i)}$ and $v = \int_{\xx}{\fp(x)\psi_{\eta}(x)~dx}$; all the formulas then remain true and adapt to our problem \cref{eq:empirical_problem}. Applying Theorem E.2 from \citet{rudi2021psd} to $\Aproj$ and using Lemma E.3 of \citet{rudi2021psd} to simplify notation, as well as the bound on the term $\|Q_{\lambda}^{-1/2}(\widehat{v} - v)\|$ combining Lemma E.4 (with $\zeta = Q_{\la}^{-1/2}\fp(x) \psi_{\eta}(x)$) using $s = d$ and Lemma E.5 (again, for more details, see  part 2 of the proof of Theorem 7 by \citet{rudi2021psd}) and using the fact that $\sqrt{a+b} \leq \sqrt{a} + \sqrt{b},~a,b \geq 0$,
with probability at least $1-\delta$, it holds :
\begin{align}
\left(\|\fgauss - \fp\|_{L^2(\xx)}^2 + \la \|\Mgauss\|_{F}^2\right)^{1/2} 
&\leq \|\fproj - \fp\|_{L^2(\xx)\nonumber}\\
&+\sqrt{\la} \|\Mproj\|_{F} + C_2~\Nsos{\fp}{\R^d,\beta}~\frac{\log\tfrac{2}{\delta}}{n\lambda^{1/4}} \nonumber \\
&+ C_3~\Nsos{\fp}{\R^d,\beta}~\frac{\tau^{d/4} \left(\log\tfrac{1}{\lambda}\right)^{d/2} \left(\log\tfrac{2}{\delta}\right)^{1/2}}{n^{1/2}},\label{eq:bound_perf_1}
\end{align}
where $C_2$ and $C_3$ are constants which depend only on $d$.

Note that in the proof of Lemma E.4 of \citet{rudi2021psd}, $\|\zeta\|$ is bounded in essential supremum and standard deviation by $\|\fp\|_{L^{\infty}(\xx)}\times $ a quantity independent of $\fp$ which is then bounded, hence the previous concentration bound since $\|\fp\|_{L^{\infty}(\xx)} \leq \Nsos{\fp}{\R^d,\beta}$.  

Now combining both events in a union bound, and plugging in the fact that $\lambda = \eps^{\tfrac{2\beta + 2d}{\beta}}$ and $\tau = \eps^{-2/\beta}$, we see that with probability at least $1-2\delta$, the left hand term is bounded by the following quantity:
\begin{align}
&\eps~\Nsos{\fp}{\R^d,\beta}~(C_0+3C_1+T),\label{eq:bound_perf_2}\\
T &=  C_2\frac{\eps^{-\tfrac{3\beta + d}{2\beta}}\log\tfrac{2}{\delta}}{n}
+ C_3\frac{\eps^{-(d+2\beta)/2\beta} \left(\tfrac{2\beta +2d}{\beta}\log\tfrac{1}{\eps}\right)^{d/2} \left(\log\tfrac{2}{\delta}\right)^{1/2}}{n^{1/2}}. \nonumber
\end{align}
Now the goal is to bound the term $T$. Note that as soon as $\eps \leq e^{-1}$ and $\delta \leq 2$, if $Y = \frac{\eps^{-(d + 2\beta)/\beta}\log^{d}\left(\tfrac{1}{\eps}\right)\log\left(\tfrac{2}{\delta}\right)}{ n}$, then it holds $T \leq \tfrac{C_2}{\log^{2} 2} Y + C_3 \sqrt{Y}$. Now note that $Y \leq 1$ iif $n \geq \eps^{-(d + 2\beta)/\beta}\log^{d}\left(\tfrac{1}{\eps}\right)\log\left(\tfrac{2}{\delta}\right)$. The theorem therefore holds with $C \leftarrow 1+3C_1+C_2/\log^{2}2 +C_3$.

Finally, the fact that all bounds involving $\Nsos{\fp}{\R^d,\beta}$ can be replaced, up to constants depending only on $\beta,d$, by the norm $\Nsos{\fp}{\xx,\beta}$, is simply a consequence of \cref{cor:extension-intersection}.
\end{proof}

We now come to the final part of our section detailing the proof of \cref{thm:performance_pgauss_app,thm:performance_p_sample_psd_app}, which consists in approximately sampling from the learnt model $\fgauss$ using \cref{alg:sampling} with well chosen parameters.

\paragraph{Performance of the re-normalized probability measure $\pgauss$.}

We start off with a technical lemma.

\blm[Technical lemma]\label{lm:technical}
Let $\|\cdot\|$ be a norm on a vector space $E$, and let $x,y \in E\setminus{\{0\}}$. Then it holds:
\begin{equation}
    \label{eq:technical_lemma}
    \left\|\tfrac{a}{\|a\|} - \tfrac{b}{\|b\|}\right\| \leq \frac{2\|a-b\|}{\|a\|}.
\end{equation}
Moreover, if $\|a-b\| \leq \|a\|/2$, it holds
\begin{equation}
    \label{eq:technical_lemma_2}
    \tfrac{\|a\|}{\|b\|} \leq 2.
\end{equation}
\elm

\bpr 
Introduce the quantity $\tfrac{b}{\|a\|}$ in order to get 
$$ \left\|\tfrac{a}{\|a\|} - \tfrac{b}{\|b\|}\right\| \leq \left\|\tfrac{a}{\|a\|} - \tfrac{b}{\|a\|}\right\| +\left\|\tfrac{b}{\|a\|} - \tfrac{b}{\|b\|}\right\|=\frac{\|a-b\|}{\|a\|} + \|b\|\left|\tfrac{1}{\|a\|} - \tfrac{1}{\|b\|}\right|.$$
One concludes by writing 
$$\left|\tfrac{1}{\|a\|} - \tfrac{1}{\|b\|}\right| = \frac{\left|\|b\| - \|a\|\right|}{\|a\|~\|b\|} \leq \frac{ \|b-a\|}{\|a\|~\|b\|} ,$$
where the last inequality is simply the triangle inequality. This concludes the proof of \cref{eq:technical_lemma}. The proof of \cref{eq:technical_lemma_2} is simply the result of applying the bound $\tfrac{1}{\|b\|} \leq \tfrac{1}{\|a\| - \|b-a\|} \leq \tfrac{2}{\|a\|}$.
\epr 

\bp[Performance of $\pgauss$]\label{thm:performance_pgauss_app} Let $p$ be a probability density w.r.t. the Lebesgue measure on $\xx = (-1,1)^d$ satisfying \cref{asm:1b} for a certain $\beta$. There exists $\eps_0 > 0$ depending only on $d,\beta$, and $\Nsos{p}{\xx,\beta}$ and $C_1,C_2, \Cp_1,\Cp_2,\Cp_3$ depending only on $d,\beta$ such that the following holds.

Let $n \in \N$ and let $(x_1,...,x_n)$ be $n$ i.i.d. samples selected uniformly at random from $\xx$. Let $\delta \in (0,1]$ and $\eps \leq \eps_0$, $\lambda = \eps^{2(\beta+d)/\beta}$ and  $\tau = \eps^{-2/\beta}$. Assume $n$ satisfies \cref{eq:bound_n}, i.e. 
\begin{equation}
    \tag{\ref{eq:bound_n}}
    n \geq \eps^{-(d + 2\beta)/\beta}\log^{d}\left(\tfrac{1}{\eps}\right)\log\left(\tfrac{2}{\delta}\right).
\end{equation}
Let $m \in \N$ and assume $m$ satisfies \cref{eq:bound_m_2}, i.e.
\begin{equation}
\tag{\ref{eq:bound_m_2}}
m \geq \Cp_1 \eps^{-d/\beta} \left(\log \tfrac{\Cp_2}{\eps}\right)^d \left(\log \tfrac{\Cp_2}{\eps} + \log \tfrac{\Cp_3}{\delta}\right),
\end{equation}

and let $\wtx_m \in \R^{m\times d}$ be a data matrix corresponding to vectors $\xt_1,...,\xt_m$ which are sampled independently and uniformly from $\xx$. 

Let $\fgauss$ be the Gaussian PSD model associated to the solution $\Agauss$ of \cref{eq:empirical_problem} with $\wtx_m,\la,\tau$ and let $\pgauss$ be the associated probability density on $\xx$ (i.e. the re-normalization of $\fgauss$). Let $\Rgauss$ be PSD operator on $\hhe$ associated to $\pgauss$. With probability at least $1-2\delta$, it holds
\begin{equation}\label{eq:error_p_gauss}
    d_{TV}(\pgauss,p) \leq C_1~\Nsos{p}{\xx,\beta}~\eps,\qquad   \| \Rgauss \|_F \leq C_2~\Nsos{p}{\xx,\beta}~\eps^{-d/\beta}.
\end{equation}
\ep

\bpr 
Since the assumptions of \cref{thm:performance_fgauss_app} are satisfied, we have by \cref{eq:error_f_gauss} the existence of a constant $C$ depending only on $d$, $\beta$, and not on $\eps,\delta, \lambda,m,\tau,\fp$, such that 
\begin{equation}\label{eq:error_f_gauss_bis}
\|\fgauss - \fp\|_{L^2(\xx)} \leq C~\Nsos{\fp}{\xx,\beta}~\eps,\qquad \| \Mgauss \|_F\leq C~\Nsos{\fp}{\xx,\beta}~\eps^{-d/\beta} ,
\end{equation}
where we have used the fact that $\la = \eps^{2 + 2d/\beta}$.

Now using the fact that $\|\bullet\|_{L^{1}(\xx)}\leq 2^{d/2}\|\bullet\|_{L^2(\xx)}$ (by Cauchy-Schwarz inequality), \cref{eq:error_f_gauss_bis} shows in particular that $\|\fgauss - \fp\|_{L^1(\xx)} \leq 2^{d/2} C ~\Nsos{\fp}{\xx,\beta}~\eps$. Now applying \cref{eq:technical_lemma} of \cref{lm:technical}, using the fact that $\pgauss = \fgauss/\|\fgauss\|_{L^1(\xx)}$ and $p = \fp/\|\fp\|_{L^1(\xx)}$, it holds 
\begin{align}
   d_{TV}(\pgauss,p) =  \|\pgauss - p\|_{L^1(\xx)} & \leq 2\|\fgauss -\fp\|_{L^1(\xx)} / \|\fp\|_{L^1(\xx)}\\
   &\leq 2^{d/2 + 1} \nonumber C~\Nsos{\fp}{\xx,\beta}/\|\fp\|_{L^1(\xx)}~\eps.
\end{align}

Since $p = \fp / \|\fp\|_{L^1(\xx)}$, we have $\Nsos{\fp}{\xx,\beta}/\|\fp\|_{L^1(\xx)} = \Nsos{p}{\xx,\beta}$. This shows 
$$d_{TV}(\pgauss,p)  \leq  2^{d/2+1}C~\Nsos{p}{\xx,\beta} \eps.$$
Now set $\eps_0 = \min(e^{-1},2^{-d/2-1}C^{-1}\Nsos{p}{\xx,\beta}^{-1})$. If $\eps \leq \eps_0$, we have $2^{d/2}C\Nsos{\fp}{\xx,\beta}\eps \leq \|\fp\|_{L^1(\xx)}/2$ and hence $\|\fgauss - \fp\|_{L^1(\xx)} \leq \|\fp\|_{L^1(\xx)}/2$. By \cref{eq:technical_lemma_2} of \cref{lm:technical}, we therefore have $\|\fp\|_{L^1(\xx)}/\|\fgauss\|_{L^1(\xx)} = \Zp/\Zgauss\leq 2$. Now since $\Rgauss = \Mgauss / \Zgauss$, using \cref{eq:error_f_gauss_bis}, it holds $\|\Rgauss\| \leq C_2~\Nsos{p}{\xx,\beta}~\eps^{-d/\beta}$ where $C_2 = 2C$, which depends only on $\beta,d$. 

\epr 

\bt[Performance of $\psample$]\label{thm:performance_p_sample_psd_app} Under the assumptions and notations of the previous theorem (\cref{thm:performance_pgauss}), there exists a constant $C_3$ depending only on $d,\beta$, such that the following holds.

Let $\pgauss$ be given by the previous proposition. Let $\psample$ be the dyadic approximation of $\pgauss$ on $Q = \xx = (-1,1)^d$ and of width $\rho$ (see \cref{eq:dyadic_approx_density}). Recall from \cref{thm:approximation_distribution} that \cref{alg:sampling} applied to $Q = (-1,1)^d,N,\rho$ returns $N$ i.i.d. samples from $\psample$. 

If on the one hand $\rho$ is set to $\eps^{1+(d+1)/\beta}$, then with probability at least $1-2\delta$,
\begin{equation}
    \label{eq:bound_final_psd_1}
    d_{TV}(\pgauss,\psample) \leq C_3~\Nsos{p}{\xx,\beta}~ \eps ,\qquad d_{TV}(p,\psample) \leq (C_1+C_3)~\Nsos{p}{\xx,\beta}~\eps.
\end{equation}
If on the other $\rho$ is set adaptively to guarantee $d_{TV}(\psample,\pgauss) \leq \eps$ as in \cref{rk:adaptive_rho} then with probability at least $1-2\delta$,
    $\rho \geq \eps^{1+(d+1)/\beta}/(C_3~\Nsos{p}{\xx,\beta})$,
and hence
\begin{equation}\label{eq:bound_final_psd_1_adaptive}
    d_{TV}(\pgauss,\psample) \leq \eps ,\qquad d_{TV}(p,\psample) \leq C_1~\Nsos{p}{\xx,\beta}~\eps + \eps.
\end{equation}
In any case, this guarantees that the complexity in terms of $\erf$ computations is bounded by
\begin{equation}
    \label{eq:complexity_psd}
    O(N m^2\log\tfrac{1}{\rho}) = O\left(N~\eps^{-2d/\beta} \log^{2d+1}\left(\tfrac{1}{\eps}\right) \left(\log\left(\tfrac{1}{\eps}\right) +\log\left(\tfrac{1}{\delta}\right)\right)\right),
\end{equation}
where the $O$ notations is taken with constants depending on $d,\beta$, $\Nsos{p}{\xx,\beta}$.
\et

\bpr 
Let us bound $\Lip_{\infty}(\pgauss)$. 
Note that 
$$ \Lip_{\infty}(\pgauss) \leq  \sup_{x \in \xx}{\sum_{k=1}^d{\partial_k \pgauss(x)}}.$$
Using \cref{lm:der_bound_psd}, we get $\Lip_{\infty}(\pgauss)~ \leq d 2^{3/2}\sqrt{\tau}\|\Rgauss\|$. Using the fact that $\tau = \eps^{-2/\beta}$ and that by \cref{eq:error_p_gauss}, $\|\Rgauss\| \leq \|\Rgauss\|_{F} \leq C_2~\Nsos{p}{\xx,\beta}~ \eps^{-d/\beta}$, we therefore have $\Lip_{\infty}(\pgauss)~ \leq 2^{3/2} d C_2~\Nsos{p}{\xx,\beta}~ \eps^{-(d+1)/\beta}$. Hence, applying \cref{thm:variation_bounds_evolved} to $\pgauss$, we get
\begin{equation}
    \label{eq:bound_perfo_ro_psd}
    d_{TV}(\psample,\pgauss) \leq 2^{3/2}~ 2^d  d C_2~\Nsos{p}{\xx,\beta}~ \eps^{-(d+1)/\beta} ~\rho.
\end{equation}
On the one hand, if we use \cref{alg:sampling} with $\rho = \eps^{1 +\tfrac{(d+1)}{\beta}} $, by the previous equation, we get $d_{TV}(\psample,\pgauss) \leq 2^{3/2} d ~ 2^d C_2~\Nsos{p}{\xx,\beta}~ \eps$.

If on the other hand we find $\rho$ adaptively by computing a bound $$\Lipt(A) = 2^{3/2}\tau^{1/2}d\|K^{1/2}AK^{1/2}\| =2^{3/2}\tau^{1/2}d\|\Rgauss\|_{F}$$
from $\pgauss$ as in \cref{rk:adaptive_rho}, and finding $\rho$ such that $2^d \Lipt(A)~ \rho = \tfrac{|Q|}{I(Q)}\Lipt(A)~ \rho = \eps$, since the adaptive bound will have computed 
$$\Lipt(A)~ \leq 2^{3/2} d C_2~\Nsos{p}{\xx,\beta}~ \eps^{-(d+1)/\beta},$$ we will get $\rho \geq \tfrac{\eps^{1+(d+1)/\beta}}{2^{d+3/2}~d~C_2~\Nsos{p}{\xx,\beta}}$ and hence $d_{TV}(\psample,\pgauss) \leq \eps$. The last point is just a consequence of \cref{thm:approximation_distribution} and the bound on $m$ in \cref{eq:bound_m_2}.
\epr

\section{Approximation and sampling using a rank one PSD model}
\label[appendix]{app:target_distribution}

In this section, we prove the results in \cref{secsec:from_distribution}, i.e. \cref{thm:bound_learning_hellinger} and \cref{thm:performance_p_sample_hellinger}.

For this section, fix a probability which has density $p$ with respect to the Lebesgue measure $dx$ on $\xx = (-1,1)^d$, (this is for the sake of simplicity; any hyper-rectangle could do), and assume that \cref{asm:1a} holds for a certain $\beta \in \N,~ \beta > 0$, i.e. there exists $q \in W^\beta_2(\xx) \cap L^\infty(\xx)$ such that $p = q^2$. This is the case, for instance, when $p \propto e^{-V(x)}$ where $V$ is $\beta$ times differentiable.

One of the main advantages of our method will be to deal with probability measures which are known up to a constant; therefore, in this section, we take $\fp$ such that $p = \fp/Z(\fp)$ where $Z(\fp) = \int_{\xx}{\fp(x)dx}$. Assuming \cref{asm:1a} holds, we take $\gp \in W^\beta_2(\xx) \cap L^\infty(\xx)$ such that $\gp^2 = \fp$ as  and assume that $p$ is only known through function evaluations of $\gp$, i.e. we can evaluate the function $\gp(x)$ for any $x \in \xx$. 

Once again, our goal is to be able to generate $N$ i.i.d. samples from a distribution which is $\eps$-close to $p$, in a sense which we will define. 
To do so, we first approximate $\gp $ by a Gaussian linear model $\ggauss = \ppl{\bullet}{\agauss,\wtx_m,\eta}$ (see \cref{eq:df_gaussian_linear} for a definition) where $\eta = \tau \ib_d$ for some $\tau > 0$, $\wtx_m \in \R^{m\times d}$ is obtained as $(\xt_1,...,\xt_m)^\top$ from $m$ i.i.d. uniform samples from $\xx$, and 
 $\agauss$ is obtained by solving the problem \cref{eq:problem_hellinger} which we rewrite here for a given $\lambda > 0$ and for $n$ i.i.d. samples $(x_1,..,x_n)$ sampled uniformly from $\xx$: 
\begin{equation}\tag{\ref{eq:problem_hellinger}}
    \agauss = \argmin{a \in \R^m}{\tfrac{1}{n}\sum_{i=1}^n{\left(\ppl{x_i}{ a,\xt_m,\tau \ib_d} - \gp(x_i)\right)^2} + \lambda a^\top \Kmat{\wtx_m}{\eta} a }.
\end{equation}

This yields a Gaussian linear model 
$\ggauss \in \hhe $ of $\gp$. Since $\ggauss^2 = \fgauss$ is a PSD model (indeed $\fgauss = \pp{\bullet}{\Agauss,\wtx_m,\tau\ib_d}$ with $\Agauss = \agauss\agauss^{\top}$), we can see $\fgauss$ as a Gaussian PSD model of $\fp$, and hence its renormalized version $\pgauss$ as a PSD model of $p$.

The parameters $\tau,m,\la,n$ are selected in order to have an $\eps$ approximation of the probability $p$. 

Furthermore, note that the first term in the optimized quantity in \cref{eq:problem_hellinger} is an empirical version of the quantity $$ \tfrac{1}{|\xx|}{\int_{\xx}{\left|\sqrt{\fgauss(x)} - \sqrt{\fp}(x)\right|^2 ~dx}} \leq \tfrac{1}{|\xx|}{\int_{\xx}{\left|\ggauss(x) -\gp(x)\right|^2 ~dx}}.$$
This quantity is related to Hellinger distance $H(p,\pgauss)$ defined in \cref{df:hellinger}.

This will therefore be the natural measure in which to express the quality of the approximation $\pgauss$ of $p$ in this section.

The bound obtained on the performance of $\pgauss$ can be decomposed into two steps.
\begin{itemize} 
\item We start by bounding the distance between any $g \in \hhe$ and $\ggauss$ in \cref{thm:rudi_2_2015}.
\item We then select a $\geps$ which is $\eps$-close to $\gp$, and use it as a reference point in order to bound the distance between $\gp$ and $\ggauss$. To do so, we need to apply different concentration inequalities to obtain a final bound in terms of performance for both $\fgauss$ with respect to $\fp$ and $\pgauss$ with respect to $p$ in Hellinger distance in \cref{thm:bound_learning_hellinger}. 
\end{itemize}

\paragraph{Bound on the performance of $\ggauss$ compared to an arbitrary function $g$.} Here, we adapt Theorem 2. from \citet{rudi2015less}. 

\bt[Bounding the error \citep{rudi2015less}]\label{thm:rudi_2_2015} Let $\eta \in \R^{d}_{++}$ and $g \in \hhe$. 
\begin{align}
\|\Cl^{1/2}(g - \ggauss)\| &\leq \theta_1^2 \theta_2~\|\gph - \Sn g\|_{\R^n} \label{eq:bound_error_1}\\
&+ \|g\|_{\hhe}(1+\theta_1\theta_2 + \theta_1^2)~ \left(\sup_{x \in \xx}{\|(I - \Projm)\phie(x)\|} +  \lambda^{1/2}\right),\nonumber
\end{align} 
where $\theta_1 = \|\Cnl^{-1/2} \Cl^{1/2}\|$, $\theta_2 = \|\Cnl^{1/2} \Cl^{-1/2}\|$ and $\gph = (\gp(x_i)/\sqrt{n})_{1 \leq i \leq n} \in \R^n$.

\et 

\begin{proof} Let $g \in \hhe$. We can apply  a modification of Theorem 2 by \citet{rudi2015less}. Indeed, consider in the notations of \citet{rudi2015less} the loss ${\cal E}(f)= \|\Cop^{1/2}(f-g)\|_{\hhe}$, and note that the assumptions are satisfied with $\nu = 0$ and $R = \|g\|_{\hhe}$, since $g$ minimizes $\cal E$ and $\|C^{-0}g\|_{\hhe} = \|g\|_{\hhe}$. Moreover, note that in the proof of that theorem, one can replace $\Cop$ by $\Cl$ without changing the result (indeed, in the proof, one always bounds $\|\Cop^{1/2}\star\| \leq \|\Cop^{1/2}\Cl^{-1/2}\|~\|\Cl^{1/2}\star\| \leq \|\Cl^{1/2}\star\|$). Thus, in that setting, without combining the "constant" terms in the bounds and looking into the proof of Theorem 2 of \citet{rudi2015less}, it holds
\begin{equation}
    \label{eq:bound_cl_1}
    \|\Cl^{1/2}(\ggauss - g)\| \leq \theta_1^2~\|\Cl^{-1/2}\Sn^* (\gph - \Sn g)\| + R (1+\theta_1\theta_2)\|(I - \Projm)\Cl^{1/2}\| + R\theta_1^2 \la^{1/2},
\end{equation}
where $\theta_1 = \|\Cnl^{-1/2} \Cl^{1/2}\|$ and $\theta_2 = \|\Cnl^{1/2} \Cl^{-1/2}\|$.

Note that $\|\Cl^{-1/2}\Sn^* (\gph - \Sn g)\| \leq  \|\Cl^{-1/2}\Sn^*\|~ \|\gph - \Sn g\|_{\R^n} \leq \theta_2~\|\gph - \Sn g\|_{\R^n}$ since $\|\Cl^{-1/2}\Sn^*\|^2 = \|\Cl^{-1/2}\Cn \Cl^{-1/2}\| \leq\|\Cl^{-1/2}\Cnl \Cl^{-1/2}\| = \theta_2^2$. 

Moreover, using the definition of $\Cop$, it holds
\begin{align*}
    \|(I - \Projm)\Cl^{1/2}\|^2 &= \|(I-\Projm)\Cop(I-\Projm) + \lambda(I-\Projm)\|\\
    & \leq \tfrac{1}{|X|}\left\|\int_{\xx}{(I-\Projm)\phie(x) \otimes \phie(x)(I-\Projm)~dx}\right\| + \lambda\|(I-\Projm)\|\\
    &\leq \sup_{x \in \xx}{\|(I-\Projm)\phie(x)\|^2} + \lambda.
\end{align*}
Combining these results and using the fact that $\sqrt{a+b} \leq \sqrt{a} + \sqrt{b}$ for any $a,b \geq 0$, we get the bound.
\end{proof}
\paragraph{Performance of $\pgauss$.} We can now state the main results of this section, i.e. the bound on the performance of $\pgauss$.

\bp[Performance of $\pgauss$]\label{thm:bound_learning_hellinger_app}
Let $p$ be a probability density on $\xx = (-1,1)^d$, and assume $p = q^2$ and $q \in L^{\infty}(\xx) \cap W^{\beta}_2(\xx)$ for some $\beta \geq 0$. Let $\nut > \min(1,d/(2\beta))$. There exists a constant $\eps_0$ depending only on $\|q\|_{L^{\infty}(\xx)},\| q\|_{W^{\beta}_2(\xx)},\beta,d$, constants $C_1,C_2,C_3,C_4,C_5$ depending only on $\beta,d$ and a constant $\Cp_1$ depending only on $\beta,d,\nut$ such that the following holds.

Let $\delta \in (0,1]$ and $\eps \leq \eps_0$, and assume $(x_1,...,x_n)$ and $(\xt_1,...,\xt_m)$ are respectively $n$ and $m$ uniform i.i.d. samples on $\xx$, satisfying
\begin{align}
    &m \geq C_1 \eps^{-d/\beta} \log^d \tfrac{C_2}{\eps} \log\tfrac{C_3}{\delta \eps}\label{eq:bound_m_final_hell}\\
    &n \geq \Cp_1 \eps^{-2\nut} \log \tfrac{8}{\delta} \label{eq:bound_n_final_hell}
\end{align}
Let $\tau = \eps^{-2/\beta}$, $\eta = \tau \ib_d$ and $\la = \eps^{2 + d/\beta}$. Let $\agauss \in \R^n$ be the vector obtained by solving \cref{eq:problem_hellinger} and $\ggauss \in \hhe$ the associated Gaussian linear model (see \cref{eq:df_gaussian_linear}). Let $\fgauss = \ggauss^2$ be the associated Gaussian PSD model, $\Zgauss=\int_{\xx}{\fgauss(x)~dx}$ be the normalizing constant, and $\pgauss = \fgauss/\Zgauss$ be the renormalized PSD model, which is a probability density. Let $\Rgauss$ be PSD operator in $\psdm{(\hhe)}$ associated to $\pgauss$.

With probability at least $1-3\delta$, it holds 
\begin{align}
    H(\pgauss,p) &\leq C_4\|q\|_{L^{\infty}(\xx)\cap W^{\beta}_2(\xx)}\eps \nonumber\\
    \tr{(\Rgauss)} &= \left\|\tfrac{\ggauss}{\sqrt{\Zgauss}}\right\|^2_{\hhe}\leq C_5 \|q\|^2_{L^{\infty}(\xx)\cap W^{\beta}_2(\xx)}\eps^{-d/\beta},\label{eq:error_p_gauss_hellinger}
\end{align}
where $\|\bullet\|_{L^{\infty}(\xx)\cap W^{\beta}_2(\xx)} = \max(\|\bullet\|_{W^{\beta}_2(\xx)},\|\bullet\|_{L^{\infty}(\xx)})$.
\ep

\begin{proof} 
Let $\tau > 0$, and define $\eta = \tau \ib_d$. By \cref{cor:extension-intersection}, we can extend $\gp$ to the whole of $\R^d$ and there exists an constant $C$ such that $\|\gp\|_{W^{\beta}_2(\R^d)} \leq \|\gp\|_{W^{\beta}_2(\xx)}$ and $\|\gp\|_{L^{\infty}(\R^d)} \leq C \|\gp\|_{L^{\infty}(\xx)}$. We still denote with $\gp$ such an extension. Let $\geps$ be given by \cref{prp:approximation_eps} when approximating $\gp$.

Setting $\tau = \eps^{-2/\beta}$ and $\lambda = \eps^{\tfrac{2\beta + d}{\beta}}$, since we assume $\eps \leq 1$, \cref{eq:approx_gaussian} gives us two constants $C_1,C_2$ depending only on $\beta,d$ such that  

\[\left\{
\begin{aligned}
\|\geps - \gp\|_{L^2(\R^d)}& \leq \eps \|\gp\|_{W^{\beta}_2(\R^d)}\\ \|\geps - \gp\|_{L^{\infty}(\R^d)} &\leq C_1 ~\eps^{1-\nu}~\|\gp\|_{\bullet}  
\end{aligned}
\right.
\|\geps\|_{\hhe} \leq C_{2}~\|\gp\|_{W^{\beta}_2(\R^d)} ~\tau^{d/4}=C_2~\|\gp\|_{W^{\beta}_2(\R^d)}\eps^{-\tfrac{d}{2\beta}}.\]

\paragraph{1. Bounding $\|\gph - \Sn \geps\|_{\R^n}$}
Apply Theorem 3 of \citet{boucheron2013}, reformulated in Proposition 10 from \citet{rudi2015less}.
Consider the random variable $\zeta = (\geps-\gp)(X)^2 - \tfrac{1}{|\xx|}\|\geps - \gp\|_{L^2(\xx)}^2$ where $X$ follows the uniform law on $\xx$. Then $|\zeta| \leq \|\geps - \gp\|^2_{L^{\infty}(\xx)} $ almost surely, and $\mathbb{E}[\zeta^2] \leq\|\geps -\gp\|^2_{L^{\infty}(\xx)}\tfrac{1}{|\xx|}\|\geps-\gp\|^2_{L^{2}(\xx)} $. Applying the concentration bound yields that with probability at least $1-\delta$, it holds 
\begin{align*} \|\gph - \Sn \geps\|^2_{\R^n} - \tfrac{1}{|\xx|}\|\geps - \gp\|_{L^2(\xx)}^2 &\leq \tfrac{2\|\geps - \gp\|^2_{L^{\infty}(\xx)}\log \tfrac{1}{\delta}}{3n} \\
&+ \sqrt{\tfrac{2 \|\geps -\gp\|^2_{L^{\infty}(\xx)}\tfrac{1}{|\xx|}\|\geps-\gp\|^2_{L^{2}(\xx)}\log \tfrac{1}{\delta}}{n}},
\end{align*}
and thus
$$ 
\|\gph - \Sn \geps\|^2_{\R^n} \leq \left( \tfrac{1}{\sqrt{|\xx|}}\|\geps - \gp\|_{L^2(\xx)} + \|\geps - \gp\|_{L^{\infty}(\xx)}\sqrt{\tfrac{2\log \tfrac{1}{\delta}}{n}}\right)^2.
$$ 
Hence, by \cref{eq:approx_gaussian}, and because $|\xx| = 2^d$, there exists two constants $C_3$ and $C_4$ depending only on $d$ and $\beta$ such that with probability at least $1-\delta$, it holds 
\begin{equation}
\label{eq:bound_empirical_part}
    \|\gph - \Sn \geps\|_{\R^n}  \leq C_3 \eps~\|\gp\|_{W^{\beta}_2(\R^d)} + C_4 \eps \frac{\|\gp\|_{\bullet}\log\tfrac{1}{\delta}}{\eps^{\nu}\sqrt{n}}.
\end{equation}

\paragraph{2. Guaranteeing $\sup_{x \in \xx}{\|(I-\Projm)\phie(x)\|} \leq \lambda^{1/2} = \eps^{1+d/(2\beta)}$} Using \cref{lm:bound_projection} and proceeding in the same way as in point 2 of the proof of \cref{cor:compression_meps}, we see that there exists constants $C_5,C_6,C_7$ depending only on $d$ and $\beta$ such that as soon as 
\begin{equation}
    \label{eq:lower_bound_m}
    m \geq C_5 \eps^{-d/\beta}\left(\log \tfrac{C_6}{\eps}\right)^d \log \tfrac{C_7}{\delta \eps},
\end{equation}
it holds $\sup_{x \in \xx}{\|(I-\Projm)\phie(x)\|} \leq \lambda^{1/2}$ with probability at least $1-\delta$.

\paragraph{3. Finding a lower bound for $\|\Cop\|$} This will be necessary in the next bound. Let $v(z) = k_{\eta}(0,z) = e^{-\tau\|z\|^2}$. Then $\|v\|_{\hhe} = 1$ and 
\begin{align*}
    \|\Cop^{1/2} v\|_{\hh}^2 &= \tfrac{1}{|\xx|}\int_{\xx}{|v(x)|^2 ~dx}\\
    & = \tfrac{1}{|\xx|}\left(\int_{-1}^1{e^{-2\tau t^2}~dt}\right)^{d}\\
    &\geq \tfrac{1}{2^d}\left(\int_{-1}^1{e^{-2t^2}~dt}\right)^d ~\tau^{-d/2} = C_8 ~ \tau^{-d/2},
\end{align*}
where the last inequality comes from the fact that $\tau \geq 1$ since $\eps \leq 1$.
Hence, $\|\Cop\| \geq C_8 \tau^{-d/2}$ where $C_8$ is a constant depending only on $d$. Hence, as soon as $\lambda \leq C_8 \tau^{-d/2}$ which rewrites $\eps \leq \sqrt{C_8}$, it holds $\la \leq \|\Cop\|$.
\paragraph{4. Bounding $\theta_1,\theta_2$.} Using the same reasoning as that of Proposition 2. of \citet{rudi2015less}, if $b=\|\Cl^{-1/2}(\Cn - \Cop)\Cl^{-1/2}\|$, then $\theta_1 \leq 1/(1-b)$ and $\theta_2^2 \leq 1+b$. Bounding $b$ can be done using Proposition 8 of \citet{rudi2015less}: if $\lambda \leq \|\Cop\|$, and $\delta \in (0,1]$ it holds, with probability at least $1-\delta$ : 
\begin{equation}
  \|\Cl^{-1/2}(\Cn - \Cop)\Cl^{-1/2}\| \leq \frac{2(1 + {\cal{N}}_{\infty}(\la))\log\tfrac{8}{\lambda \delta}}{3n} + \sqrt{\frac{2 {\cal{N}}_{\infty}(\la)\log \tfrac{8}{\lambda \delta}}{n}},
\end{equation}

where we have used the fact that $\tr{(\Cop)} \leq 1$.

 Note that   ${\cal{N}}_{\infty}(\la) = \sup_{x \in \xx}{\|\Cl^{-1/2}\phie(x)\|^2} \leq C_9 \tau^{(s-d)d/(2s)}\la^{-d/(2s)} $ for any $s > d/2$ where $C_9$ depends only on $s,d$ by a proof completely analog as that of Step 2 of Lemma E.4 by \citet{rudi2021psd}. Replacing the values of $\tau,\la$ yields : ${\cal{N}}_{\infty}(\la) \leq C_9 \eps^{-\tfrac{2d(\beta+s)-d^2}{2s\beta}}$.
 
 Note that the function $\gamma : s \in ]d/2,+\infty[ \mapsto \tfrac{2d\beta + 2 d s  - d^2}{2s\beta}$ is a homography and therefore reaches all the values $\nut$ strictly between $2$ and $d/\beta$.

 Therefore, for any $\nut > \nu$, there exists a constant $C_{10}$ depending only on $d$ and $\nut$ such that $(1+{\cal{N}}_{\infty}(\la))\log\tfrac{1}{\la} \leq C_{10} \eps^{-2\nut}$.
 
 Hence, there exists a constant depending only on $d,\beta,\nut$ such that if $n \geq C_{11} \eps^{-2\nut} \log \tfrac{8}{\delta}$, and if $\eps \leq \min(1/2,\sqrt{C_8})$ then $b \leq 1/3$ (here we have bounded $\log \tfrac{8}{\delta \la}$ by a constant times $\log\tfrac{1}{\la} \log\tfrac{8}{\delta}$ provided $\eps \leq 1/2$ and hence $\la \leq 1/4$. Moreover, note that $C_{11}$ can be taken large enough, by \cref{eq:bound_empirical_part}, to guarantee the following, also with probability $1-\delta$ : 
\begin{equation}
    \label{bound_empirical_part_2}
\|\gph - \Sn \geps\|_{\R^n}  \leq C_3 \eps~\|\gp\|_{W^{\beta}_2(\R^d)} + C_4 \eps \|\gp\|_{\bullet}.
\end{equation}

 \paragraph{5. Applying \cref{thm:rudi_2_2015} to $\geps$.} Combining all the previous equations, we get that if  $n \geq C_{11} \eps^{-2\nut} \log \tfrac{8}{\delta}$, $\eps \leq  \min(1/2,\sqrt{C_8})$ and $m \geq C_5 \eps^{-d/\beta}\left(\log \tfrac{C_6}{\eps}\right)^d \log \tfrac{C_7}{\delta \eps}$, it holds \cref{bound_empirical_part_2} and $b \leq 1/3$ as well as $\sup_{x \in \xx}{\|(I-\Projm)\phie(x)\|} \leq \lambda^{1/2}$ and hence, using the bound on $\geps$, there exists a constant $C_{12}$ depending only on $d,\beta$ such that
$$ \|\Cl^{1/2}(\geps - \ggauss)\| \leq  C_{12} \max(\|\gp\|_{W^{\beta}_2(\R^d)},\|\gp\|_{\bullet}) \eps.$$

Thus, using the bound on $\|\geps - \gp\|_{L^2(\R^d)}$, and the fact that $g\Cop g = \tfrac{1}{|\xx|}\|g\|^2_{L^2(\xx)} $ we get 

\begin{align}
    \|\gp-\ggauss\|_{L^2(\xx)}& \leq C_{13}\max(\|\gp\|_{W^{\beta}_2(\R^d)},\|\gp\|_{\bullet}) \eps,\nonumber\\
    \|\ggauss\|_{\hhe} & \leq  C_{14} \max(\|\gp\|_{W^{\beta}_2(\R^d)},\|\gp\|_{\bullet})~\eps^{-d/2\beta}.\label{eq:intermediate_result}
\end{align}

\paragraph{6. Bounding the performance of $\pgauss$.}
Note that $q = \tfrac{\gp}{\|\gp\|_{L^2(\xx)}}$ and $\sqrt{\pgauss} = \tfrac{|\ggauss|}{\|\ggauss\|_{L^2(\xx)}}$. Thus, using \cref{eq:technical_lemma}, it holds

\begin{align*}
    H(\pgauss,p) &= \left\|\tfrac{\gp}{\|\gp\|_{L^2(\xx)}} - \tfrac{|\ggauss|}{\|\ggauss\|_{L^2(\xx)}}\right\|_{L^2(\xx)}\\
    & \leq 2\frac{\|\ggauss - \gp\|_{L^2(\xx)}}{\|\gp\|_{L^2(\xx)}}.
\end{align*}
Hence, since $q = \gp/\|\gp\|_{L^2(\xx)}$, we have by \cref{eq:intermediate_result} : 
$$H(p,\pgauss) \leq 2 C_{13} \max(\|q\|_{W^{\beta}_2(\R^d)},\|q\|_{\bullet}) \eps.$$
Moreover, by \cref{eq:technical_lemma_2}, if $2 C_{13} \max(\|q\|_{W^{\beta}_2(\R^d)},\|q\|_{\bullet}) \eps \leq 1$, then $\tfrac{\|\geps\|_{L^2(\xx)}}{\|\ggauss\|_{L^2(\xx)}} \leq 2$ and hence again by \cref{eq:intermediate_result}, $ \|\pgauss\|_{\hhe} \leq  2C_{14} \max(\|q\|_{W^{\beta}_2(\R^d)},\|q\|_{\bullet})~\eps^{-d/2\beta}$. Setting $\eps_0 = \min(1/2,\sqrt{C_8},(2 C_{13} \max(\|q\|_{W^{\beta}_2(\R^d)},\|q\|_{\bullet}))^{-1})$, we therefore have all the desired properties.
\paragraph{7. Replacing norms on $\R^d$ with norm on $\xx$.}
To do so, we just use \cref{cor:extension-intersection}, which does not change anything up to multiplicative constants depending only on $d,\beta$.

\end{proof}

\bt[Performance of $\psample$]\label{thm:performance_p_sample_hellinger_app} Under the assumptions and notations of the previous theorem (\cref{thm:bound_learning_hellinger_app}), there exists a constant $C_6$ depending only on $d,\beta$, such that the following holds.
Let $\pgauss$ be given by the previous proposition. Let $\psample$ be the dyadic approximation of $\pgauss$ on $Q = \xx = (-1,1)^d$ and of width $\rho$ (see \cref{eq:dyadic_approx_density}). Recall from \cref{thm:approximation_distribution} that \cref{alg:sampling} applied to $Q = (-1,1)^d,N,\rho$ returns $N$ i.i.d. samples from $\psample$. 

If on the one hand $\rho$ is set to $\eps^{1+(d+2)/(2\beta)}$, then with probability at least $1-3\delta$,
\begin{align}
    H(\pgauss,\psample) &\leq C_6 \|q\|_{L^{\infty}(\xx)\cap W^{\beta}_2(\xx)} ~\eps ,\nonumber\\
     H(p,\psample) &\leq (C_4+C_6)\|q\|_{L^{\infty}(\xx)\cap W^{\beta}_2(\xx)}~\eps.    \label{eq:bound_final_hellinger_1}
\end{align}
If on the other $\rho$ is set adaptively to guarantee $H(\psample,\pgauss) \leq \eps$ as in \cref{rk:adaptive_rho}, then with probability at least $1-3\delta$,
\begin{align}
    &\rho \geq \eps^{1+(d+2)/\beta}/(C_6~\|q\|_{L^{\infty}(\xx)\cap W^{\beta}_2(\xx)}),\nonumber \\
    &H(\pgauss,\psample) \leq \eps ,
    H(p,\psample) \leq (C_1+1)\eps.    \label{eq:bound_final_hellinger_2}
\end{align}
In any case, this guarantees that the complexity in terms of $\erf$ computations is bounded by
\begin{equation}
    \label{eq:complexity_hellinger}
    O(N m^2\log \tfrac{1}{\rho}) = O\left(N~\eps^{-2d/\beta} \log^{2d+1}\left(\tfrac{1}{\eps}\right) \left(\log\left(\tfrac{1}{\eps}\right) +\log\left(\tfrac{1}{\delta}\right)\right)\right),
\end{equation}
where the $O$ notations is taken with constants depending on $d,\beta$, $\|q\|_{L^{\infty}(\xx)\cap W^{\beta}_2(\xx)}$.
\et

\bpr 
  Let us bound $\Lip_{\infty}(\sqrt{\pgauss})$. 
Note that since for any $x,y \in \xx$, it holds 
\begin{align*}
\left|\sqrt{\pgauss}(x) - \sqrt{\pgauss}(y)\right| &= ||\ggauss(x)| - |\ggauss(y)||/\sqrt{\Zgauss} \\
&\leq |\ggauss(x) - \ggauss(y)|/\sqrt{\Zgauss},
\end{align*}
we have $\Lip_{\infty}(\sqrt{\pgauss}) \leq \Lip_{\infty}(\ggauss)/\Zgauss$. Now 
$$ \Lip_{\infty}(\ggauss) \leq  \sup_{x \in \xx}{\sum_{k=1}^d{\partial_k \ggauss(x)}}.$$
Using \cref{lm:differential_gaussian_embedding}, we get $\Lip_{\infty}(\ggauss)~ \leq d \sqrt{2\tau}\|\ggauss\|_{\hhe}$. Using the fact that $\tau = \eps^{-2/\beta}$ and that by \cref{eq:error_p_gauss_hellinger}, $\|\ggauss\|_{\hhe}/\sqrt{\Zgauss} \leq \sqrt{C_5} \|q\|_{L^{\infty}(\xx)\cap W^{\beta}_2(\xx)} \eps^{-d/(2\beta)}$, we therefore have $\Lip_{\infty}(\sqrt{\pgauss})~ \leq d \sqrt{2C_5} \|q\|_{L^{\infty}(\xx)\cap W^{\beta}_2(\xx)} \eps^{-(d+2)/(2\beta)}$. Hence, applying \cref{thm:variation_bounds_evolved} to $\pgauss$, we get
\begin{equation}
    \label{eq:bound_perfo_ro_hellinger}
    H(\psample,\pgauss) \leq  2^{d/2}  d \sqrt{2C_5} \|q\|_{L^{\infty}(\xx)\cap W^{\beta}_2(\xx)} \eps^{-(d+2)/(2\beta)} ~\rho.
\end{equation}
On the one hand, if we use \cref{alg:sampling} with $\rho = \eps^{1 +\tfrac{(d+2)}{2\beta}} $, by the previous equation, we get $H(\psample,\pgauss) \leq 2^{d/2}  d \sqrt{2C_5}~ \eps$.

If on the other hand we find $\rho$ adaptively by computing an upper bound $\Lipt(a)$ defined in \cref{eq:bound_lip_a} s.t. $\Lipt(a) = \sqrt{2\tau}d\|K^{1/2}a\| = \sqrt{2\tau}d\|\ggauss\|/\sqrt{\Zgauss}\geq \Lip_{\infty}(\sqrt{\pgauss})~$ from $\pgauss$ and finding $\rho$ such that $2^{d/2} \Lipt(a)~ \rho  = \eps$, we will get $\rho \geq \tfrac{\eps^{1+(d+2)/(2\beta)}}{2^{d/2}~d~\sqrt{2C_5}\|q\|_{L^{\infty}(\xx)\cap W^{\beta}_2(\xx)}}$ and hence $H(\psample,\pgauss) \leq \eps$. The last point is just a consequence of \cref{thm:approximation_distribution} and the bound on $m$ in \cref{eq:bound_m_final_hell}.

\epr

\vfill

\end{document}